\documentclass{article} 
\usepackage{iclr2027_conference,times}

\usepackage{amsmath,amsfonts,bm}

\def\eqref#1{equation~\ref{#1}}

\def\1{\bm{1}}

\DeclareMathAlphabet{\mathsfit}{\encodingdefault}{\sfdefault}{m}{sl}
\SetMathAlphabet{\mathsfit}{bold}{\encodingdefault}{\sfdefault}{bx}{n}

\newcommand{\E}{\mathbb{E}}

\usepackage{graphicx}
\graphicspath{{../}}
\usepackage{minitoc}
\usepackage{titletoc}
\usepackage[utf8]{inputenc} 
\usepackage[T1]{fontenc}    
\usepackage{xcolor}
\definecolor{mydarkred}{rgb}{0.6,0,0}
\definecolor{mydarkgreen}{rgb}{0,0.6,0}
\usepackage[colorlinks, linkcolor=mydarkred, citecolor=mydarkgreen]{hyperref}

\usepackage{hyperref}
\usepackage{url}
\usepackage{bbm}
\usepackage{amsthm}
\usepackage{amsmath}
\usepackage{amssymb}
\usepackage{booktabs} 
\usepackage{algorithm}
\usepackage{subfig}
\usepackage{listings}
\usepackage[most]{tcolorbox}
\usepackage{caption}
\usepackage{multirow}
\usepackage{algpseudocode}
\newtheorem{proposition}{Proposition}
\newtheorem{assumption}{Assumption}
\newtheorem{theorem}{Theorem}

\newtheorem{remark}{Remark}
\newtheorem{lemma}{Lemma}
\newtheorem{corollary}{Corollary}

\newif\ifshowrevisions
\showrevisionstrue
\newcommand{\revisioncolor}{%
  \ifshowrevisions\color{red}\ifmmode\else\hypersetup{allcolors=red}\fi\fi}

\newif\ifshowfcpchanges
\showfcpchangestrue
\newcommand{\bluerevisioncolor}{%
  \ifshowfcpchanges\color{blue}\ifmmode\else\hypersetup{allcolors=blue}\fi\fi}

\newcommand{\marknewreferences}{%
  \ifshowrevisions
    \let\originalbibitem\bibitem
    \renewcommand{\bibitem}[2][]{%
      \normalcolor
      \ifcsname revisedbib@##2\endcsname\fi
      \originalbibitem[##1]{##2}}%
  \fi}
\expandafter\def\csname revisedbib@hartoglei2025\endcsname{}
\expandafter\def\csname revisedbib@howard2020\endcsname{}
\expandafter\def\csname revisedbib@fan2015exponential\endcsname{}
\expandafter\def\csname revisedbib@hoeffding1963\endcsname{}

\newcommand{\Pp}{\mathbb P}
\newcommand{\cB}{\mathcal B}
\newcommand{\cC}{\mathcal C}
\newcommand{\cE}{\mathcal E}
\newcommand{\cF}{\mathcal F}

\newcommand{\cH}{\mathcal H}
\newcommand{\cI}{\mathcal I}
\newcommand{\cJ}{\mathcal J}
\newcommand{\cP}{\mathcal P}

\newcommand{\cV}{\mathcal V}
\newcommand{\one}{\bm 1}

\lstdefinestyle{promptstyle}{
    basicstyle=\ttfamily\small,
    breaklines=true,
    breakatwhitespace=false,
    columns=fullflexible,
    frame=none,
    xleftmargin=0pt,
    xrightmargin=0pt,
    showstringspaces=false,
    keepspaces=true,
}

\newtcblisting{promptbox}{
    colback=white,
    colframe=black,
    boxrule=0.5pt,
    arc=0pt,       
    sharp corners,
    listing only,
    listing options={style=promptstyle},
    left=6pt, right=6pt, top=4pt, bottom=4pt,
}
\title{Anytime-Valid LLM Leaderboards via Benchmark-weighted and Block-Factorized $e$-Processes}

\author{
Hongfu Gao$^{1}$,
Songxin Zhang$^{2}$,
Zejian Xie$^{2}$,
Bingyi Jing$^{3}$,
Wang Zhou$^{1}$\thanks{Corresponding authors: \texttt{wangzhou@nus.edu.sg} and \texttt{liuyiming@jnu.edu.cn}.},
Yiming Liu$^{4*}$ \\
$^{1}$National University of Singapore \\
$^{2}$Southern University of Science and Technology \\
$^{3}$The Chinese University of Hong Kong, Shenzhen \\
$^{4}$Jinan University
}

\iclrfinalcopy 
\begin{document}
\doparttoc

\maketitle

\begin{abstract}
Large language model (LLM) leaderboards compare model capabilities by ranking models according to their mean performance on fixed benchmarks.
However, variability in evaluation outcomes across runs may produce unsupported claims of model superiority on the benchmark, a risk compounded by leaderboard updates.
In this paper, we propose BB-EDGE (Benchmark-Weighted and Block-Factorized $e$-processes for Directed Graph Evaluation), a principled framework that represents an LLM leaderboard as a directed graph whose edges certify pairwise mean-performance advantages, with anytime-valid family-wise error rate (FWER) control.
Concretely, for each direction, BB-EDGE constructs an empirical-Bernstein $e$-process by factorizing evidence over protocol-defined blocks and assigning stakes proportional to the corresponding block weights, then applies direct $e$-Holm across these $e$-processes to certify directional advantages as edges.
Theoretically, we characterize weight-proportional linear stakes under heterogeneous benchmark-average nulls and prove anytime FWER control under arbitrary within-block and cross-pair dependence.
BB-EDGE further supports anytime-valid Top-$k$ certification and simultaneous rank intervals.
Extensive experiments on synthetic data and four real-world benchmarks demonstrate that BB-EDGE maintains anytime FWER control while achieving high efficiency.
\end{abstract}

\section{Introduction}

Large language model (LLM) leaderboards have become the dominant means of comparing LLM capabilities.
They rank models by mean performance on a fixed benchmark under a standardized protocol and are often updated over successive evaluation runs \citep{biderman2024lessons,madaan2024quantifying}.
Due to stochastic factors such as sampled decoding \citep{gonzalez2025repetitions,du2026stability} and LLM-based judging \citep{chen2024humans,divekar2026precise}, LLM responses and their evaluation scores vary across runs.
Consequently, a higher empirical mean does not by itself establish that one model has higher expected performance than another \citep{ye2024benchmarking,vashurin2025benchmarking,neuhof2026rank}.
Treating such score differences as definitive evidence of model superiority can lead to insufficiently supported claims, thereby undermining the credibility of leaderboards.
This highlights the need for \textit{anytime-valid} LLM leaderboards whose model comparisons remain statistically guaranteed across all updates.


Several approaches represent leaderboards as confidence graphs \citep{wang2024confidence}, quantify ranking uncertainty through rank intervals \citep{li2026low,neuhof2026rank}, or support model selection through Top-$k$ identification \citep{fan2025ranking}.
However, these methods target a fixed evaluation budget and do not provide anytime-valid guarantees under repeated monitoring or data-dependent stopping.
Other work has developed online evaluation procedures for evaluating individual LLMs \citep{hsu2026efficient,zhou2026celeus} and inferring pairwise human preferences \citep{chiang2024chatbot,gu2026anytimevalid}, typically under an independent and identically distributed item stream.
Simulations show that an i.i.d.\ item-stream procedure can falsely
certify benchmark-average advantages when item-level model differences
are heterogeneous despite a near-zero benchmark-average difference,
or when observations within multi-turn dialogues are dependent.
Same-model comparisons also reveal false certifications on real-world benchmarks.

In this paper, we develop a principled framework for efficiently constructing anytime-valid confidence graphs that certify pairwise model advantages as edges while controlling the family-wise error rate (FWER), i.e., the probability of ever reporting a false edge across all replicates.
Constructing such confidence graphs presents two key challenges: arbitrary dependence within multi-turn dialogues or multi-step trajectories, and heterogeneity in item-level model advantages.
To address both, we propose Benchmark-weighted and Block-Factorized $e$-processes for Directed Graph Evaluation (\textbf{BB-EDGE}).
In particular, BB-EDGE factorizes empirical-Bernstein evidence across protocol-defined blocks and uses weight-proportional stakes to preserve the heterogeneous benchmark-average null. It then applies direct \(e\)-Holm to certify mean-performance advantages with anytime FWER control. We further provide a theoretical analysis showing that BB-EDGE's weight-proportional stakes are sufficient and necessary for nonpositive linear drift over the full benchmark-average null, and that the resulting procedure guarantees anytime-valid FWER control.

Extensive experiments on both synthetic and real-world datasets validate both the anytime validity and efficiency of BB-EDGE. The simulation results demonstrate that BB-EDGE maintains anytime-valid FWER control while achieving the highest power among FWER-controlling baselines across the examined scenarios. In real-world evaluations, we apply BB-EDGE to ten LLMs across four benchmarks, the single-turn benchmarks MMLU-Pro \citep{wang2024mmlu} and GSM8K \citep{cobbe2021gsm8k}, the multi-turn dialogue benchmark MT-Bench \citep{zheng2023judging}, and ACEBench \citep{chen2025acebench}. For example, on GSM8K, BB-EDGE certifies 41 of the 45 pairwise comparisons (91.1\%) after just 2 replicates, using $67\%$ fewer replicates than the strongest competing baseline.
Moreover, BB-EDGE extends to anytime-valid Top-\(k\) certification and simultaneous rank intervals enabling certified model selection and ranking uncertainty quantification.

Our contributions are summarized as follows:
\begin{enumerate}
    \item We introduce BB-EDGE, an anytime-valid and efficient framework for certifying LLM leaderboards.
    It uses block-factorized empirical-Bernstein $e$-processes with weight-proportional stakes under heterogeneous block means and arbitrary within-block dependence, while preserving the benchmark-average estimand.
    \item We theoretically show that proportionality is necessary and sufficient among nonnegative stake allocations for nonpositive linear drift over the full benchmark-average null. 
    Combining the resulting $e$-processes with direct $e$-Holm yields anytime FWER control across all directions and completed replicates.
    \item We conduct extensive experiments on synthetic and real-world datasets, showing that BB-EDGE controls anytime FWER while achieving high power.
    We further extend BB-EDGE to anytime-valid Top-$k$ certification and simultaneous rank intervals.
\end{enumerate}

\section{Related Work}

\paragraph{LLM Leaderboards.}
Existing work on LLM leaderboards largely focuses on benchmark construction and evaluation protocols, spanning single-turn knowledge and reasoning tasks \citep{peng2024humaneval,zhu2024dyval}, multi-turn dialogues \citep{zheng2023judging}, and multi-step trajectories \citep{chen2025acebench,nakash2026efficient}. 
Evaluation outcomes may also vary across repeated runs because of sampled decoding \citep{gonzalez2025repetitions,du2026stability} and stochastic LLM-based judging \citep{chen2024humans,divekar2026precise}. Nevertheless, common practice summarizes each model by a single average score under a fixed protocol and reports the resulting ranking \citep{ye2024benchmarking,frick2025prompttoleaderboard,wang2025rankings,zhou2026lost}, without quantifying which pairwise performance advantages are statistically supported.

\paragraph{Model Ranking.}
A separate line of work develops statistical methods for model ranking and uncertainty quantification. Some methods replace the raw mean score with alternative ranking criteria designed to improve reliability \citep{wei2024diff,hariri2026don}. Others provide inferential guarantees using Holm-adjusted paired one-sided $t$-tests for simultaneous pairwise comparisons \citep{chandrahas2026evalci} or conformal methods for rank uncertainty quantification \citep{neuhof2026rank}. These methods generally target inference at a fixed evaluation budget and therefore do not protect against repeated examination of the ranking. Related work also ranks models using estimated pairwise preference probabilities \citep{chatzi2024predictionpowered,haghtalab2026pluralistic,huang2026dropping,li2026llm,menendez2026prompt}, with extensions to online preference data \citep{chiang2024chatbot,gu2026anytimevalid}. This estimand differs from that of objectively scored benchmarks such as MMLU-Pro \citep{wang2024mmlu} and GSM8K \citep{cobbe2021gsm8k}, where models are evaluated against fixed ground truth and ranked by their mean task scores.

\paragraph{Anytime-Valid Inference.}
Sequential data collection is common in many real-world LLM evaluation \citep{wang2025cer,singh2026leaderboard}. Anytime-valid procedures preserve statistical validity under continuous monitoring and data-dependent stopping \citep{ramdas2023game,xu2024active}. Recent work applies $e$-processes to reduce the number of items required for evaluating individual LLMs \citep{hsu2026efficient,zhou2026celeus} or to infer pairwise preference probabilities from human preferences arriving online \citep{gu2026anytimevalid}. These methods formulate evaluation as a stream of i.i.d.\ items. This formulation does not directly accommodate fixed benchmarks with heterogeneous item-level model advantages, where one model may outperform another on some items but underperform it on others, nor arbitrary dependence within multi-turn dialogues or multi-step trajectories.
In contrast, this paper targets pairwise mean-performance differences on a fixed benchmark and controls the probability of reporting any false edge across all model pairs and replicates.

\section{Statistical Framework for Anytime-Valid LLM Leaderboards}
We consider a sequential evaluation setting in which LLMs are evaluated on a fixed benchmark across multiple replicates. In this section, we formalize this setting, represent LLM leaderboards through certified edges, and outline our goal of achieving \textit{anytime-valid} and \textit{efficient} certification.

\paragraph{Sequential Evaluation for LLM Leaderboards.}
An LLM leaderboard compares $L$ candidate models on a fixed benchmark $\mathcal{B}=\{x_1,\ldots,x_N\}$, where each $x_i$ denotes an evaluation item. All models are evaluated under a standardized protocol $\pi$ that specifies the prompts, decoding configuration, available tools, grading rules, and dataset structure.

In the sequential setting, the benchmark is evaluated over successive complete replicates $r = 1, 2, \ldots$, where each replicate evaluates every model on all $N$ items in $\mathcal{B}$. Let $S_{i,r}^{(\ell)}\in[0,1]$ be the evaluation score of model $\ell$ on item $i$ in the $r$-th complete replicate. We define the mean performance of model $\ell$ on $\mathcal{B}$ as
\[
\theta_\ell^{\cB}=\E\left[\frac1N\sum_{i=1}^N S_{i,r}^{(\ell)}\middle|\cB,\pi\right].
\]
where the expectation is taken over the randomized execution of the fixed protocol $\pi$.
Leaderboards rank models by the empirical mean $\hat\theta_{\ell,r}=\frac{1}{rN}\sum_{s=1}^{r}\sum_{i=1}^{N}S_{i,s}^{(\ell)}$, often computed from a single run ($r=1$), which is an unbiased estimate of $\theta_\ell^{\cB}$; thus $\theta_\ell^{\cB}$ is precisely the quantity that leaderboards implicitly estimate.
However, this estimate varies across runs, so single-run rankings can change from one run to the next.
For example, on GSM8K, Figure~\ref{fig:rank-stability} shows that the single-run ranking differs from the 20-replicate ranking in 45\% of replicates (9 of 20), with all reversals occurring among three closely matched models.
We therefore need to certify comparisons of $\theta_\ell^{\cB}$ with statistical guarantees rather than relying on point estimates, and the margin $\tau$ introduced below lets practitioners require practically meaningful advantages rather than merely nonzero ones.
Comparisons of $\theta_\ell^{\cB}$ concern performance on $\cB$ under $\pi$, not generalization to unseen items.


\paragraph{Ranking via Certified Edges.}
Let $\overrightarrow{\mathcal{P}}=\{(a,b)\in[L]^2\mid a\neq b\}$ denote the ordered pairs of distinct models, where $[L]=\{1,\ldots,L\}$. For each $(a,b)\in\overrightarrow{\mathcal{P}}$, define the mean-performance difference on $\cB$ as $\Delta_{a,b}^{\cB}=\theta_a^{\cB}-\theta_b^{\cB}$. For a prespecified margin $\tau\in[0,1)$, each direction $e=(a\to b)$ is tested via
\[
H_e^{(\tau)}:
\Delta_{a,b}^{\cB}\leq\tau\qquad\text{against}\qquad\Delta_{a,b}^{\cB}>\tau.
\]
Rejecting $H_e^{(\tau)}$ certifies the edge $a\to b$, indicating that model $a$ outperforms model $b$ on the fixed benchmark $\mathcal{B}$ by more than $\tau$.  If neither direction is certified, the pair remains \textit{unresolved}, that is, the evidence is insufficient to establish that either model outperforms the other by more than $\tau$. The certified edge set $\mathcal{E}$ assembles into a directed comparison graph $\mathcal{G}=(\mathcal{V},\mathcal{E})$ over the $L$ models, which serves as the leaderboard. A directed path $a\rightsquigarrow b$ of length $d$ certifies $\Delta_{a,b}^{\cB}>d\tau$, so the graph reports a partial order in which every reachability relation is statistically supported.



\paragraph{Anytime Validity and Power.} 
Let $\mathcal{E}_r=\{e\in\overrightarrow{\mathcal{P}}: H_e^{(\tau)} \text{ is rejected at replicate } r\}$
denote the set of directly certified edges at replicate $r$, $R$ the maximum number of
complete replicates, and $\alpha\in(0,1)$ the target level. For $1\le r\le R$, define the
anytime family-wise error rate
\begin{equation}
\operatorname{FWER}_r=\mathbb{P}\left(\exists s\le r,\ \exists e=(a\to b)\in\mathcal{E}_s
\text{ such that } \Delta^{\mathcal{B}}_{a,b}\le\tau \,\middle|\, \mathcal{B},\pi\right).
\label{eq:anytime-fwer-target}
\end{equation}
We say that a procedure is \textit{anytime-valid} if $\operatorname{FWER}_r\le\alpha$ for every $r\le R$.
It permits the graph to be examined after every completed replicate \(r\le R\) while maintaining simultaneous error control over all such monitoring replicates.


To measure efficiency, we also evaluate $\operatorname{Power}_r$, the expected proportion of true edges certified:
\begin{equation}
\operatorname{Power}_r=\mathbb E\!\left[\frac{\left|\left\{e=(a\to b)\in\mathcal E_r:\Delta_{a,b}^{\mathcal B}>\tau\right\}\right|}{\left|\left\{e=(a\to b)\in\overrightarrow{\mathcal P}:\Delta_{a,b}^{\mathcal B}>\tau\right\}\right|}\,\middle|\,\mathcal B,\pi\right].
\end{equation}
Ideally, a practical approach should simultaneously ensure that $\operatorname{FWER}_r \leq \alpha$ at every replicate $r$ while attaining as high a $\mathrm{Power}_r$ as possible.

\section{Methodology}
\label{sec:methodology}

In this section, we first describe the statistical challenges of fixed-benchmark inference. We then propose BB-EDGE for constructing anytime-valid and efficient leaderboards.

\paragraph{Dependence and Heterogeneity.} 
Certifying mean-performance advantages on a fixed benchmark $\mathcal B$ and protocol $\pi$ presents two key statistical challenges:
\begin{itemize}
    \item Dependence. Evaluation scores arising from the same multi-turn dialogue or multi-step trajectory may be arbitrarily dependent.
    \item Heterogeneity. Item-level performance differences may vary across the benchmark, while the directional null constrains only their benchmark-weighted average.
\end{itemize}
Both features are common in modern LLM benchmarks.
For example, Figure~\ref{figure:heterogeneity} (a) shows that Gemma-4-E4B-it and MiMo-7B-RL have opposing item-level advantages on GSM8K, despite a near-zero benchmark-average gap of \(-0.008\). 
In Section~\ref{section:simulation_studies}, we empirically show that either feature can cause item-level online methods to inflate the FWER.
To address both, we propose BB-EDGE, which factorizes evidence across protocol-defined blocks and places weight-proportional stakes that preserve the heterogeneous benchmark-average null.
We present its construction in three steps.

\paragraph{Design-Determined Resolution.}
\label{sec:block-evidence}
We first partition the fixed benchmark $\mathcal B$ into $M$ prespecified blocks $\mathcal C=(C_1,\ldots,C_M)$ based on the dependence structure induced by the fixed protocol $\pi$. For each direction $e=(a\to b)$, define the item-level score difference $D^{e}_{i,r} = S^{(a)}_{i,r} - S^{(b)}_{i,r}$, its rescaled version $X^{e}_{i,r} = \frac{D^{e}_{i,r}+1}{2} \in [0,1]$, and the null reference point $\mu_0 = \frac{1+\tau}{2}$.  Then the directional null $H_e^{(\tau)}:\Delta_{a,b}^{\mathcal B}\leq\tau$ is equivalently written as $\frac{1}{N}\sum_{i=1}^{N}\mathbb E\!\left[X_{i,r}^{e}\mid\mathcal B,\pi\right]\leq \mu_0.$ For each block \(C_m\), let \(w_m=|C_m|/N\) and define the block-level observation $Y_{e,m,r}=\frac{1}{|C_m|}\sum_{i\in C_m}X^e_{i,r}.$ These block summaries preserve the original benchmark-level target because
\begin{equation}
\sum_{m=1}^M w_mY_{e,m,r}=\frac1N\sum_{i=1}^NX^e_{i,r}.
\label{eq:block-average-identity}
\end{equation}
Thus, the partition changes only the evidence resolution, not the estimand. Here, $M=1$ permits arbitrary within-replicate dependence, $M=N$ enables item-level factorization under conditional independence, and intermediate choices accommodate within-block dependence.

\paragraph{Benchmark-Weighted Block-Factorized $e$-Process.}
Let $W_{i,r}=(S_{i,r}^{(1)},\ldots,S_{i,r}^{(L)})$ collect the scores of all $L$ models on item $i$, and let $Z_r=(W_{1,r},\ldots,W_{N,r})$ denote the complete evaluation panel at replicate $r$. Let $\mathcal Q$ collect all pre-inference information, including the directional family, stake grid, mixture weights, and any independent pilot data used to specify them. We define $\mathcal F_r=\sigma(\mathcal B,\pi,\mathcal Q,Z_1,\ldots,Z_r)$ and, for each direction $e=(a\to b)$, let $\nu_{e,m,r}=\mathbb E[Y_{e,m,r}\mid\mathcal F_{r-1}]$ denote the conditional mean of block $m$ before replicate $r$. Under \(H_e^{(\tau)}\), the block means need not individually satisfy \(\nu_{e,m,r}\leq\mu_0\). The null constrains only their benchmark-weighted average
\begin{equation}
\sum_{m=1}^M w_m\nu_{e,m,r}=\frac{1+\Delta_{a,b}^{\cB}}2\le\mu_0.
\label{eq:conditional-average-null}
\end{equation}

To preserve the benchmark-weighted constraint in (\ref{eq:conditional-average-null}), let $\Lambda=\{\lambda_1,\ldots,\lambda_G\}\subset[0,1)$ be a prespecified stake grid and define $w_\star=\max_{m\in[M]}w_m$, 
\begin{equation}
\lambda_{g,m}=\lambda_g\frac{w_m}{w_\star}.
\label{eq:proportional-stakes}
\end{equation}
The linear drift satisfies $\sum_{m=1}^{M}\lambda_{g,m}(\nu_{e,m,r}-\mu_0)=\frac{\lambda_g}{w_\star}\left(\sum_{m=1}^{M}w_m\nu_{e,m,r}-\mu_0\right)\leq0$ under $H_e^{(\tau)}$. Because the benchmark-average null does not require \(\nu_{e,m,r}\leq\mu_0\) for each block, naive blockwise \(e\)-processes with arbitrary block-specific stakes can induce positive drift under \(H_e^{(\tau)}\). Proposition~\ref{prop:stake-necessity-iclr} characterizes the linear drift requirement used in our factorized construction: among nonnegative blockwise stake allocations, proportionality to the benchmark weights is necessary and sufficient for this drift to be nonpositive under every admissible vector of block means.

Before replicate $r$, let $\widehat Y_{e,m,r}\in[0,1]$ be a predictable estimate of $Y_{e,m,r}$. For each $\lambda_g\in\Lambda$, building on the predictable empirical-Bernstein construction of \citet{smith2024estimating}, define the block-factorized empirical-Bernstein contribution
\begin{equation}
F_{e,r}^{\cC}(\lambda_g)=\exp\left\{\sum_{m=1}^{M}\left[\lambda_{g,m}(Y_{e,m,r}-\mu_0)-\psi_E(\lambda_{g,m})(Y_{e,m,r}-\widehat Y_{e,m,r})^2\right]\right\},
\label{eq:block-factor}
\end{equation}
where $\psi_E(\lambda)=-\log(1-\lambda)-\lambda$. Let $\rho_1,\ldots,\rho_G$ be prespecified mixture weights satisfying $\rho_g\geq0$ and $\sum_{g=1}^{G}\rho_g=1$. The corresponding evidence process is
\begin{equation}
E_{e,r}^{\cC}(\lambda_g)=\prod_{s=1}^{r}F_{e,s}^{\cC}(\lambda_g),
\qquad
E_{e,r}^{\cC}=\sum_{g=1}^{G}\rho_gE_{e,r}^{\cC}(\lambda_g).
\label{eq:block-mixture}
\end{equation} 
with $E_{e,0}^{\cC}(\lambda_g)=E_{e,0}^{\cC}=1$. Together, (\ref{eq:block-factor}) and (\ref{eq:block-mixture}) define the evidence process used by BB-EDGE. The proportional stakes ensure nonpositive drift under the heterogeneous benchmark-average null, while the block-specific predictable estimates provide variance adaptation without changing the target hypothesis. Section~\ref{sec:theory} establishes the validity of this construction.

\paragraph{Confidence Graph Construction.}
At each completed replicate $r$, we apply direct $e$-Holm \citep{hartoglei2025} to $\{E_{e,r}^{\cC}:e\in\overrightarrow{\mathcal P}\}$. Specifically, define $\mathcal J_r=\left\{e\in\overrightarrow{\mathcal P}:E_{e,r}^{\cC}<\frac{1}{\alpha}\right\}$ and the rejection threshold $c_r = \frac{1}{\alpha} + \sum_{e \in \mathcal{J}_r}\left(\frac{1}{\alpha} - E_{e,r}^{\mathcal{C}}\right)$. The set of directly certified edges is
\begin{equation}
\mathcal E_r=\left\{e\in\overrightarrow{\mathcal P}:E_{e,r}^{\mathcal C}\ge c_r\right\}.
\label{eq:eholm-rejections}
\end{equation}
Direct $e$-Holm exploits the evidence across the entire directional family and remains valid under arbitrary dependence among the pairwise $e$-processes.

The resulting confidence graph is $\mathcal G_r=(\mathcal V,\mathcal E_r)$, where an edge $a\to b$ certifies $\Delta_{a,b}^{\cB}>\tau$.  A directed path \(a\rightsquigarrow_r b\) of length \(d\) implies \(\Delta_{a,b}^{\mathcal B}>d\tau\), inducing a partial ranking in which unreachable pairs remain unresolved. The edge sets need not be nested over time, and Section~\ref{sec:theory} establishes anytime FWER control across all model pairs and replicates. 
The confidence graph further supports anytime-valid \textbf{Top-$k$ certification} and \textbf{rank intervals} (Appendices~\ref{sec:pilot-topk} and~\ref{section_rank_interval}).

\section{Theoretical Analysis}
\label{sec:theory}
We establish validity using nonnegative supermartingales \citep{howard2020}.  The benchmark-level null imposes a constraint on the weight-averaged block means; proportional stakes preserve this constraint and yield a valid $e$-process for each direction. Direct $e$-Holm then gives anytime false-edge control for the full graph.

\begin{assumption}
\label{ass:design-iclr}
The following assumptions hold conditional on the fixed benchmark and protocol:
\begin{enumerate}
\item Conditional on \((\mathcal{Q},\cB,\pi)\), the complete panels \(Z_1,\ldots,Z_R\) are independent and identically distributed across runs. 
\item Within each run, the block vectors \(V_{m,r}=(W_{i,r}\mid i\in C_m)\) are mutually independent conditional on the history \((\cF_{r-1},\cB,\pi)\). Their distributions may differ across blocks.
\item Let \(\mathcal{Q}\) denote the information used to specify the partition, directional family, stake grid, and mixture weights. These objects are \(\mathcal{Q}\)-measurable, and \(\mathcal{Q}\) is independent of the inference panels \(Z_1,\ldots,Z_R\) conditional on \((\cB,\pi)\). 
In particular, they are fixed before the inference runs or selected using independent data.
\item Each prediction \(\widehat Y_{e,m,r}\) is \(\cF_{r-1}\)-measurable and lies in \([0,1]\).
\end{enumerate}
\end{assumption}
For a finite budget \(R\), all guarantees are restricted to \(1\le r\le R\). Statements covering every \(r\ge1\) require the same conditions for the entire infinite sequence of inference panels.
\begin{remark}
The independence conditions concern the data collection procedure and cannot be verified from the completed score table alone. Different random seeds do not establish independence if runs share a provider state, a cache, or a source of evaluator randomness. If the protocol does not justify independence across the selected blocks, the theorem-backed fallback is the one-block complete-run analysis.  Singleton blocks require independent randomization of each item vector, but they do not require the item vectors to be exchangeable or identically distributed.
\end{remark}

Under a true directional null, (\ref{eq:conditional-average-null}) bounds only the benchmark-weighted average of the block means. Individual blocks may favor either model, which is why unrestricted block-specific stakes can turn a valid benchmark-level null into a different claim.

\begin{proposition}
\label{prop:stake-necessity-iclr}
Let \(w_m>0\), \(\sum_m w_m=1\), and \(\mu_0\in(0,1)\). Suppose a nonnegative
stake vector \(\eta\) satisfies
\begin{equation}
\eta^\top(\nu-\mu_0\one)\le0
\quad\text{whenever}\quad
\nu\in[0,1]^M
\text{ and }w^\top\nu\le\mu_0.
\label{eq:stake-drift-condition}
\end{equation}
Then \(\eta=cw\) for some \(c\ge0\). Conversely, every vector of this form
satisfies
\[
\eta^\top(\nu-\mu_0\one)
=c\,w^\top(\nu-\mu_0\one)\le0
\]
for every \(\nu\in[0,1]^M\) with \(w^\top\nu\le\mu_0\).
\end{proposition}

Thus, proportional stakes are necessary for nonpositive linear drift over the full benchmark-level null, not merely a convenient choice. The proof is given in Appendix~\ref{app:sec3-proofs}.

\begin{theorem}
\label{thm:bedge-main-iclr}
Suppose Assumption~\ref{ass:design-iclr} holds. For every true directional null, the process \(\{E_{e,r}^{\cC}\}_{r=0}^R\) in (\ref{eq:block-mixture}) is a nonnegative supermartingale.
Let \(\mathcal N_0=\{e=(a\to b)\in\overrightarrow{\cP}:\Delta_{a,b}^{\cB}\le\tau\}\) denote the set of true directional nulls. 
If the edge sets \(\mathcal E_r\) are obtained by direct $e$-Holm, then
\begin{equation}
{\Pp\left(
\exists r\le R,\ \exists e\in\mathcal N_0:\ e\in\cE_r
\middle|\cB,\pi
\right)\le\alpha.}
\label{eq:main-anytime-fwer}
\end{equation}
On the complementary event of probability at least \(1-\alpha\), every direct edge is a true margin claim, every path \(a=v_0\to\cdots\to v_d=b\) implies \(\Delta_{a,b}^{\cB}>d\tau\), and the graph is acyclic. 
\end{theorem}

The complete proof is given in Appendix~\ref{app:proof-main-theorem}.
The displayed FWER bound conditions on the benchmark and protocol, averaging over any independent pilot used to fix the design.
It does not require independence among the directional $e$-processes.
The theorem applies to every resolution fixed in advance and justified by the evaluation design. 
One block gives the robust complete-run graph, singleton blocks give the item-factorized graph when item-level independence is justified, and intermediate partitions provide the corresponding intermediate resolutions. 
The resolution cannot be selected after observing which graph reports more edges.

Validity alone does not determine certification speed. For an alternative $e=(a\to b)$, let $\delta_e=(\Delta_{a,b}^{\mathcal B}-\tau)/2>0$.
The conditional expected log increment of grid component $g$ is $a_{e,g,r}=\frac{\lambda_g}{w_\star}\delta_e-\sum_{m=1}^{M}\psi_E(\lambda_{g,m})\mathbb E\!\left[(Y_{e,m,r}-\widehat Y_{e,m,r})^2 \mid\mathcal F_{r-1}\right].$
The first term captures the benchmark-average advantage, while the second accounts for blockwise prediction error. 
Theorem~\ref{thm:power-consistency-iclr} shows that a positive-weight grid component with persistently positive average conditional log growth yields eventual certification.


\section{Experiments}
\subsection{Synthetic Data Analysis}
\label{section:simulation_studies}

\paragraph{Data-Generating Process.}
We simulate $L=10$ models on a fixed benchmark of $N=100$ items.
Their benchmark-level means are \((\theta_1,\theta_1,\ldots,\theta_{L/2},\theta_{L/2})\), with \(\theta_1,\ldots,\theta_{L/2}\) evenly spaced over \([0.41,0.59]\).
To specify ground-truth fixed-benchmark performance, we define the probability that model $\ell$ correctly answers item $i$ as $p_i^{(\ell)}=\theta_\ell+c_i+h_{i,\ell}$, where $c_i$ is a fixed item-difficulty offset, equally spaced over $[-0.08,0.08]$, and $h_{i,\ell}$ is a model-item effect constructing heterogeneous benchmark-average nulls. Both effects are centered so that $\sum_{i=1}^{N}c_i=\sum_{i=1}^{N}h_{i,\ell}=0$ for every $\ell$.

We generate the LLMs' synthetic outcomes via a latent Gaussian variable $Z_{i,\ell,r}=\sqrt{\kappa_1}\,U^{\mathrm{block}}_{b(i),\ell,r}+\sqrt{\kappa_2}\,U^{\mathrm{item}}_{i,r}+\sqrt{1-\kappa_1-\kappa_2}\,\varepsilon_{i,\ell,r}$, where $U^{\text{block}}$, $U^{\text{item}}$, and $\varepsilon$ are independent standard Gaussian variables, and $\kappa_1$ and $\kappa_2$ control within-block and within-item cross-model dependence, respectively.
We set $\kappa_1=0.5$ and $\kappa_2=0.1$. 
Let $S_{i,r}^{(\ell)} = \mathbbm{1}\{Z_{i,\ell,r} \le\Phi^{-1}(p_i^{(\ell)})\}$ denote whether model $\ell$ answers item $i$ correctly at replicate $r$. 
Each configuration is repeated 500 times.

\textbf{Baseline Methods.} 
We compare BB-EDGE with external methods and three internal baselines. 
External baselines: EMR performs uncorrected empirical-mean comparisons;
CBTL \citep{wang2024confidence} combines context-dependent Bradley–Terry–Luce model with a Gaussian multiplier bootstrap; and $t$-Holm \citep{chandrahas2026evalci} applies Holm's procedure to paired one-sided $t$-tests. SERPANT \citep{gu2026anytimevalid} is an online method for pairwise preferences based on $e$-processes under item-level independence. 
We include three internal baselines.
BEB-Holm uses the same block observations but performs fixed-time empirical-Bernstein inference followed by Holm's procedure, providing a fixed-time baseline.
BFH-$e$-Holm differs from BB-EDGE only in replacing the empirical-Bernstein penalty with a Hoeffding penalty, isolating the benefit of variance adaptation.
CR-EDGE differs from BB-EDGE only in treating each complete replicate as a single block, isolating the benefit of block factorization.
Full details of synthetic experiments are in Appendix \ref{section_appendix_baseline}.

\begin{figure*}[t]
\centering
\includegraphics[width=0.99\textwidth]{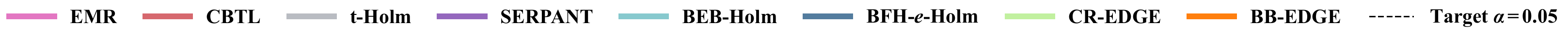}\\[-2ex]
\subfloat[$|C_m|=1$, $h_{i,\ell}=0$]{%
    \includegraphics[width=0.48\textwidth]{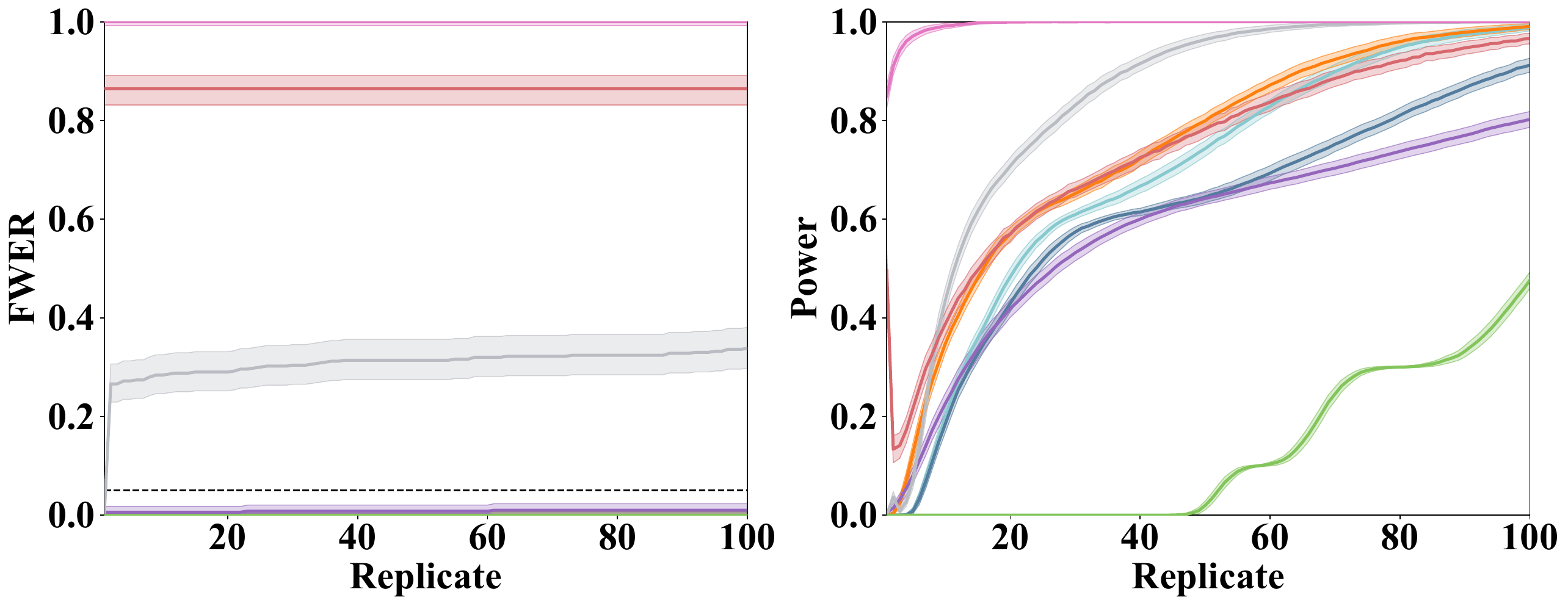}
}
\hfill
\subfloat[$|C_m|=5$, $h_{i,\ell}=0$]{%
    \includegraphics[width=0.48\textwidth]{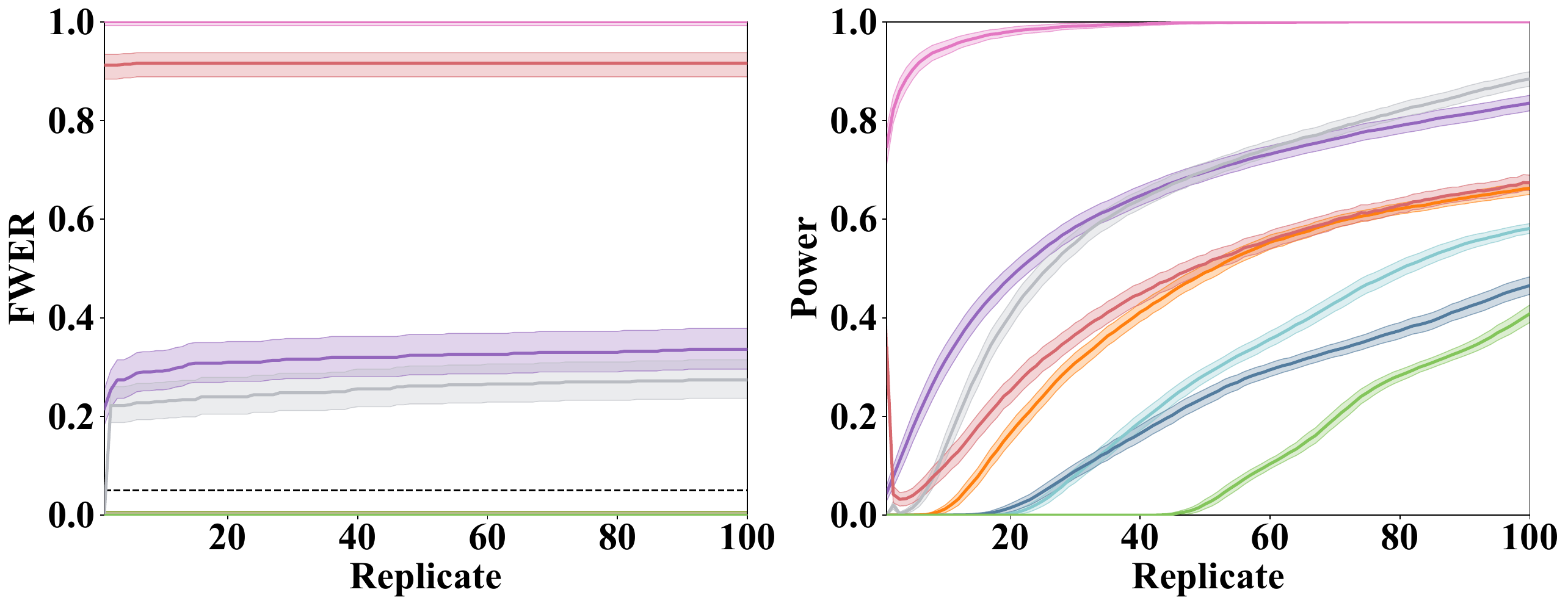}
}
\hfill
\subfloat[$|C_m|=1$, $h_{i,\ell}=0.2$]{%
    \includegraphics[width=0.48\textwidth]{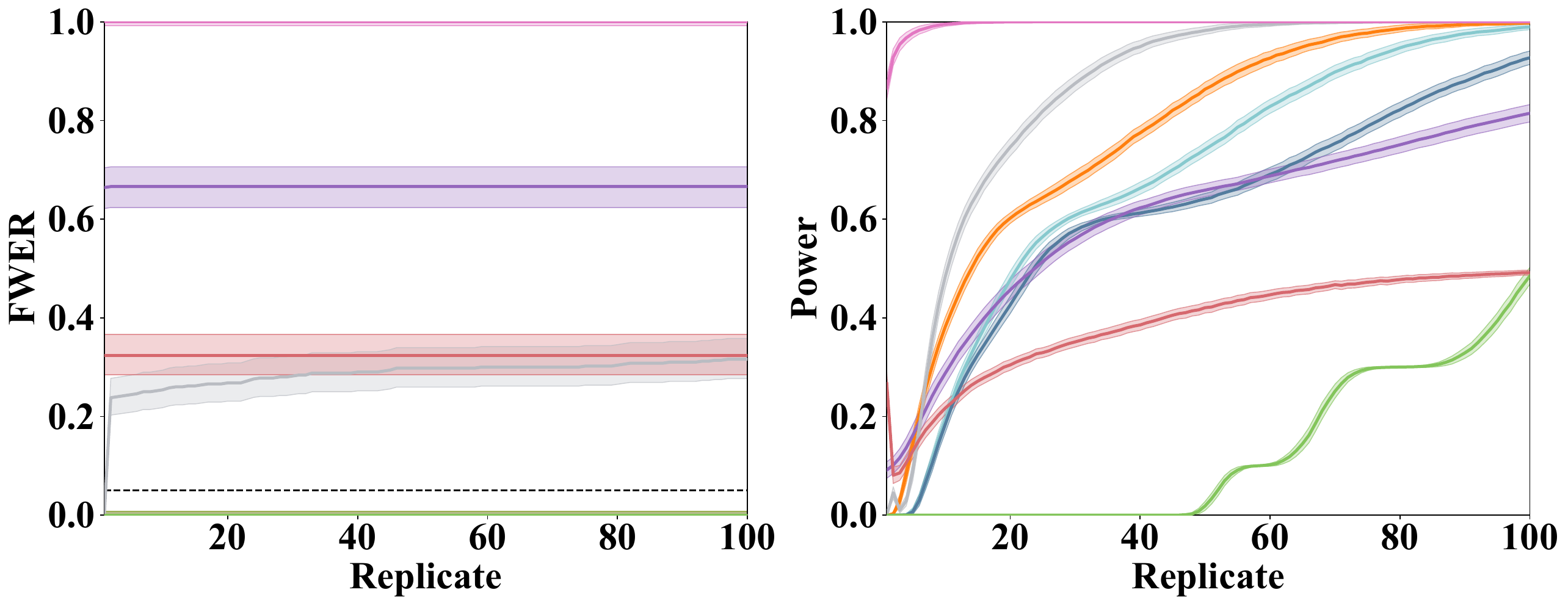}
}
\hfill
\subfloat[$|C_m|=5$, $h_{i,\ell}=0.2$]{%
    \includegraphics[width=0.48\textwidth]{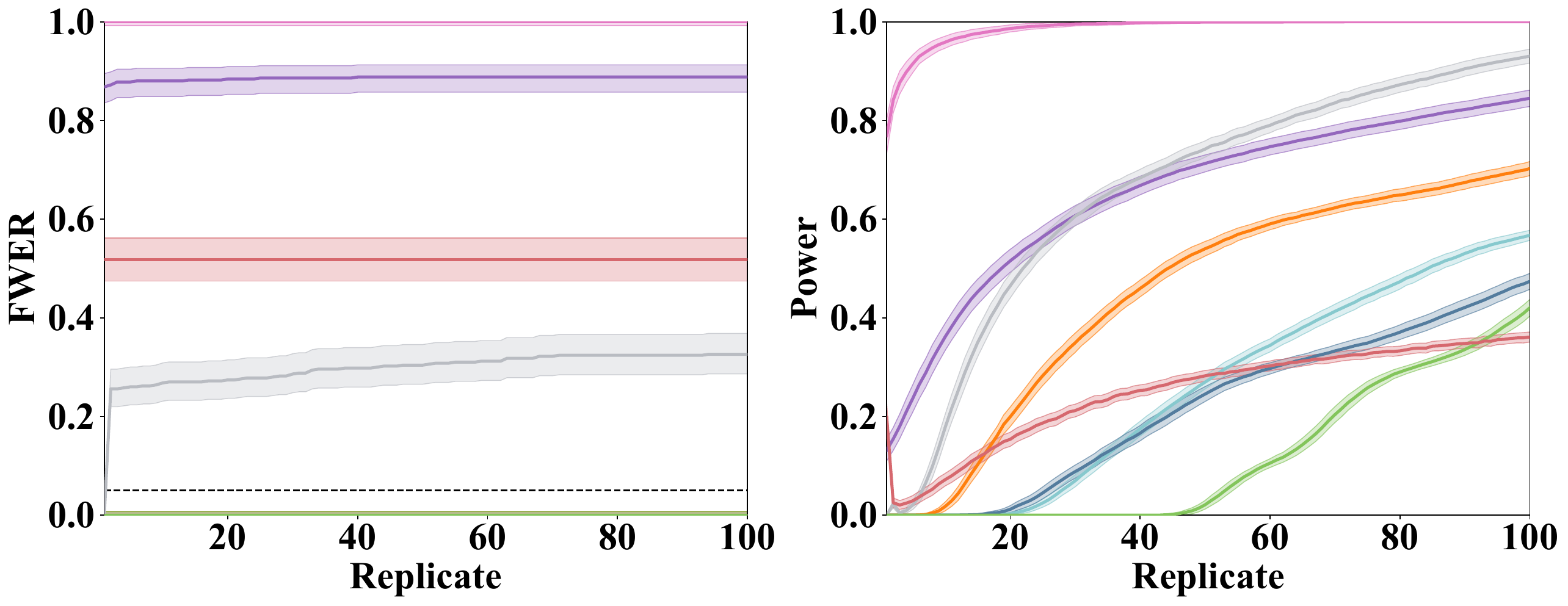}
}
\caption{Empirical $\operatorname{FWER}_r$ and $\operatorname{Power}_r$ (with 95\% confidence intervals) versus replicate $r$.
Panels (a)--(d) correspond to the i.i.d., dependence-only, heterogeneity-only, and combined dependence-and-heterogeneity settings, respectively.
All settings use $\tau=0$, $N=100$, and $L=10$.
The dashed line marks the target level $\alpha=0.05$.
}
\label{figure:1}
\end{figure*}

\textbf{BB-EDGE Controls $\operatorname{FWER}_r$ with High $\operatorname{Power}_r$.}
Figure~\ref{figure:1} first examines error control under the fixed-benchmark target. 
BB-EDGE, BFH-$e$-Holm, and CR-EDGE maintain empirical anytime $\operatorname{FWER}_r$ control across the examined settings. 
The item-level SERPANT adaptation controls error in the i.i.d.\ item setting but can report false benchmark-average advantages under item heterogeneity or within-block dependence.
This diagnostic illustrates the consequence of applying item-stream monitoring to the fixed-benchmark target. 
Additionally, three offline baselines EMR, CBTL, and $t$-Holm significantly exceed the target level.
BEB-Holm is included as a fixed-time reference; its pointwise behavior does not constitute an anytime guarantee.
Figure~\ref{figure:1} also presents that BB-EDGE exhibits faster power growth than CR-EDGE and BFH-$e$-Holm in the examined settings.
For example, under item-level heterogeneity, BB-EDGE reaches 90\% power within 56 replicates, compared with 95 replicates for BFH-$e$-Holm.
These comparisons isolate the effects of block factorization and empirical-Bernstein variance adaptation, respectively.
The conditional log-growth expression in Section~\ref{sec:theory} identifies how blockwise prediction error enters evidence accumulation.


\textbf{Ablation Study.}
We conduct a series of ablation studies under diverse experimental conditions.
First, BB-EDGE continues to perform well with a larger model set ($L=20$; Figure~\ref{figure:a5}), across benchmark sizes $N\in\{50,500\}$ (Figures~\ref{figure:a3}--\ref{figure:a4}), and under prespecified margins $\tau\in\{0.01,0.03\}$ (Figures~\ref{figure:tau1}--\ref{figure:tau3}).
Second, Figure~\ref{figure:grid} compares BB-EDGE with BB-EDGE-PI (plug-in stake), AsympCS-Bonf (non-factorized asymptotic confidence sequences), and SERPANT-R (SERPANT evaluated at completed replicates): BB-EDGE-PI performs similarly, whereas AsympCS-Bonf and, under within-block dependence, SERPANT-R exceed the target level.
Third, Figure~\ref{figure:block_heterogeneity} varies the within-block correlation from $0.1$ to $0.7$ (panel~(a)) and the block size from $1$ to $10$ (panel~(b)); BB-EDGE controls $\operatorname{FWER}_r$ with the highest power, since the weight-proportional stakes require prespecified independent blocks but neither equal sizes nor equal correlations (Theorem~\ref{thm:bedge-main-iclr}).

\subsection{Real Data Analysis}
\label{section:real_data}
\paragraph{Models and Datasets.}
We evaluate BB-EDGE on four benchmarks: knowledge reasoning on MMLU-Pro~\citep{wang2024mmlu}; mathematical reasoning on GSM8K~\citep{cobbe2021gsm8k}; multi-turn dialogues on MT-Bench~\citep{zheng2023judging}; and multi-turn tool calling on ACEBench~\citep{chen2025acebench}.
We use accuracy for MMLU-Pro, exact match for GSM8K, AST-based accuracy for ACEBench, and the LLM-judge score for MT-Bench~\citep{zheng2023judging} as the evaluation metrics.
For MT-Bench, we use MiniMax-M3~\citep{lai2026minimax} as the LLM judge and rescale its ratings to $[0,1]$.
 We evaluate ten open-source LLMs: DeepSeek-R1-Distill-Qwen-7B \citep{guo2025deepseek}, Falcon-H1R-7B \citep{chaabane2026falcon}, Gemma-4-E4B-it \citep{team2026gemma}, Granite-4.1-8B \citep{padhi2024graniteguardian}, Intern-S1-Mini \citep{bai2025intern}, Llama-3.1-8B-Instruct \citep{grattafiori2024llama}, MiMo-7B-RL \citep{xiaomi2025mimo}, Mistral-7B-Instruct-v0.3 \citep{jiang2023mistral7b}, Phi-4 \citep{abdin2024phi}, and Qwen3.5-9B \citep{qwen3.5}. Details of the datasets, prompt templates, block sizes, and other protocol configurations are provided in Appendix~\ref{section_experiments_details}.

\textbf{Results.}
We present two representative examples---GSM8K and ACEBench---and defer the MMLU-Pro, MT-Bench, and API results to Appendix~\ref{section_read_data_results} due to space constraints. 
Figure~\ref{figure:2} presents the confidence graphs for GSM8K and ACEBench, while Tables~\ref{table:a3} and~\ref{table:a5} report exact match and AST-based accuracy, respectively. 
On GSM8K, the confidence graph contains 10 nodes and is acyclic.
Granite-4.1-8B occupies the top position and is certified above Phi-4, followed by Qwen3.5-9B and Intern-S1-Mini.
Gemma-4-E4B-It, Llama-3.1-8B-Instruct, and MiMo-7B-RL remain mutually unresolved, while all three are certified above DeepSeek-R1-Distill-Qwen-7B.
The graph then forms a strictly resolved chain from DeepSeek-R1-Distill-Qwen-7B through Mistral-7B-Instruct-v0.3 to Falcon-H1R-7B.
Additionally, when we increase the prespecified margin from $\tau=0$ to $\tau=0.01$, Figure~\ref{figure:real_data_tau} shows that the edge from Granite-4.1-8B to Phi-4 disappears because their exact-match difference, $0.0071$, does not exceed the required margin. 
On ACEBench, Gemma-4-E4B-It and Qwen3.5-9B have the highest mean scores, but BB-EDGE does not certify either model as better than the other.
Both models are certified above Intern-S1-Mini and Phi-4, followed by a descending chain ending with Falcon-H1R-7B.

\begin{figure*}[t]
\centering
\subfloat[GSM8K]{%
    \includegraphics[width=0.49\textwidth]{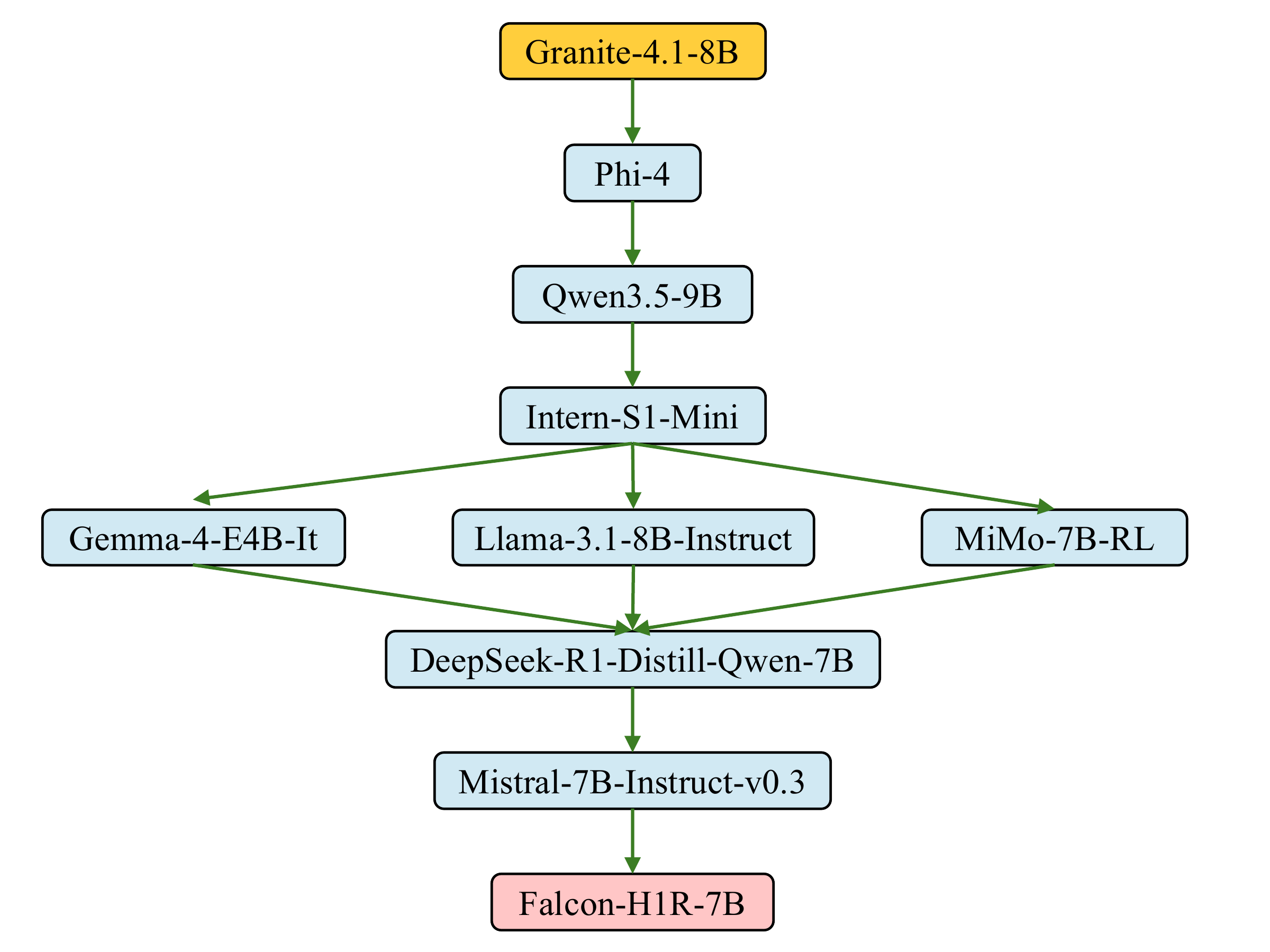}
}
\subfloat[ACEBench]{%
    \includegraphics[width=0.49\textwidth]{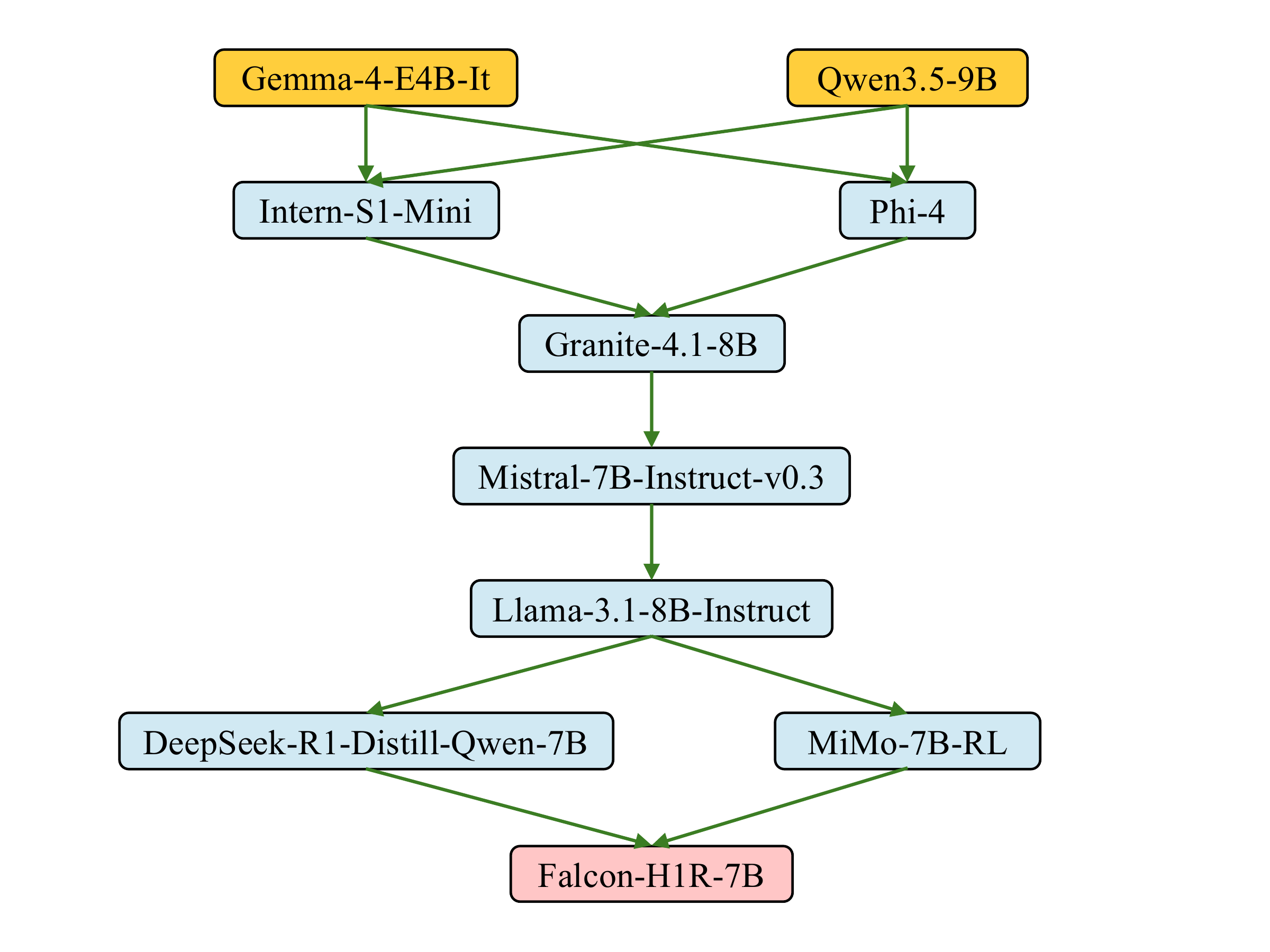}
}
\caption{Certified comparison graphs at replicate $r=20$ on GSM8K and ACEBench, respectively.
An edge \(A\to B\) indicates that model $A$ is certified to outperform model $B$ on the benchmark under the evaluation protocol. Pairs not connected by a directed path remain unresolved.}
\label{figure:2}
\end{figure*}

Figure~\ref{figure:3} reports a same-model null diagnostic on GSM8K and ACEBench. For each model, we randomly split its 20 observed replicates into two groups of 10 and compare them as pseudo-models. Both groups have the same expected score on every item, so any directional certification is false. Across these splits, the SERPANT adaptation produces more false certifications than BB-EDGE. Because this diagnostic reuses the same 20 replicates, it does not estimate repeated-evaluation FWER or isolate item-level heterogeneity; Section~\ref{section:simulation_studies} examines the latter through simulations.
Meanwhile, Figure~\ref{figure:3} also presents that BB-EDGE exhibits substantially faster gains under small evaluation budgets. 
Under the prespecified singleton-block analysis for GSM8K, BB-EDGE therefore uses 66.7\% fewer queries than BFH-$e$-Holm (six replicates, 79,140 queries), the fastest compared baseline with an anytime FWER guarantee for the same benchmark-average hypotheses.
Appendix~\ref{sec:real_data_topk} reports that BB-EDGE certifies Top-$k$ sets with fewer replicates than the baselines. 
For example, on ACEBench, BB-EDGE certifies the Top-4 set after 4 replicates, whereas BFH-$e$-Holm requires 17 and CR-EDGE issues no certificate within 20 replicates.
Appendix~\ref{sec:real_data_interval} reports rank intervals, which are narrower for BB-EDGE and indicate less uncertainty about model rankings.

\begin{figure*}[t]
\centering
\subfloat[GSM8K]{%
    \includegraphics[width=0.49\textwidth]{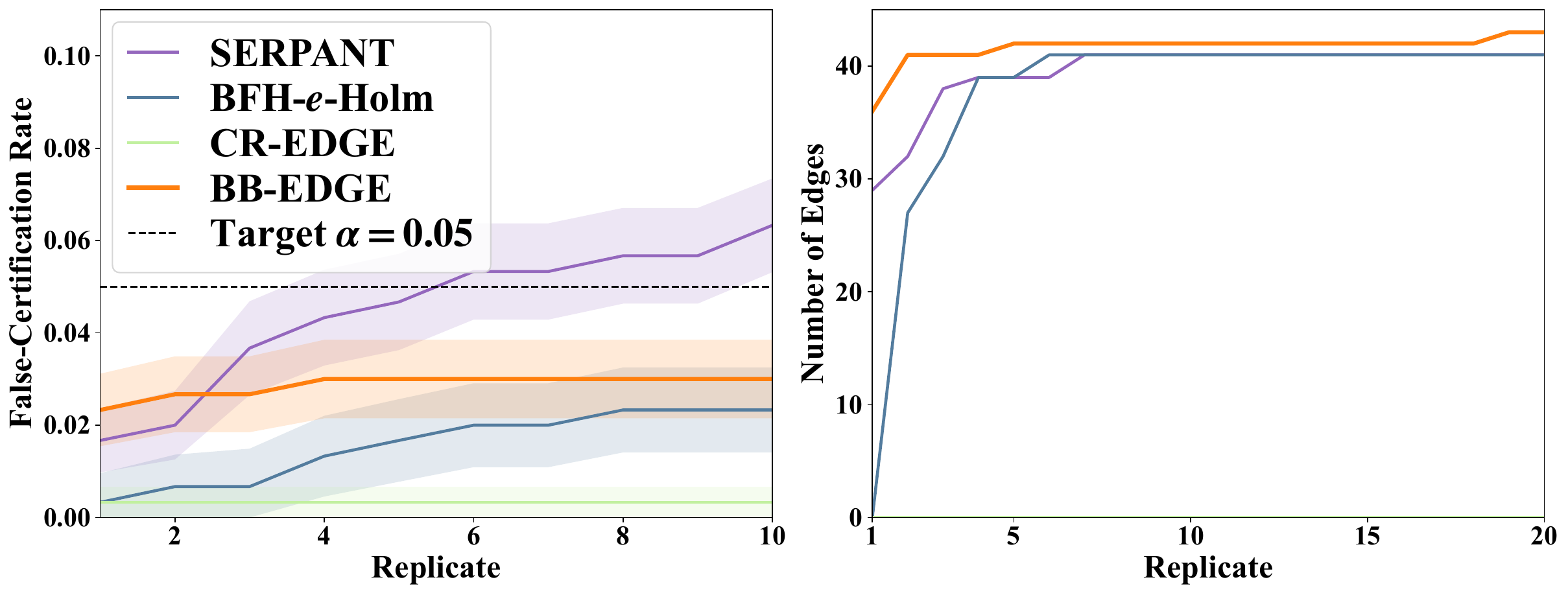}
}
\subfloat[ACEBench]{%
    \includegraphics[width=0.49\textwidth]{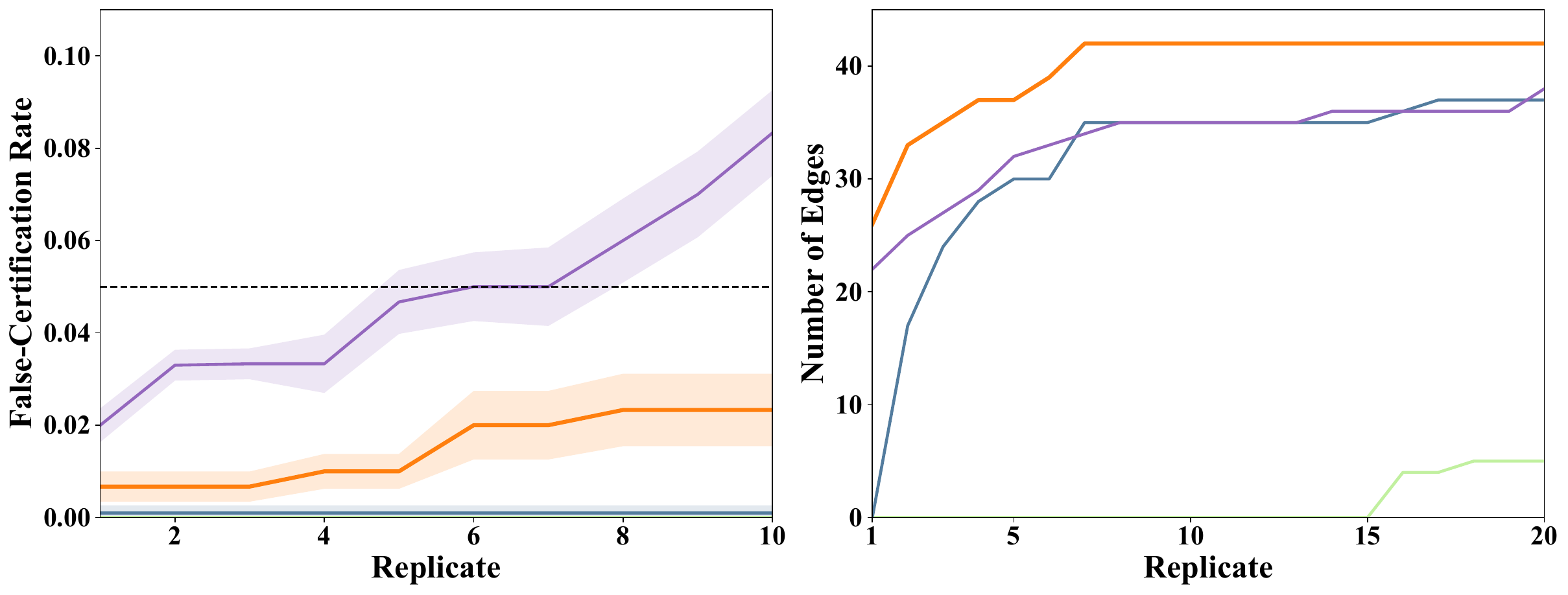}
}
\caption{
Same-model null diagnostic and certification efficiency on GSM8K and ACEBench after $r$ replicates.
Left: mean false-certification rate across 100 random 10-versus-10 splits of the same 20 observed replicates; Monte Carlo interval across random splits, conditional on the observed replicates.
Right: number of distinct model pairs resolved by reachability in the current confidence graph after each complete replicate. At most $\binom{10}{2}=45$ pairs can be resolved.
}
\label{figure:3}
\end{figure*}

\section{Conclusion}
This paper proposes BB-EDGE, an anytime-valid and efficient method for certifying model mean-performance advantages on fixed-benchmark leaderboards. 
Our approach begins by partitioning the benchmark according to the evaluation protocol. 
We then construct block-factorized empirical-Bernstein $e$-processes with weight-proportional stakes and apply direct $e$-Holm to certify mean-performance advantages. 
Furthermore, we provide a theoretical analysis demonstrating that BB-EDGE controls the probability of ever reporting a false edge across all model pairs and replicates at level $\alpha$. 
Extensive experiments on both synthetic and real-world benchmarks demonstrate that BB-EDGE achieves anytime-valid FWER control with the highest statistical power. 

\textbf{Limitation.} BB-EDGE's anytime-valid guarantees are conditional on the fixed benchmark and protocol, and do not extend to out-of-distribution tasks. 
Developing item-level anytime-valid inference that accommodates heterogeneous item effects and cross-item dependence is left for future work.




\clearpage

\begingroup
\marknewreferences
\bibliography{iclr2027_conference,proof_references}

@article{chandrahas2026evalci,
  title={evalci: A Python Library for Statistically Rigorous Comparison of Language Model Evaluations},
  author={Chandrahas, Shreyas K},
  journal={arXiv preprint arXiv:2607.04429},
  year={2026}
}

@inproceedings{peng2024humaneval,
  title={Humaneval-xl: A multilingual code generation benchmark for cross-lingual natural language generalization},
  author={Peng, Qiwei and Chai, Yekun and Li, Xuhong},
  booktitle={Proceedings of the 2024 joint international conference on computational linguistics, language resources and evaluation (LREC-COLING 2024)},
  pages={8383--8394},
  year={2024}
}

@article{wang2024confidence,
  title={Confidence Diagram of Nonparametric Ranking for Uncertainty Assessment in Large Language Models Evaluation},
  author={Wang, Zebin and Han, Yi and Fang, Ethan X and Wang, Lan and Lu, Junwei},
  journal={arXiv preprint arXiv:2412.05506},
  year={2024}
}

@inproceedings{chen2025acebench,
    title = "{ACEB}ench: A Comprehensive Evaluation of {LLM} Tool Usage",
    author = "Chen, Chen  and
      Hao, Xinlong  and
      Liu, Weiwen  and
      Huang, Xu  and
      Zeng, Xingshan  and
      Yu, Shuai  and
      Li, Dexun  and
      Huang, Yuefeng  and
      Liu, Xiangcheng  and
      Xinzhi, Wang  and
      Liu, Wu",
      booktitle = "Findings of the Association for Computational Linguistics: EMNLP 2025",
    year = "2025",
    pages = "12970--12998",
}

@inproceedings{zhu2024dyval,
  title={Dyval: Dynamic evaluation of large language models for reasoning tasks},
  author={Zhu, Kaijie and Chen, Jiaao and Wang, Jindong and Gong, Neil and Yang, Diyi and Xie, Xing},
  booktitle={International Conference on Learning Representations},
  volume={2024},
  pages={18091--18128},
  year={2024}
}

@inproceedings{
frick2025prompttoleaderboard,
title={Prompt-to-Leaderboard: Prompt-Adaptive {LLM} Evaluations},
author={Evan Frick and Connor Chen and Joseph Tennyson and Tianle Li and Wei-Lin Chiang and Anastasios Nikolas Angelopoulos and Ion Stoica},
booktitle={Forty-second International Conference on Machine Learning},
year={2025},
}

@article{madaan2024quantifying,
  title={Quantifying variance in evaluation benchmarks},
  author={Madaan, Lovish and Singh, Aaditya K and Schaeffer, Rylan and Poulton, Andrew and Koyejo, Sanmi and Stenetorp, Pontus and Narang, Sharan and Hupkes, Dieuwke},
  journal={arXiv preprint arXiv:2406.10229},
  year={2024}
}

@article{biderman2024lessons,
  title={Lessons from the trenches on reproducible evaluation of language models},
  author={Biderman, Stella and Schoelkopf, Hailey and Sutawika, Lintang and Gao, Leo and Tow, Jonathan and Abbasi, Baber and Aji, Alham Fikri and Ammanamanchi, Pawan Sasanka and Black, Sidney and Clive, Jordan and others},
  journal={arXiv preprint arXiv:2405.14782},
  year={2024}
}

@article{ye2024benchmarking,
  title={Benchmarking llms via uncertainty quantification},
  author={Ye, Fanghua and Yang, Mingming and Pang, Jianhui and Wang, Longyue and Wong, Derek F and Yilmaz, Emine and Shi, Shuming and Tu, Zhaopeng},
  journal={Advances in Neural Information Processing Systems},
  volume={37},
  pages={15356--15385},
  year={2024}
}

@misc{qwen3.5,
    title  = {{Qwen3.5}: Towards Native Multimodal Agents},
    author = {{Qwen Team}},
    month  = {February},
    year   = {2026},
    url    = {https://qwen.ai/blog?id=qwen3.5}
}

@article{li2026low,
  title={Low Rank for Rank: Uncertainty-Aware Task-Specific LLM Ranking under Sparse Pairwise Comparisons},
  author={Li, Jiachun and Simchi-Levi, David and Sun, Will Wei},
  journal={arXiv preprint arXiv:2605.29395},
  year={2026}
}

@article{abdin2024phi,
  title={Phi-4 technical report},
  author={Abdin, Marah and Aneja, Jyoti and Behl, Harkirat and Bubeck, S{\'e}bastien and Eldan, Ronen and Gunasekar, Suriya and Harrison, Michael and Hewett, Russell J and Javaheripi, Mojan and Kauffmann, Piero and others},
  journal={arXiv preprint arXiv:2412.08905},
  year={2024}
}

@misc{jiang2023mistral7b,
      title={Mistral 7B}, 
      author={Albert Q. Jiang and Alexandre Sablayrolles and Arthur Mensch and Chris Bamford and Devendra Singh Chaplot and Diego de las Casas and Florian Bressand and Gianna Lengyel and Guillaume Lample and Lucile Saulnier and Lélio Renard Lavaud and Marie-Anne Lachaux and Pierre Stock and Teven Le Scao and Thibaut Lavril and Thomas Wang and Timothée Lacroix and William El Sayed},
      year={2023},
      eprint={2310.06825},
      archivePrefix={arXiv},
      primaryClass={cs.CL},
      url={https://arxiv.org/abs/2310.06825}, 
}

@article{xiaomi2025mimo,
  title={MiMo: Unlocking the Reasoning Potential of Language Model--From Pretraining to Posttraining},
  author={Xiaomi, LLM and Xia, Bingquan and Shen, Bowen and Zhu, Dawei and Zhang, Di and Wang, Gang and Zhang, Hailin and Liu, Huaqiu and Xiao, Jiebao and Dong, Jinhao and others},
  journal={arXiv preprint arXiv:2505.07608},
  year={2025}
}

@article{grattafiori2024llama,
  title={The llama 3 herd of models},
  author={Grattafiori, Aaron and Dubey, Abhimanyu and Jauhri, Abhinav and Pandey, Abhinav and Kadian, Abhishek and Al-Dahle, Ahmad and Letman, Aiesha and Mathur, Akhil and Schelten, Alan and Vaughan, Alex and others},
  journal={arXiv preprint arXiv:2407.21783},
  year={2024}
}

@article{bai2025intern,
  title={Intern-s1: A scientific multimodal foundation model},
  author={Bai, Lei and Cai, Zhongrui and Cao, Yuhang and Cao, Maosong and Cao, Weihan and Chen, Chiyu and Chen, Haojiong and Chen, Kai and Chen, Pengcheng and Chen, Ying and others},
  journal={arXiv preprint arXiv:2508.15763},
  year={2025}
}

@misc{padhi2024graniteguardian,
      title={Granite Guardian}, 
      author={Inkit Padhi and Manish Nagireddy and Giandomenico Cornacchia and Subhajit Chaudhury and Tejaswini Pedapati and Pierre Dognin and Keerthiram Murugesan and Erik Miehling and Martín Santillán Cooper and Kieran Fraser and Giulio Zizzo and Muhammad Zaid Hameed and Mark Purcell and Michael Desmond and Qian Pan and Zahra Ashktorab and Inge Vejsbjerg and Elizabeth M. Daly and Michael Hind and Werner Geyer and Ambrish Rawat and Kush R. Varshney and Prasanna Sattigeri},
      year={2024},
      eprint={2412.07724},
      archivePrefix={arXiv},
      primaryClass={cs.CL},
      url={https://arxiv.org/abs/2412.07724}, 
}

@article{team2026gemma,
  title={Gemma 4 technical report},
  author={Team, Gemma and Abd, Sherif El and Aggarwal, Vaibhav and Algayres, Robin and Andreev, Alek and Bachem, Olivier and Ballantyne, Ian and Brick, Cormac and C{\u{a}}rbune, Victor and Casbon, Michelle and others},
  journal={arXiv preprint arXiv:2607.02770},
  year={2026}
}

@article{chaabane2026falcon,
  title={Falcon-H1R: Pushing the reasoning frontiers with a hybrid model for efficient test-time scaling},
  author={Chaabane, Iheb and Khanna, Puneesh and Mohmad, Suhail and Frikha, Slim and Hu, Shi and Abubaker, Abdalgader and Alami, Reda and Lubinets, Mikhail and Seddik, Mohamed El Amine and Hacid, Hakim and others},
  journal={arXiv preprint arXiv:2601.02346},
  year={2026}
}

@article{guo2025deepseek,
  title={Deepseek-r1: Incentivizing reasoning capability in llms via reinforcement learning},
  author={Guo, Daya and Yang, Dejian and Zhang, Haowei and Song, Junxiao and Wang, Peiyi and Zhu, Qihao and Xu, Runxin and Zhang, Ruoyu and Ma, Shirong and Bi, Xiao and others},
  journal={arXiv preprint arXiv:2501.12948},
  year={2025}
}

@article{zhou2026celeus,
  title={CELEUS: Certifiable and Efficient LLM Evaluation via E-Processes},
  author={Zhou, Zhijian and Ye, Zesheng and Chen, Zhaorun and Li, Bo and Liu, Feng},
  journal={arXiv preprint arXiv:2606.20820},
  year={2026}
}

@article{menendez2026prompt,
  title={Prompt-Dependent Ranking of Large Language Models with Uncertainty Quantification},
  author={Menendez, Angel Rodrigo Avelar and Liu, Yufeng and Dai, Xiaowu},
  journal={arXiv preprint arXiv:2603.03336},
  year={2026}
}

@inproceedings{huang2026dropping,
  title={Dropping just a handful of preferences can change top large language model rankings},
  author={Huang, Jenny and Shen, Yunyi and Wei, Dennis and Broderick, Tamara},
  booktitle={International Conference on Learning Representations},
  volume={2026},
  pages={90685--90721},
  year={2026}
}

@article{wei2024diff,
  title={Diff-erank: A novel rank-based metric for evaluating large language models},
  author={Wei, Lai and Tan, Zhiquan and Li, Chenghai and Wang, Jindong and Huang, Weiran},
  journal={Advances in Neural Information Processing Systems},
  volume={37},
  pages={39501--39521},
  year={2024}
}

@article{du2026stability,
  title={On the Stability of Prompt Ranking in Large Language Model Evaluation},
  author={Du, Shaoshuai and Liang, Penghao and Shen, Yixian and Shi, Chuanqi and Zhang, Hang and Wang, Lun},
  journal={arXiv preprint arXiv:2606.24381},
  year={2026}
}

@article{chatzi2024predictionpowered,
  title={Prediction-powered ranking of large language models},
  author={Chatzi, Ivi and Straitouri, Eleni and Thejaswi, Suhas and Rodriguez, Manuel G},
  journal={Advances in Neural Information Processing Systems},
  pages={113096--113133},
  year={2024}
}

@article{wang2025rankings,
  title   = {From Rankings to Insights: Evaluation Should Shift Focus from Leaderboard to Feedback},
  author  = {Wang, Zongqi and Gu, Tianle and Gong, Chen and Tian, Xin and Bao, Siqi and Yang, Yujiu},
  journal = {arXiv preprint arXiv:2505.06698},
  year    = {2025}
}

@article{gonzalez2025repetitions,
  title={Do repetitions matter? Strengthening reliability in LLM evaluations},
  author={Gonzalez, Miguel Angel Alvarado and Hernandez, Michelle Bruno and Perez, Miguel Angel Pe{\~n}aloza and Orozco, Bruno Lopez and Soto, Jesus Tadeo Cruz and Malagon, Sandra},
  journal={arXiv preprint arXiv:2509.24086},
  year={2025}
}

@article{li2026llm,
  title={LLM Evaluation as Tensor Completion: Low Rank Structure and Semiparametric Efficiency},
  author={Li, Jiachun and Simchi-Levi, David and Sun, Will Wei},
  journal={arXiv preprint arXiv:2604.05460},
  year={2026}
}

@misc{openai2026gpt56luna,
  author       = {{OpenAI}},
  title        = {{GPT-5.6 Luna Model}},
  year         = {2026},
  howpublished = {\url{https://developers.openai.com/api/docs/models/gpt-5.6-luna}},
  note         = {Accessed: 2026-09-09}
}

@misc{qwen3.8flashnext,
    title  = {{Qwen3.8-Flash-Next}: A New Architecture, Towards Ultimate Cost-Efficiency},
    author = {{Qwen Team}},
    month  = {August},
    year   = {2026},
    url    = {https://qwen.ai/blog?id=qwen3.8-flash-next}
}

@misc{glm5team2026glm5vibecodingagentic,
      title={GLM-5: from Vibe Coding to Agentic Engineering},
      author={GLM-5-Team},
      year={2026},
      eprint={2602.15763},
      archivePrefix={arXiv},
      primaryClass={cs.LG},
      url={https://arxiv.org/abs/2602.15763},
}

@article{xu2026deepseek,
  title={Deepseek-v4: Towards highly efficient million-token context intelligence},
  author={Xu, Anyi and Lin, Bangcai and Xue, Bing and Wang, Bingxuan and Xu, Bingzheng and Wu, Bochao and Zhang, Bowei and Lin, Chaofan and Dong, Chen and Ling, Chenchen and others},
  journal={arXiv preprint arXiv:2606.19348},
  year={2026}
}

@article{vashurin2025benchmarking,
  title={Benchmarking uncertainty quantification methods for large language models with lm-polygraph},
  author={Vashurin, Roman and Fadeeva, Ekaterina and Vazhentsev, Artem and Rvanova, Lyudmila and Vasilev, Daniil and Tsvigun, Akim and Petrakov, Sergey and Xing, Rui and Sadallah, Abdelrahman and Grishchenkov, Kirill and others},
  journal={Transactions of the Association for Computational Linguistics},
  volume={13},
  pages={220--248},
  year={2025},
  publisher={MIT Press 255 Main Street, 9th Floor, Cambridge, Massachusetts 02142, USA~…}
}

@article{fan2025ranking,
  title={Ranking inferences based on the top choice of multiway comparisons},
  author={Fan, Jianqing and Lou, Zhipeng and Wang, Weichen and Yu, Mengxin},
  journal={Journal of the American Statistical Association},
  volume={120},
  number={549},
  pages={237--250},
  year={2025},
  publisher={Taylor \& Francis}
}

@article{lai2026minimax,
  title={Minimax sparse attention},
  author={Lai, Xunhao and Xu, Weiqi and Yang, Yufeng and Chen, Qiaorui and Xu, Yang and Zeng, Lunbin and Li, Xiaolong and Sun, Haohai and Zhu, Haichao and Zhang, Vito and others},
  journal={arXiv preprint arXiv:2606.13392},
  year={2026}
}

@article{zheng2023judging,
  title={Judging llm-as-a-judge with mt-bench and chatbot arena},
  author={Zheng, Lianmin and Chiang, Wei-Lin and Sheng, Ying and Zhuang, Siyuan and Wu, Zhanghao and Zhuang, Yonghao and Lin, Zi and Li, Zhuohan and Li, Dacheng and Xing, Eric and others},
  journal={Advances in Neural Information Processing Systems},
  volume={36},
  pages={46595--46623},
  year={2023}
}

@misc{eval-harness,
  author       = {Gao, Leo and Tow, Jonathan and Abbasi, Baber and Biderman, Stella and Black, Sid and DiPofi, Anthony and Foster, Charles and Golding, Laurence and Hsu, Jeffrey and Le Noac'h, Alain and Li, Haonan and McDonell, Kyle and Muennighoff, Niklas and Ociepa, Chris and Phang, Jason and Reynolds, Laria and Schoelkopf, Hailey and Skowron, Aviya and Sutawika, Lintang and Tang, Eric and Thite, Anish and Wang, Ben and Wang, Kevin and Zou, Andy},
  title        = {The Language Model Evaluation Harness},
  month        = 07,
  year         = 2024,
  publisher    = {Zenodo},
  version      = {v0.4.3},
  doi          = {10.5281/zenodo.12608602},
  url          = {https://zenodo.org/records/12608602}
}

@inproceedings{
gu2026anytimevalid,
title={Anytime-Valid Inference for Online Ranking of Large Language Models},
author={Runzhe Gu and Wenguang Sun and Bowen Gang and Xintao Xia},
booktitle={Forty-third International Conference on Machine Learning},
year={2026},
}

@article{hsu2026efficient,
  title={Efficient Sequential Evaluation of Large Language Models},
  author={Hsu, Chia-Yu and Shekhar, Shubhanshu},
  journal={arXiv preprint arXiv:2607.17409},
  year={2026}
}

@inproceedings{zhou2026lost,
  title={Lost in benchmarks? rethinking large language model benchmarking with item response theory},
  author={Zhou, Hongli and Huang, Hui and Zhao, Ziqing and Han, Lvyuan and Wang, Huicheng and Chen, Kehai and Yang, Muyun and Bao, Wei and Dong, Jian and Xu, Bing and others},
  booktitle={Proceedings of the AAAI Conference on Artificial Intelligence},
  volume={40},
  number={41},
  pages={35085--35093},
  year={2026}
}

@article{smith2024estimating,
    author = {Waudby-Smith, Ian and Ramdas, Aaditya},
    title = {Estimating means of bounded random variables by betting},
    journal = {Journal of the Royal Statistical Society Series B: Statistical Methodology},
    volume = {86},
    number = {1},
    pages = {1-27},
    year = {2024},
}

@article{neuhof2026rank,
  title={Rank Intervals for Leaderboards: A Hierarchical Framework for Model Evaluation},
  author={Neuhof, Bitya and Benjamini, Yuval},
  journal={arXiv preprint arXiv:2606.08679},
  year={2026}
}

@article{cobbe2021gsm8k,
  title={Training Verifiers to Solve Math Word Problems},
  author={Cobbe, Karl and Kosaraju, Vineet and Bavarian, Mohammad and Chen, Mark and Jun, Heewoo and Kaiser, Lukasz and Plappert, Matthias and Tworek, Jerry and Hilton, Jacob and Nakano, Reiichiro and Hesse, Christopher and Schulman, John},
  journal={arXiv preprint arXiv:2110.14168},
  year={2021}
}

@inproceedings{nakash2026efficient,
  title={Efficient Agent Evaluation via Diversity-Guided User Simulation},
  author={Nakash, Itay and Kour, George and Tavor, Ateret Anaby},
  booktitle={Proceedings of the 64th Annual Meeting of the Association for Computational Linguistics (ACL 2026)},
  pages={1627--1648},
  year={2026}
}

@inproceedings{chen2024humans,
  title={Humans or LLMs as the judge? a study on judgement bias},
  author={Chen, Guiming Hardy and Chen, Shunian and Liu, Ziche and Jiang, Feng and Wang, Benyou},
  booktitle={Proceedings of the 2024 Conference on Empirical Methods in Natural Language Processing},
  pages={8301--8327},
  year={2024}
}

@article{wang2025cer,
  title={Cer-eval: Certifiable and cost-efficient evaluation framework for llms},
  author={Wang, Ganghua and Chen, Zhaorun and Li, Bo and Xu, Haifeng},
  journal={arXiv preprint arXiv:2505.03814},
  year={2025}
}

@article{haghtalab2026pluralistic,
  title={Pluralistic Leaderboards},
  author={Haghtalab, Nika and Procaccia, Ariel D and Shao, Han and Wang, Serena Lutong and Yang, Kunhe},
  journal={arXiv preprint arXiv:2606.02547},
  year={2026}
}

@inproceedings{divekar2026precise,
  title={PRECISE: Reducing the bias of LLM evaluations using prediction-powered ranking estimation},
  author={Divekar, Abhishek and Majumder, Anirban},
  booktitle={Proceedings of the AAAI Conference on Artificial Intelligence},
  pages={39929--39938},
  year={2026}
}

@article{singh2026leaderboard,
  title={The leaderboard illusion},
  author={Singh, Shivalika and Nan, Yiyang and Wang, Alex and Dsouza, Daniel and Kapoor, Sayash and {\"U}st{\"u}n, Ahmet and Koyejo, Sanmi and Deng, Yuntian and Longpre, Shayne and Smith, Noah and others},
  journal={Advances in Neural Information Processing Systems},
  volume={38},
  year={2026}
}

@article{ramdas2023game,
  title={Game-theoretic statistics and safe anytime-valid inference},
  author={Ramdas, Aaditya and Gr{\"u}nwald, Peter and Vovk, Vladimir and Shafer, Glenn},
  journal={Statistical Science},
  volume={38},
  number={4},
  pages={576--601},
  year={2023},
  publisher={Institute of Mathematical Statistics}
}

@article{xu2024active,
  title={Active, anytime-valid risk controlling prediction sets},
  author={Xu, Ziyu and Karampatziakis, Nikos and Mineiro, Paul},
  journal={Advances in Neural Information Processing Systems},
  volume={37},
  pages={60110--60132},
  year={2024}
}

@inproceedings{hariri2026don,
  title={Don’t pass@ k: A bayesian framework for large language model evaluation},
  author={Hariri, Mohsen and Samandar, Amirhossein and Hinczewski, Michael and Chaudhary, Vipin},
  booktitle={International Conference on Learning Representations},
  volume={2026},
  pages={148539--148579},
  year={2026}
}

@article{chiang2024chatbot,
  title={Chatbot arena: An open platform for evaluating llms by human preference},
  author={Chiang, Wei-Lin and Zheng, Lianmin and Sheng, Ying and Angelopoulos, Anastasios Nikolas and Li, Tianle and Li, Dacheng and Zhang, Hao and Zhu, Banghua and Jordan, Michael and Gonzalez, Joseph E and others},
  journal={arXiv preprint arXiv:2403.04132},
  year={2024}
}

@article{wang2024mmlu,
  title={Mmlu-pro: A more robust and challenging multi-task language understanding benchmark},
  author={Wang, Yubo and Ma, Xueguang and Zhang, Ge and Ni, Yuansheng and Chandra, Abhranil and Guo, Shiguang and Ren, Weiming and Arulraj, Aaran and He, Xuan and Jiang, Ziyan and others},
  journal={Advances in Neural Information Processing Systems},
  volume={37},
  pages={95266--95290},
  year={2024}
}

@article{hartoglei2025,
  title={Family-wise error rate control with e-values},
  author={Hartog, Will and Lei, Lihua},
  journal={arXiv preprint arXiv:2501.09015},
  year={2025}
}

@article{howard2020,
  author = {Howard, Steven R. and Ramdas, Aaditya and McAuliffe, Jon and Sekhon, Jasjeet},
  title = {Time-uniform {Chernoff} bounds via nonnegative supermartingales},
  journal = {Probability Surveys},
  volume = {17},
  pages = {257--317},
  year = {2020},
  doi = {10.1214/18-PS321}
}

@article{fan2015exponential,
  author = {Fan, Xiequan and Grama, Ion and Liu, Quansheng},
  title = {Exponential inequalities for martingales with applications},
  journal = {Electronic Journal of Probability},
  volume = {20},
  number = {1},
  pages = {1--22},
  year = {2015},
  doi = {10.1214/EJP.v20-3496}
}

@article{hoeffding1963,
  author = {Hoeffding, Wassily},
  title = {Probability inequalities for sums of bounded random variables},
  journal = {Journal of the American Statistical Association},
  volume = {58},
  number = {301},
  pages = {13--30},
  year = {1963},
  doi = {10.1080/01621459.1963.10500830}
}
\endgroup
\bibliographystyle{iclr2027_conference}

\clearpage
\appendix

\section*{Appendix}



\section{Anytime-Valid Inference for Top-$k$ Certification}
\label{sec:pilot-topk}
When the number of models is large, we may care only about identifying the Top-$k$ models, motivating a more efficient procedure that tests only the relevant cross-set comparisons.
Specifically, write \(a\rightsquigarrow_r b\) if model \(b\) is reachable from model \(a\) in the current confidence graph \(\mathcal G_r=(\mathcal V,\mathcal E_r)\).
A set $\mathcal{T}\subseteq\cV$ with $|\mathcal{T}|=k$ is certified as a Top-$k$ set at replicate $r$ if
\[
a\leadsto_r b\qquad\text{for every }a\in\mathcal{T}\text{ and }b\in\cV\setminus\mathcal{T}.
\]
On the no-false-edge event in (\ref{eq:anytime-fwer-target}), every path from $a$ to $b$ has length $d\ge1$ and satisfies $\theta_a^{\cB}-\theta_b^{\cB}>d\tau\ge\tau$. Hence Theorem~\ref{thm:bedge-main-iclr} implies that, with probability at least $1-\alpha$, every Top-$k$ certificate issued from the full graph by replicate $R$ satisfies $\min_{a\in\mathcal T}\theta_a^{\cB}>\max_{b\notin\mathcal T}\theta_b^{\cB}+\tau$. This conclusion permits searching the certified graph for a set; it does not permit selecting a new testing family from the same confirmatory data.
No ordering is asserted among the models within $\mathcal{T}$ or within $\cV\setminus\mathcal{T}$; the certificate concerns only the between-set comparison.

Our procedure consists of three steps: selecting a candidate set using independent pilot replicates, testing the corresponding cross-set directions, and certifying the candidate set using confirmatory evidence.
Fix a pilot sample size $R_0\ge1$. For pilot run $s$, let
$Z_s^{(0)}=(S_{i,s}^{(\ell,0)}:1\le i\le N,\ 1\le\ell\le L)$
contain the scores of all models on all benchmark items, and write
$\mathcal D_0=(Z_1^{(0)},\ldots,Z_{R_0}^{(0)})$.
These pilot data are independent of the subsequent confirmatory runs
$(Z_1,\ldots,Z_R)$ conditional on $(\cB,\pi)$.
The pre-inference information in Assumption~\ref{ass:design-iclr} comprises
$\mathcal D_0$ and all design choices fixed before confirmatory sampling.

\paragraph{Step 1: Candidate selection.}
Compute the pilot mean of model $\ell$ as
\begin{equation}
\widehat\theta_\ell^{(0)}
=\frac1{NR_0}\sum_{s=1}^{R_0}\sum_{i=1}^N S_{i,s}^{(\ell,0)}.
\label{eq:pilot-score}
\end{equation}
Order the models by decreasing $\widehat\theta_\ell^{(0)}$, breaking
equal pilot means by increasing model index in $\cV=\{1,\ldots,L\}$.
Let $\mathcal T_0$ contain the first $k$ models in this order.
The model indices are fixed before the pilot, and $\mathcal T_0$ remains
fixed throughout the confirmatory phase. Pilot scores determine this
candidate set but are not included as observations in the products
in (\ref{eq:block-mixture}), which start from one at confirmatory run zero.

\paragraph{Step 2: Cross-set testing.} 
During the confirmatory evaluation, we test whether each model in $\mathcal{T}_0$ outperforms every model outside $\mathcal{T}_0$ by more than the prespecified margin $\tau$. Accordingly, we restrict the directional hypothesis family to
\[
\cH(\mathcal{T}_0)=\left\{H_e^{(\tau)}\;\middle|\;e=(a\to b),\ a\in\mathcal{T}_0,\ b\in\cV\setminus\mathcal{T}_0\right\},
\]
which contains the $k(L-k)$ cross-set directions required to certify $\mathcal{T}_0$.

\paragraph{Step 3: Top-$k$ certification.}
After each completed replicate, we update $E_{e,r}^{\cC}$ for every $H_e^{(\tau)}\in\cH(\mathcal{T}_0)$ and apply direct $e$-Holm at level $\alpha$ to this family. 
We certify $\mathcal{T}_0$ as the Top-$k$ set once all hypotheses in $\cH(\mathcal{T}_0)$ are rejected; otherwise, the result remains unresolved.

The pilot determines only the targeted hypothesis family, while all certification evidence is obtained from the independent confirmatory replicates. 
Theorem \ref{thm:pilot-topk} shows that restricting inference to this pilot-selected family preserves anytime FWER control and yields a valid Top-$k$ certificate, whose proof is given in Appendix~\ref{app:proof-pilot-topk}.

\begin{theorem}
\label{thm:pilot-topk}
Under Assumption~\ref{ass:design-iclr}, the pilot-targeted procedure with
$1\le k<L$ satisfies
\begin{equation}
\Pp\!\left(
\exists r\in\{1,\ldots,R\},\ \exists(a\to b)\in\cE_r
\text{ with }\Delta_{a,b}^{\cB}\le\tau
\,\middle|\,\mathcal D_0,\cB,\pi
\right)\le\alpha
\quad\text{a.s.}
\label{eq:topk-anytime-fwer}
\end{equation}
With conditional probability at least $1-\alpha$ given
$(\mathcal D_0,\cB,\pi)$, every certificate issued by replicate $R$ therefore satisfies
\[
\min_{a\in\mathcal T_0}\theta_a^{\cB}
>\max_{b\in\cV\setminus\mathcal T_0}\theta_b^{\cB}+\tau.
\]
\end{theorem}
Theorem~\ref{thm:pilot-topk} controls the false Top-$k$ certification probability, denoted by $\operatorname{FCP}_r$ in Appendix~\ref{section:evaluation_metric}, for the fixed pilot candidate. The bound holds for almost every pilot outcome and after averaging over the pilot. An incorrect pilot candidate is allowed, but its probability of being certified by replicate $R$ is at most $\alpha$. This does not bound the error probability conditional on a certificate being issued. Equation~(\ref{eq:pilot-certificate-error}) and the calculation following it make the distinction explicit. Full-graph certification instead uses Theorem~\ref{thm:bedge-main-iclr} and the path implication above; it is not an application of the reduced-family pilot procedure.

Table~\ref{table:1} reports the Top-4 certification results under the i.i.d., heterogeneity, dependence, and combined settings.
Both BB-EDGE variants maintain estimated $\operatorname{FWER}_{R}=\operatorname{FCP}_{R}=0$ throughout. 
Although several baselines certify earlier, their FWER is substantially inflated; notably, SERPANT's $\operatorname{FWER}_{R}$ increases from $0.006$ under i.i.d.\ items to $0.574$--$0.896$ under heterogeneity or dependence.
Relative to the strongest baseline, full-graph BB-EDGE reduces the mean CR by $44.7\%$--$69.9\%$, while the pilot-targeted BB-EDGE reduces the mean confirmatory CR by $62.3\%$--$83.7\%$ by testing only the $k(L-k)=24$ cross-set directions.

\begin{table*}[htbp]
    \centering
    \small
    \setlength{\tabcolsep}{3.5pt}
    \resizebox{\textwidth}{!}{
    \begin{tabular}{ccccccccc}
        \toprule
        & \multicolumn{4}{c}{$|C_m|=1, h_{i,\ell}=0$}
        & \multicolumn{4}{c}{$|C_m|=5,h_{i,\ell}=0$} \\
        \cmidrule(lr){2-5}
        \cmidrule(lr){6-9}
        Method
        & $\operatorname{FWER}_{R}\downarrow$
        & $\operatorname{FCP}_{R}\downarrow$
        & $\operatorname{CertProb}_{R}\uparrow$
        & $\operatorname{CR}\downarrow$
        & $\operatorname{FWER}_{R}\downarrow$
        & $\operatorname{FCP}_{R}\downarrow$
        & $\operatorname{CertProb}_{R}\uparrow$
        & $\operatorname{CR}\downarrow$ \\
        \midrule
        EMR
        & $1.000_{\pm0.000}$ & $0.096_{\pm0.295}$
        & $1.000_{\pm0.000}$ & $1.042_{\pm0.211}$
        & $1.000_{\pm0.000}$ & $0.532_{\pm0.499}$
        & $1.000_{\pm0.000}$ & $1.052_{\pm0.240}$ \\

        CBTL
        & $0.796_{\pm0.403}$ & $0.006_{\pm0.077}$
        & $1.000_{\pm0.000}$ & $8.222_{\pm5.473}$
        & $0.850_{\pm0.357}$ & $0.000_{\pm0.000}$
        & $1.000_{\pm0.000}$ & $38.152_{\pm14.999}$ \\

        $t$-Holm
        & $0.248_{\pm0.432}$ & $0.000_{\pm0.000}$
        & $1.000_{\pm0.000}$ & $9.764_{\pm2.068}$
        & $0.236_{\pm0.425}$ & $0.000_{\pm0.000}$
        & $1.000_{\pm0.000}$ & $24.176_{\pm7.386}$ \\

        SERPANT
        & $0.006_{\pm0.077}$ & $0.000_{\pm0.000}$
        & $1.000_{\pm0.000}$ & $\underline{18.708_{\pm5.294}}$
        & $0.386_{\pm0.487}$ & $0.028_{\pm0.165}$
        & $1.000_{\pm0.000}$ & $18.802_{\pm9.705}$ \\

        BEB-Holm
        & $0.000_{\pm0.000}$ & $0.000_{\pm0.000}$
        & $1.000_{\pm0.000}$ & $56.006_{\pm3.841}$
        & $0.000_{\pm0.000}$ & $0.000_{\pm0.000}$
        & $0.888_{\pm0.316}$ & $84.968_{\pm9.812}$ \\

        BFH-$e$-Holm
        & $0.000_{\pm0.000}$ & $0.000_{\pm0.000}$
        & $0.532_{\pm0.499}$ & $96.562_{\pm4.842}$
        & $0.000_{\pm0.000}$ & $0.000_{\pm0.000}$
        & $0.266_{\pm0.442}$ & $97.426_{\pm5.730}$ \\
        CR-EDGE
& $0.000_{\pm0.000}$ & $0.000_{\pm0.000}$
& $1.000_{\pm0.000}$ & $75.050_{\pm3.783}$
& $0.000_{\pm0.000}$ & $0.000_{\pm0.000}$
& $0.944_{\pm0.230}$ & $83.868_{\pm8.562}$ \\
        
        \hline

        \textbf{BB-EDGE}
        & $0.000_{\pm0.000}$ & $0.000_{\pm0.000}$
        & $1.000_{\pm0.000}$ & $23.780_{\pm2.752}$
        & $0.000_{\pm0.000}$ & $0.000_{\pm0.000}$
        & $1.000_{\pm0.000}$ & $\underline{46.966_{\pm10.143}}$ \\

        \textbf{+Pilot}
        & $0.000_{\pm0.000}$ & $0.000_{\pm0.000}$
        & $1.000_{\pm0.000}$ & $\mathbf{12.252_{\pm2.024}}$
        & $0.000_{\pm0.000}$ & $0.000_{\pm0.000}$
        & $0.920_{\pm0.272}$ & $\mathbf{32.010_{\pm21.491}}$ \\

\midrule
        & \multicolumn{4}{c}{$|C_m|=1, h_{i,\ell}=0.2$}
        & \multicolumn{4}{c}{$|C_m|=5,h_{i,\ell}=0.2$} \\
                \cmidrule(lr){2-5}
        \cmidrule(lr){6-9}
        Method
        & $\operatorname{FWER}_{R}\downarrow$
        & $\operatorname{FCP}_{R}\downarrow$
        & $\operatorname{CertProb}_{R}\uparrow$
        & $\operatorname{CR}\downarrow$
        & $\operatorname{FWER}_{R}\downarrow$
        & $\operatorname{FCP}_{R}\downarrow$
        & $\operatorname{CertProb}_{R}\uparrow$
        & $\operatorname{CR}\downarrow$ \\
        \midrule
        EMR
        & $1.000_{\pm0.000}$ & $0.042_{\pm0.201}$
        & $1.000_{\pm0.000}$ & $1.030_{\pm0.171}$
        & $1.000_{\pm0.000}$ & $0.516_{\pm0.500}$
        & $1.000_{\pm0.000}$ & $1.078_{\pm0.276}$ \\

        CBTL
        & $0.010_{\pm0.100}$ & $0.000_{\pm0.000}$
        & $0.000_{\pm0.000}$ & $100.000_{\pm0.000}$
        & $0.196_{\pm0.397}$ & $0.000_{\pm0.000}$
        & $0.000_{\pm0.000}$ & $100.000_{\pm0.000}$ \\

        $t$-Holm
        & $0.302_{\pm0.460}$ & $0.000_{\pm0.000}$
        & $1.000_{\pm0.000}$ & $8.996_{\pm1.674}$
        & $0.212_{\pm0.409}$ & $0.000_{\pm0.000}$
        & $1.000_{\pm0.000}$ & $20.444_{\pm6.027}$ \\

        SERPANT
        & $0.574_{\pm0.495}$ & $0.028_{\pm0.165}$
        & $0.986_{\pm0.118}$ & $18.554_{\pm11.314}$
        & $0.896_{\pm0.306}$ & $0.174_{\pm0.379}$
        & $0.932_{\pm0.252}$ & $20.978_{\pm23.038}$ \\

        BEB-Holm
        & $0.000_{\pm0.000}$ & $0.000_{\pm0.000}$
        & $1.000_{\pm0.000}$ & $79.880_{\pm4.743}$
        & $0.000_{\pm0.000}$ & $0.000_{\pm0.000}$
        & $0.552_{\pm0.498}$ & $95.032_{\pm6.652}$ \\

        BFH-$e$-Holm
        & $0.000_{\pm0.000}$ & $0.000_{\pm0.000}$
        & $0.578_{\pm0.494}$ & $96.524_{\pm4.410}$
        & $0.000_{\pm0.000}$ & $0.000_{\pm0.000}$
        & $0.310_{\pm0.463}$ & $97.110_{\pm5.787}$ \\ 
        
CR-EDGE
& $0.000_{\pm0.000}$ & $0.000_{\pm0.000}$
& $1.000_{\pm0.000}$ & $74.624_{\pm3.350}$
& $0.000_{\pm0.000}$ & $0.000_{\pm0.000}$
& $0.980_{\pm0.140}$ & $82.146_{\pm7.666}$ \\
        \hline

        \textbf{BB-EDGE}
        & $0.000_{\pm0.000}$ & $0.000_{\pm0.000}$
        & $1.000_{\pm0.000}$ & $\underline{24.072_{\pm2.487}}$
        & $0.000_{\pm0.000}$ & $0.000_{\pm0.000}$
        & $1.000_{\pm0.000}$ & $\underline{42.464_{\pm8.497}}$ \\

        \textbf{+Pilot}
        & $0.000_{\pm0.000}$ & $0.000_{\pm0.000}$
        & $1.000_{\pm0.000}$ & $\mathbf{13.010_{\pm1.742}}$
        & $0.000_{\pm0.000}$ & $0.000_{\pm0.000}$
        & $0.952_{\pm0.214}$ & $\mathbf{26.788_{\pm17.537}}$ \\

        \bottomrule
    \end{tabular}
    }
    \caption{
Monte Carlo means and standard deviations of
$\operatorname{FWER}_{R}$, $\operatorname{FCP}_{R}$,
$\operatorname{CertProb}_{R}$, and the censored certification replicate
$\operatorname{CR}$ for Top-4 certification.
Formal definitions of the four metrics are defined in Appendix~\ref{section:evaluation_metric}.
All experiments use $\tau=0$, $R=100$, $N=100$, and $L=10$.
A repetition without certification by $R$ is assigned $\operatorname{CR}=R$ ($R+R_0$ for +Pilot).
Among methods with $\operatorname{FWER}_{R}\leq0.05$, the smallest and second-smallest mean $\operatorname{CR}$ in each setting are shown in bold and underlined, respectively.
}
    \label{table:1}
\end{table*}

\section{Anytime-Valid Inference for Rank Intervals}
\label{section_rank_interval}
Rank intervals quantify ranking uncertainty and avoid reporting a potentially misleading unique ranking when performance differences are not statistically significant.
At each completed replicate \(r\), BB-EDGE applies direct \(e\)-Holm to the pairwise \(e\)-processes to construct the current confidence graph \(\mathcal G_r=(\mathcal V,\mathcal E_r)\).
In the transitive closure of $\mathcal{G}_r$, let $\mathcal{A}_{\ell,r}$ and $\mathcal{D}_{\ell,r}$ denote the strict ancestors and descendants of model $\ell$, respectively.
The rank interval of model $\ell$ is defined as
\begin{equation}
\mathcal{I}_{\ell,r}=\left[1+\left|\mathcal{A}_{\ell,r}\right|,\;L-\left|\mathcal{D}_{\ell,r}\right|\right].
\label{eq:graph-rank-interval}
\end{equation}
Its lower endpoint counts the models certified to outperform $\ell$, while its upper endpoint excludes those certified to be outperformed by $\ell$.

\begin{corollary}
\label{cor:anytime_rank_intervals}
Define the true benchmark rank of model $\ell$ as $\operatorname{rank}_{B}(\ell)=1+\sum_{j\neq\ell}\mathbbm{1}\!\left\{\theta_j^B>\theta_\ell^B\right\}$ where rank $1$ corresponds to the highest benchmark mean.
Suppose $\tau\geq 0$ and the conditions of Theorem~\ref{thm:bedge-main-iclr} hold.
Then the rank intervals in~(\ref{eq:graph-rank-interval}) satisfy
\[
\mathbb{P}\!\left(\forall\,1\le r\le R,\ \forall \ell\in\mathcal{V}:\operatorname{rank}_{B}(\ell)\in\mathcal{I}_{\ell,r}\,\middle|\,B,\pi\right)\geq 1-\alpha.
\]
\end{corollary}
The proof, including tied benchmark means, is given in Appendix~\ref{app:proof-rank-intervals}. 
If the assumptions hold for the entire infinite sequence, the same bound holds simultaneously for every \(r\ge1\).
Thus, with probability at least $1-\alpha$, the intervals simultaneously cover the true benchmark ranks of all models at every completed replicate $1\le r\le R$, conditional on $(B,\pi)$.

We generate the synthetic data in Section~\ref{section:simulation_studies}, and the evaluation metrics are described in Section~\ref{section:evaluation_metric}.
We additionally compare BB-EDGE with two rank-interval baselines, TLRCI \citep{neuhof2026rank} and MB-Rank \citep{li2026low}.
The results show that, among the methods maintaining the desired anytime coverage level, BB-EDGE produces the narrowest rank intervals.
For example, under item-level heterogeneity at replicate $r=15$, BB-EDGE attains an anytime coverage of $1.00$ with a mean rank-interval width of $1.633$, reducing the width by $62.6\%$ relative to the strongest valid baseline, BEB-Holm ($4.369$).

\begin{figure}[htbp]
\centering
\includegraphics[width=0.99\textwidth]{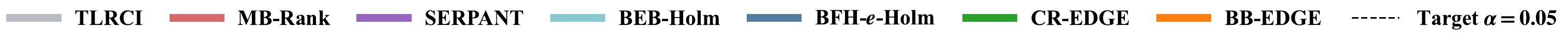}\\[-2ex]
\subfloat[$|C_m|=1$, $h_{i,\ell}=0$]{%
    \includegraphics[width=0.48\textwidth]{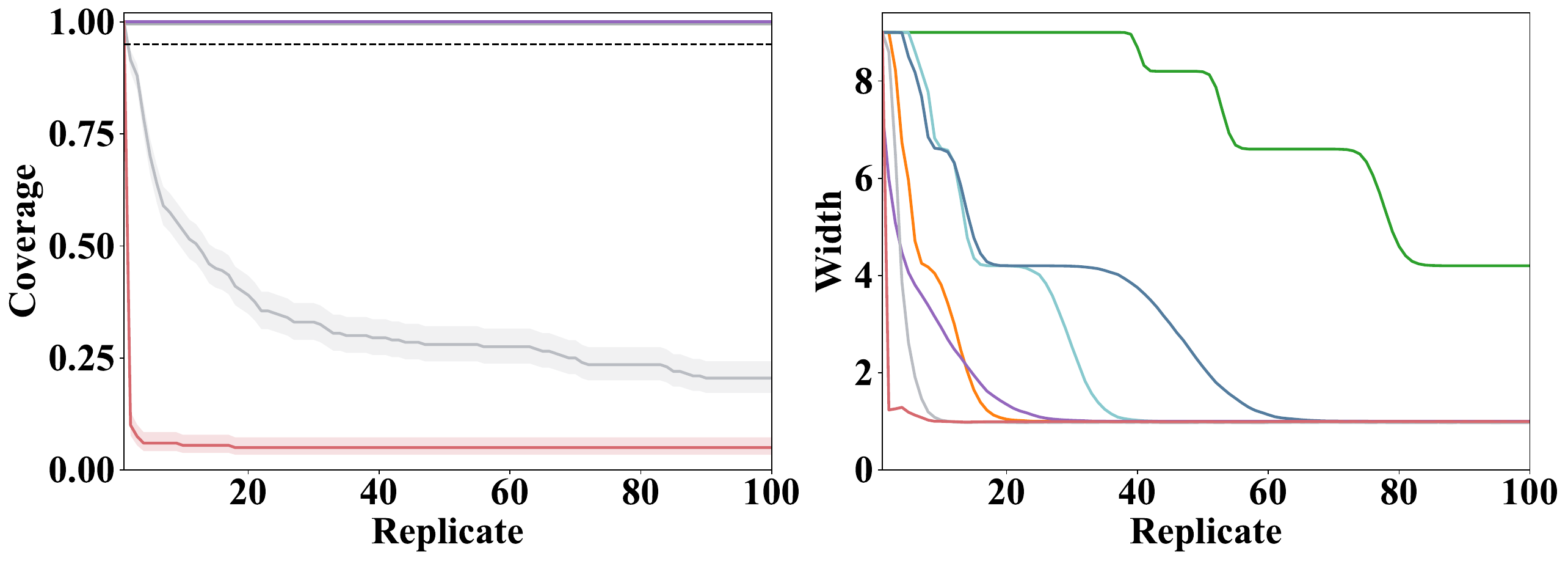}
}
\hfill
\subfloat[$|C_m|=5$, $h_{i,\ell}=0$]{%
    \includegraphics[width=0.48\textwidth]{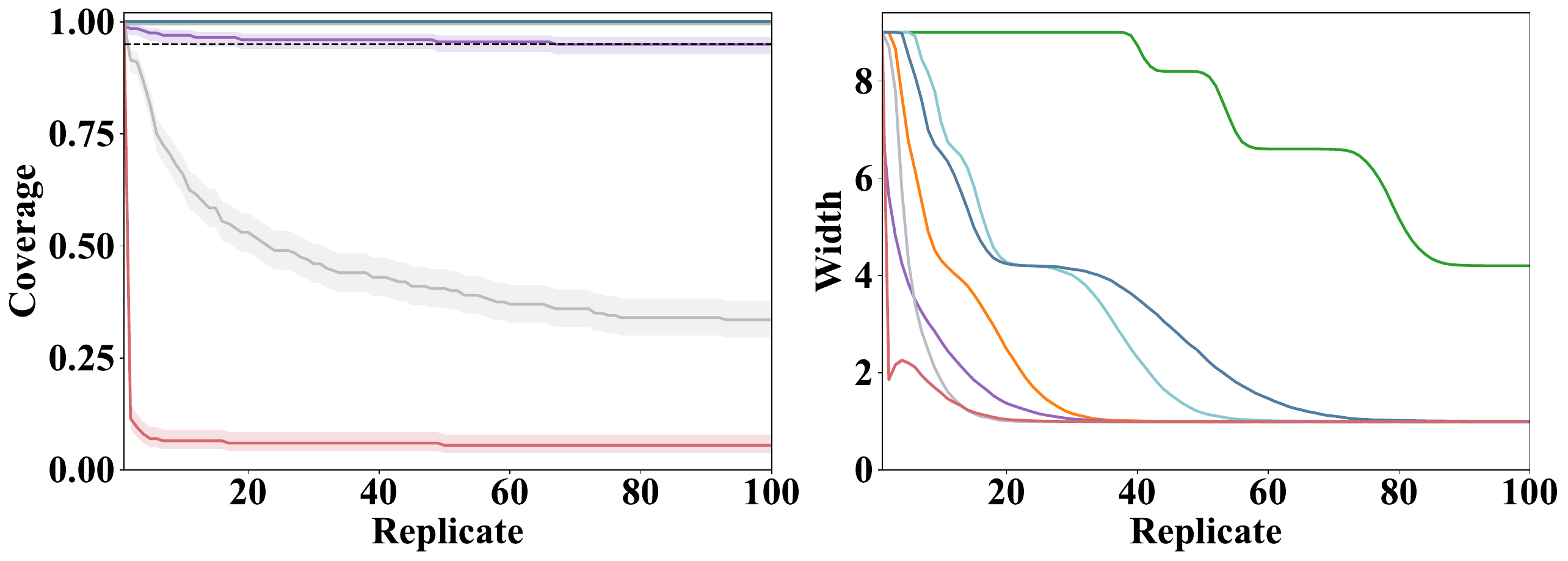}
}
\hfill
\subfloat[$|C_m|=1$, $h_{i,\ell}=0.2$]{%
    \includegraphics[width=0.48\textwidth]{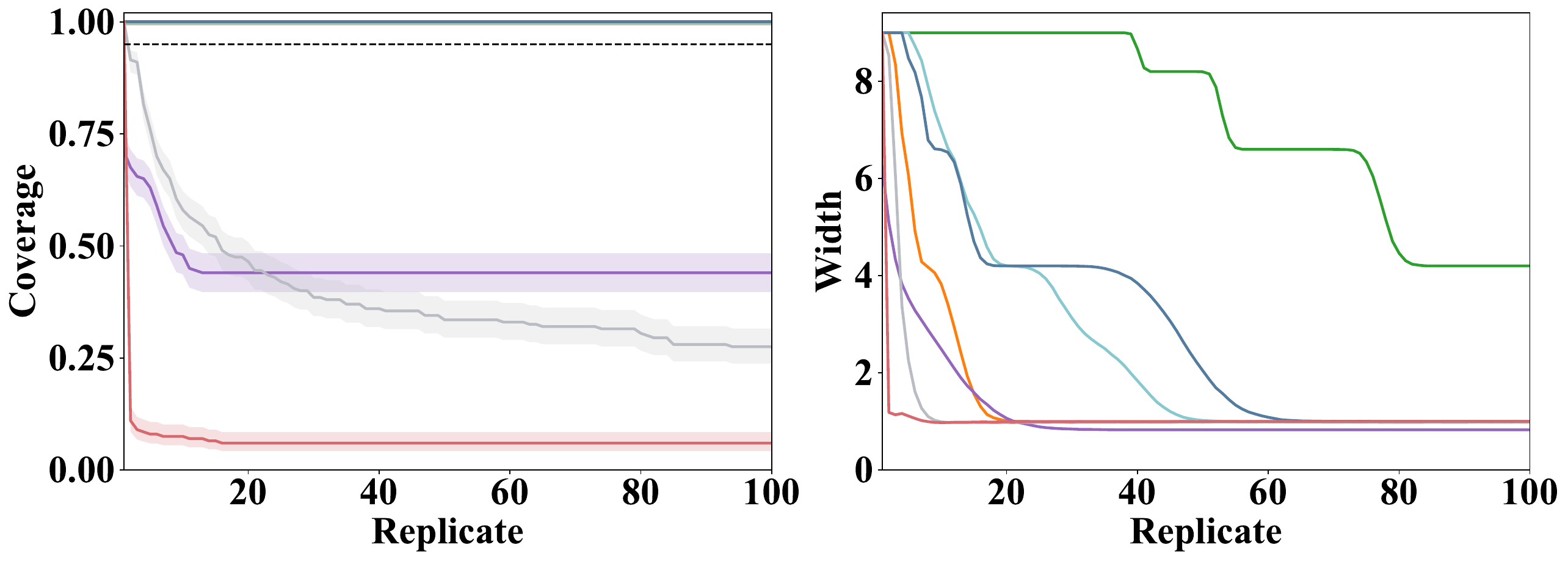}
}
\hfill
\subfloat[$|C_m|=5$, $h_{i,\ell}=0.2$]{%
    \includegraphics[width=0.48\textwidth]{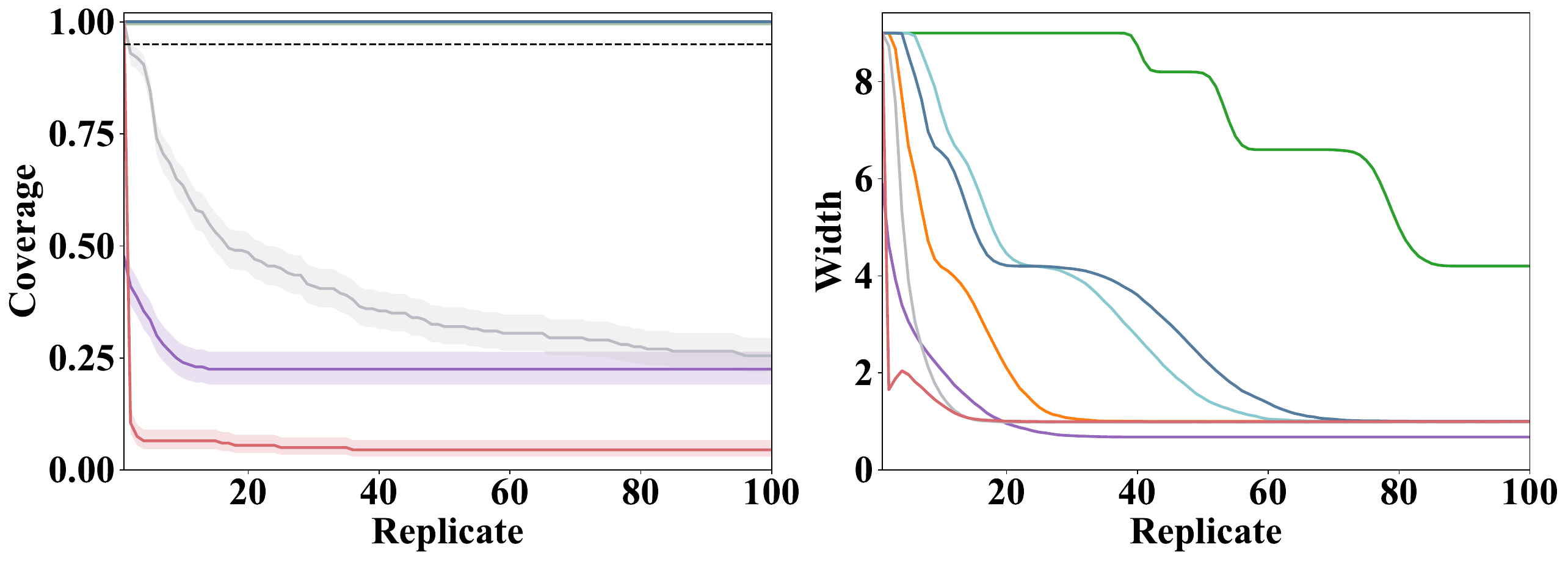}
}
\caption{Empirical $\operatorname{Coverage}_r$ and $\operatorname{Width}_r$ (with 95\% confidence intervals) versus replicate $r$. 
Panels (a)--(d) correspond to the i.i.d., dependence-only, heterogeneity-only, and combined dependence-and-heterogeneity settings, respectively.
All settings use  $\tau=0$, $N=100$, and $L=10$.
The dashed line marks the target level $1-\alpha=0.95$.}
\label{figure:rank_interval}
\end{figure}

\section{Anytime-Valid Inference under Updated Benchmarks}
\label{sec:epoch-reset-procedure}
\subsection{Updated Benchmark Items}
As benchmarks are updated with newly added items, the performance target changes. Valid inference must account for this change rather than treating evidence for an earlier version as evidence for the updated benchmark.
We propose Epoch-reset BB-EDGE, an extension of BB-EDGE that supports anytime-valid inference under updated benchmarks.
Let $\mathcal K_{v-1}$ be the sigma-field of all information available before the fresh replicates of epoch $v$, including its selected design.
Epoch-reset BB-EDGE consists of three steps: 

\paragraph{Step 1: Benchmark expansion.}
At update epoch $v$, let $\mathcal{A}^{(v)}$ denote the set of newly added items. We expand the preceding benchmark to obtain $B^{(v)}=B^{(v-1)}\cup\mathcal{A}^{(v)}=\{x_1,\ldots,x_{N_v}\}$ with $N_v=\lvert B^{(v)}\rvert$.
\paragraph{Step 2: Benchmark Freezing.}
After expansion, we freeze $B^{(v)}$ throughout epoch $v$. The frozen benchmark defines the version-specific performance estimand
\[
\theta_{\ell}^{B^{(v)}}=\frac{1}{N_v}\sum_{i\in B^{(v)}}\mathbb{E}\!\left[S_{i,r}^{(\ell)}\,\middle|\,B^{(v)},\pi\right],\qquad\Delta_{a,b}^{B^{(v)}}=\theta_a^{B^{(v)}}-\theta_b^{B^{(v)}}.
\]
The finite benchmark, finite testing family $\mathcal H^{(v)}$,
partition, finite stake grid, mixture weights, and budget
$\alpha_v\in[0,1)$ are $\mathcal K_{v-1}$-measurable and remain fixed
within the epoch.
\paragraph{Step 3: $e$-Process Restart and Cross-Epoch Control.}
For each $H_e^{(v,\tau)}:\Delta_{a,b}^{B^{(v)}}\le\tau$, we restart $E_{e,0}^{(v)}=1$ using fresh replicates on $B^{(v)}$, preserving the supermartingale property under the new estimand, and apply direct $e$-Holm at level $\alpha_v$ with $\sum_v \alpha_v\le\alpha$.
Every grid component starts at one. 
Predictions are predictable with respect to $\mathcal F^{(v)}_r= \sigma(\mathcal K_{v-1},Z^{(v)}_1,\ldots,Z^{(v)}_r)$.
Let $\mathcal E^{(v)}_r$ denote the edge sets after replicate $r$. 
Set $\mathcal E^{(v)}_r=\varnothing$ outside completed runs or when $\alpha_v=0$; an epoch not entered has budget zero.
For a finite run budget, the assumptions need to hold only through that budget. 
A guarantee for every $r\ge1$ requires them for the entire sequence of fresh panels.

Because each epoch uses a newly initialized $e$-process constructed for its frozen benchmark, the anytime-validity of BB-EDGE applies separately within every epoch.
Allocating a version-specific error budget $\alpha_v$ and requiring the cumulative budget to satisfy $\sum_{v\ge1}\alpha_v\le\alpha$ then extends this guarantee across benchmark versions.
The resulting bound averages over the adaptive history and covers
all reports actually issued. Each claim remains specific to the
benchmark version on which it was tested.
Theorem \ref{thm:update_benchmark} formalizes the resulting FWER control, whose proof is given in Appendix~\ref{app:proof-epochs}.

\begin{theorem}
\label{thm:update_benchmark}
Suppose that, conditional on $\mathcal K_{v-1}$, the fresh panels in each
epoch satisfy Assumption~\ref{ass:design-iclr} under the fixed execution law
for $(B^{(v)},\pi)$ defining $\theta_\ell^{B^{(v)}}$.
If $\sum_{v\ge1}\alpha_v\le\alpha$ almost surely, then Epoch-reset BB-EDGE satisfies
\[
\Pr\!\left(
\exists v\ge 1,\ \exists r\ge 1,\ 
\exists e=(a\to b)\in\mathcal{E}_{r}^{(v)}
\text{ such that }
\Delta_{a,b}^{B^{(v)}}\le\tau
\right)
\le\alpha.
\]
\end{theorem}

Following the data-generating process in Section~\ref{section:simulation_studies}, we expand the benchmark from 40 to 70 and 100 items at replicates $r=10$ and $r=20$, respectively.
Epoch-reset BB-EDGE restarts at each update using $\alpha_v=\alpha/3$, whereas fixed BB-EDGE continues on the initial benchmark at level $\alpha$. 
Figure~\ref{figure:update_benchmark} shows that both procedures maintain empirical FWER below $0.05$ in all four settings. 
Epoch-reset BB-EDGE has lower power immediately after each reset because it discards previously accumulated evidence and uses a smaller per-version error budget; however, the expanded benchmarks provide more evidence per replicate, allowing it to catch up with and eventually outperform the fixed-benchmark reference, with the improvement being most pronounced under dependence and combined dependence and heterogeneity.

\begin{figure}[htbp]
\centering
\subfloat[$|C_m|=1$, $h_{i,\ell}=0$]{%
    \includegraphics[width=0.48\textwidth]{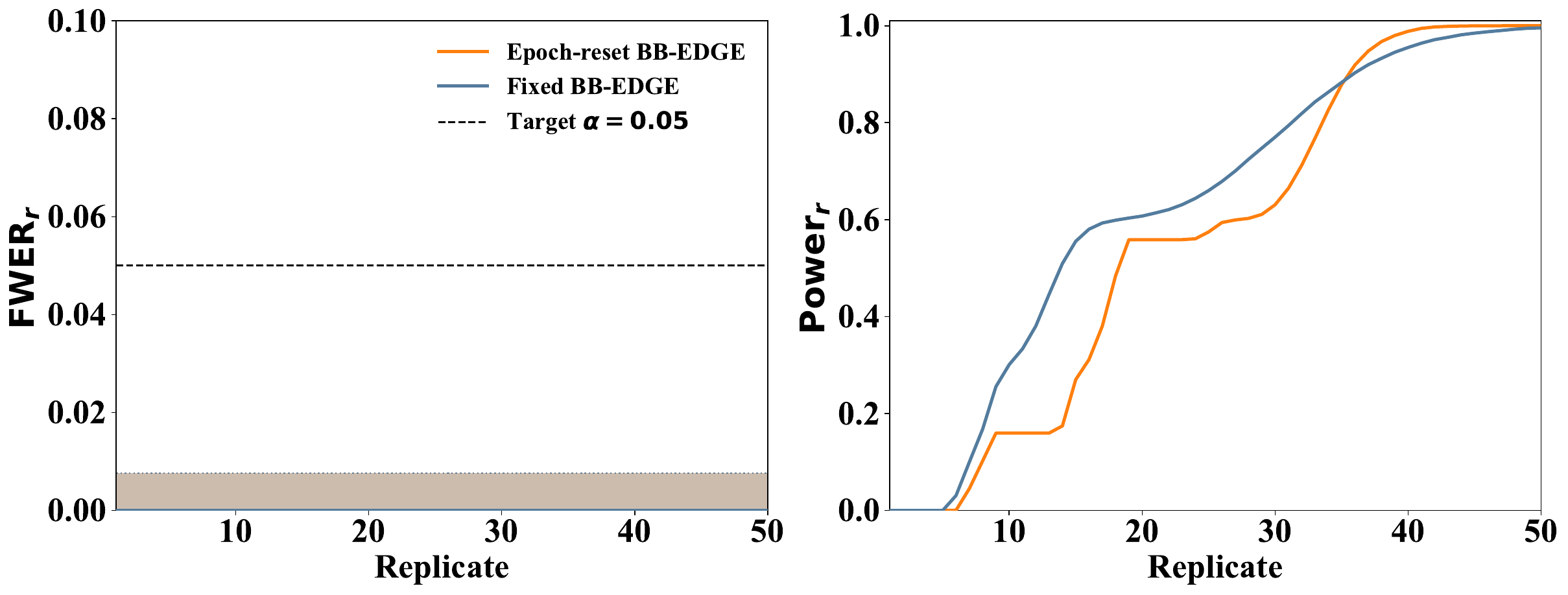}
}
\hfill
\subfloat[$|C_m|=5$, $h_{i,\ell}=0$]{%
    \includegraphics[width=0.48\textwidth]{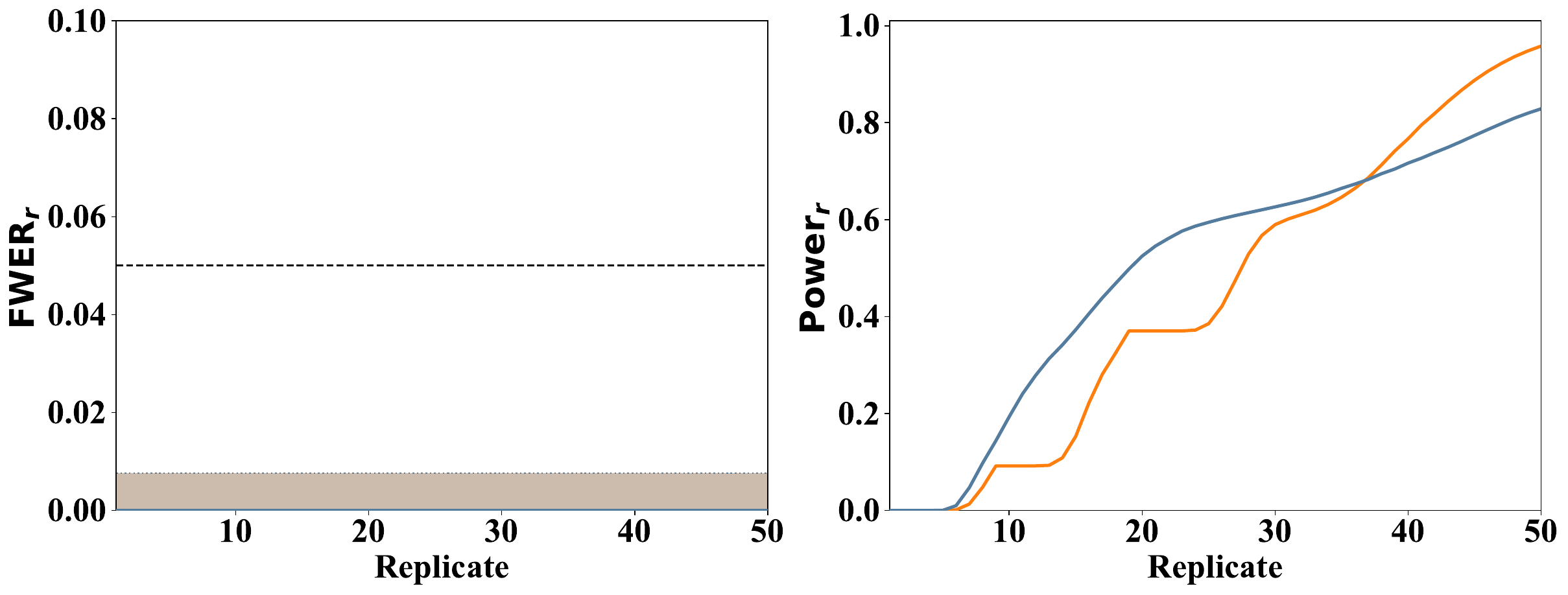}
}
\hfill
\subfloat[$|C_m|=1$, $h_{i,\ell}=0.2$]{%
    \includegraphics[width=0.48\textwidth]{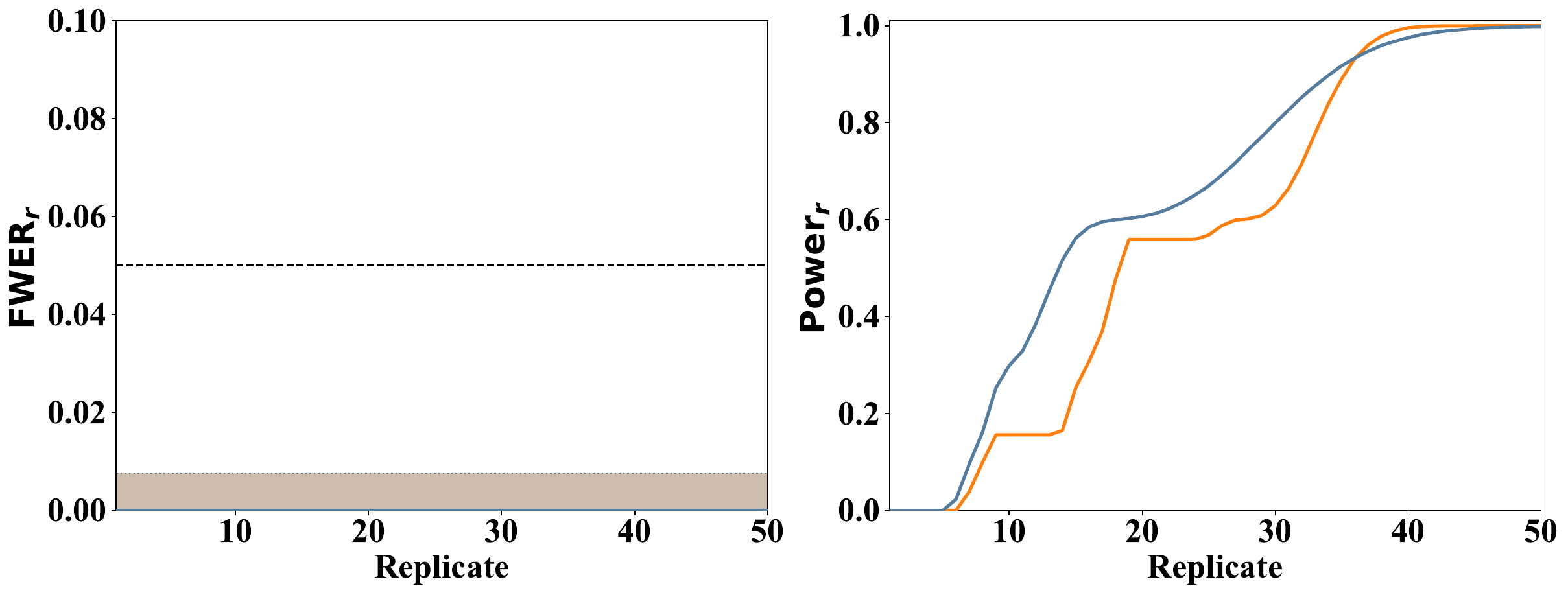}
}
\hfill
\subfloat[$|C_m|=5$, $h_{i,\ell}=0.2$]{%
    \includegraphics[width=0.48\textwidth]{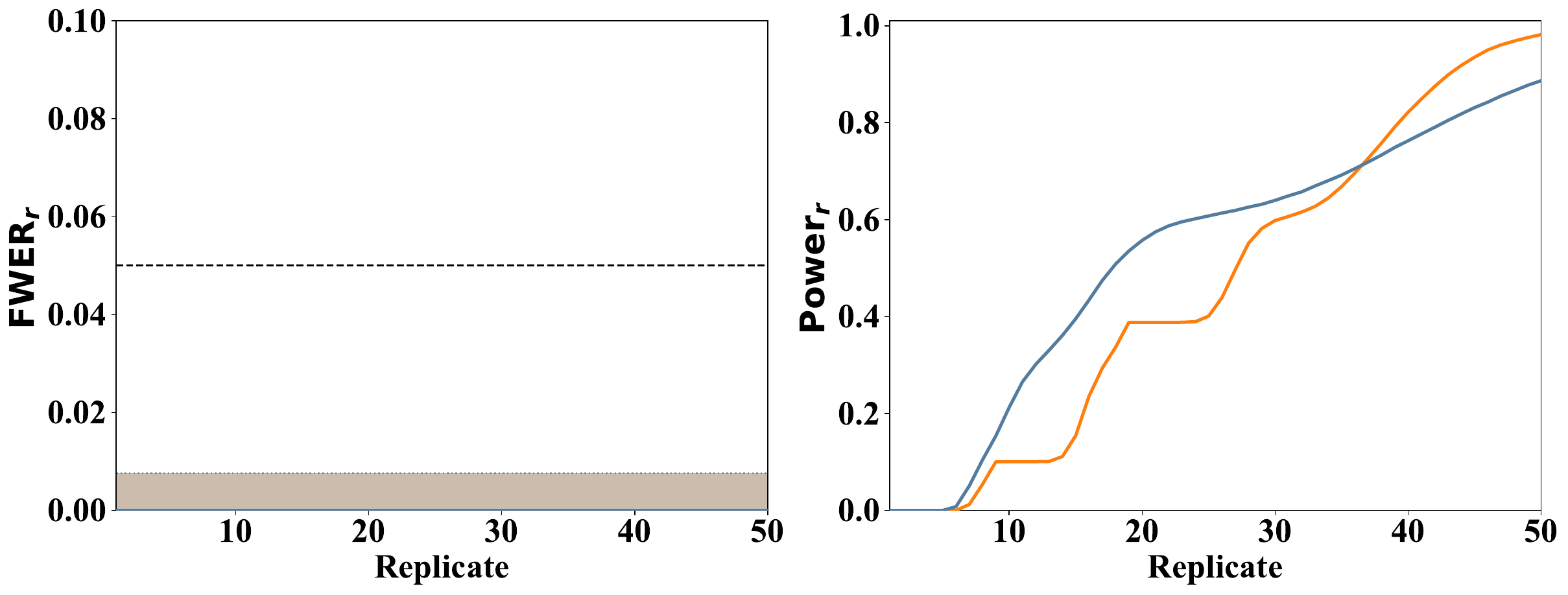}
}
\caption{Empirical $\operatorname{FWER}_r$ and $\operatorname{Power}_r$ versus replicate $r$. 
Panels (a)--(d) correspond to the i.i.d., dependence-only, heterogeneity-only, and combined dependence-and-heterogeneity settings, respectively.
All settings use  $\tau=0$, $N=100$, and $L=10$.
The dashed line marks the target level $\alpha=0.05$.}
\label{figure:update_benchmark}
\end{figure}

\subsection{Newly Added Models}
\label{sec:new-models}

BB-EDGE can accommodate newly arriving models without reevaluating the existing models. To allow an unknown number of future model additions, we prespecify an error-spending sequence $(\alpha_v)_{v\ge0}$ satisfying $\sum_{v\ge0}\alpha_v\le\alpha,$ where $\alpha_0$ is allocated to the initial leaderboard and $\alpha_v$ is allocated to the $v$th model-addition stage for $v\ge1$. The procedure consists of two steps.

\paragraph{Step 1: Model addition and panel augmentation}
When the inclusion of a new model $\ell_v$ is determined independently of
the inference panels, we set $\mathcal V^{(v)}=\mathcal V^{(v-1)}\cup\{\ell_v\}.$
We evaluate $\ell_v$ on the same benchmark $\cB$ and under the same protocol $\pi$. According to a pairing rule fixed before inference, its outcomes are combined with the stored outcomes to form $\widetilde Z_r^{(v)}=\left(Z_r^{(v-1)},\{S_{i,r}^{(\ell_v)}:i\in\cB\}\right),$ where $r=1,\ldots,R^{(v)}$ and $R^{(v)}$ is the number of available paired replicates at stage $v$.
The augmented panels are replayed in replicate order, with each predictable quantity at replicate $r$ depending only on the preceding augmented panels.

\paragraph{Step 2: New-direction inference}
We initialize $e$-processes only for the newly introduced directional family $\cH_{\mathrm{new}}^{(v)}=\left\{H_{\ell_v\to\ell}^{(\tau)},H_{\ell\to\ell_v}^{(\tau)}:\ell\in\mathcal V^{(v-1)}\right\}$ and apply direct $e$-Holm at level $\alpha_v$. 
Let $\cE_{r,\mathrm{new}}^{(v)}$ denote the resulting edge set at replicate $r\le R^{(v)}$. Previously certified edges are retained; their errors are already accounted for by the budgets allocated to the earlier stages.

Corollary~\ref{cor:new-models} establishes FWER control for the combined graph across the initial leaderboard and all subsequent model additions.

\begin{corollary}
\label{cor:new-models}
Let $\cE_{r,\mathrm{new}}^{(0)}=\cE_r$ denote the edge set of the initial
leaderboard, constructed over $R^{(0)}$ replicates at level $\alpha_0$.
Suppose that the initial panels $Z_1,\ldots,Z_{R^{(0)}}$ satisfy Assumption~\ref{ass:design-iclr}. For each $v\ge1$, suppose that
model inclusion, the pairing rule, and all stage-$v$ design choices are fixed independently of the augmented inference panels conditional on $(\cB,\pi)$. 
Suppose further that $\widetilde Z_1^{(v)},\ldots,\widetilde Z_{R^{(v)}}^{(v)}$ satisfy Assumption~\ref{ass:design-iclr} and are processed through the non-anticipating replay described above.
If $(\alpha_v)_{v\ge0}$ is a prespecified error-spending sequence satisfying
$\sum_{v\ge0}\alpha_v\le\alpha$, then
\[
\Pp\!\left(
\exists v\ge0,\ 
\exists r\le R^{(v)},\ 
\exists e\in\cE_{r,\mathrm{new}}^{(v)}
\text{ such that }
\Delta_e^{\cB}\le\tau
\,\middle|\,
\cB,\pi
\right)
\le\alpha.
\]
On the complementary event, all path implications, simultaneous rank intervals, and Top-$k$ certificates derived from the combined graph are valid.
\end{corollary}

The current experiments allocate the full error budget to the initial
leaderboard; the staged allocation above applies prospectively when future
model additions are anticipated. The independence condition is satisfied
when model additions are driven by external events, such as new model
releases, rather than selected using the certified graph itself.

\section{Proofs}
\label{app:sec3-proofs}
\renewcommand{\theHequation}{proof.\arabic{equation}}

The fixed-benchmark bounds in Theorem~\ref{thm:bedge-main-iclr} and Corollary~\ref{cor:anytime_rank_intervals} condition on $(\cB,\pi)$;
Theorem~\ref{thm:pilot-topk} additionally conditions on the pilot data $\mathcal D_0$. Theorem~\ref{thm:update_benchmark} first controls each epoch conditional on its past and then averages over that past.
All evidence processes are monitored at completed-replicate boundaries, using the common filtration containing every model's observations.

\subsection{Proof of Proposition~\ref{prop:stake-necessity-iclr}}

\begin{proof}
Take any $d\in\mathbb R^M$ with $w^\top d=0$. If $d\ne0$, choose
\[
0<\varepsilon<\frac{\min\{\mu_0,1-\mu_0\}}{\max_m|d_m|}.
\]
Then both vectors $\nu_+=\mu_0\one+\varepsilon d$ and $\nu_-=\mu_0\one-\varepsilon d$ belong to $[0,1]^M$.
Since $w^\top\one=1$, they satisfy $w^\top\nu_+=w^\top\nu_-=\mu_0$.
Applying (\ref{eq:stake-drift-condition}) to $\nu_+$ and $\nu_-$ gives
\[
\eta^\top(\nu_+-\mu_0\one)=\varepsilon\eta^\top d\le0,\qquad\eta^\top(\nu_--\mu_0\one)=-\varepsilon\eta^\top d\le0.
\]
Since $\varepsilon>0$, these inequalities imply $\eta^\top d=0$.
This equality also holds for $d=0$. We have therefore proved
\begin{equation}
w^\top d=0\quad\Longrightarrow\quad\eta^\top d=0,\qquad d\in\mathbb R^M.
\label{eq:proof-orthogonality}
\end{equation}

Set $c=(\eta^\top w)/(w^\top w)$ and $h=\eta-cw$.
The denominator is positive because every $w_m>0$, and
\[
w^\top h=w^\top\eta-cw^\top w
=w^\top\eta-\eta^\top w=0.
\]
Equation (\ref{eq:proof-orthogonality}), applied with $d=h$, now yields
\[
0=\eta^\top h=(cw+h)^\top h=\|h\|_2^2.
\]
Consequently $h=0$ and $\eta=cw$. The nonnegativity of $\eta$ and positivity of $w$ imply $c\ge0$. 
Conversely, if $\eta=cw$ with
$c\ge0$, then for every admissible $\nu$,
\[
\eta^\top(\nu-\mu_0\one)=c(w^\top\nu-\mu_0)\le0.
\]
This proves both directions, including $c=0$ and $M=1$. 
The necessity concerns precisely the nonpositive linear drift condition in the proposition.
\end{proof}

\subsection{The conditional empirical-Bernstein inequality}
\label{app:proof-eb-inequality}

The following bounded-variable inequality is the ingredient needed for the block factors, which is the predictable empirical-Bernstein construction of \citet{smith2024estimating}. We below give its scalar argument, following \citet{fan2015exponential}, to specify the normalization used here.

\begin{lemma}
\label{lem:conditional-eb-proof}
Let $\mathcal A$ be a sigma-field, let $Y\in[0,1]$, and write $\nu=\E[Y\mid\mathcal A]$. If $p\in[0,1]$ and $u\in[0,1)$ are $\mathcal A$-measurable, then
\begin{equation}
\E\!\left[\exp\{u(Y-\nu)-\psi_E(u)(Y-p)^2\}\middle|\mathcal A\right]\le1,\qquad \psi_E(u)=-\log(1-u)-u.
\label{eq:proof-conditional-eb}
\end{equation}
\end{lemma}

\begin{proof}
First we prove the deterministic inequality
\begin{equation}
\exp\{ux-\psi_E(u)x^2\}\le1+ux,\qquad 0\le u<1,\quad -1\le x\le1,
\label{eq:proof-scalar-eb}
\end{equation}
which is an equality if $u=0$ or $x=0$. For $x>0$, let $z=ux$.
The function $f(z)=\log(1+z)-z+z^2/2$ satisfies $f(0)=0$ and $f'(z)=z^2/(1+z)\ge0$ for $z\ge0$. 
Since $\psi_E(u)=\sum_{j=2}^\infty u^j/j\ge u^2/2$, this gives
\[
\log(1+ux)\ge ux-u^2x^2/2\ge ux-\psi_E(u)x^2.
\]
For $x=-v$ with $0<v\le1$, absolute convergence and nonnegativity
of the series terms give
\[
\psi_E(uv)=\sum_{j=2}^\infty\frac{u^jv^j}{j}\le v^2\sum_{j=2}^\infty\frac{u^j}{j}=v^2\psi_E(u).
\]
Hence $\log(1-uv)=-uv-\psi_E(uv)\ge-uv-v^2\psi_E(u)$. As $1+ux\ge1-u>0$, exponentiation proves (\ref{eq:proof-scalar-eb}) in both cases.
To apply (\ref{eq:proof-scalar-eb}) with $x=Y-p$, write
\[
u(Y-\nu)=u(Y-p)+u(p-\nu).
\]
Here $u(p-\nu)$ is $\mathcal A$-measurable because $u$, $p$, and
$\nu=\E[Y\mid\mathcal A]$ are $\mathcal A$-measurable. Thus
\begin{align}
\exp\{u(Y-\nu)-\psi_E(u)(Y-p)^2\}
&=\exp\{u(p-\nu)\}\exp\{u(Y-p)-\psi_E(u)(Y-p)^2\}\nonumber\\
&\le\exp\{u(p-\nu)\}\{1+u(Y-p)\}.
\label{eq:proof-centering-bound}
\end{align}
In particular, $\exp\{u(p-\nu)\}$ and $u$ can be taken outside a
conditional expectation given $\mathcal A$. The exponential on the
left satisfies
\[
0<\exp\{u(Y-\nu)-\psi_E(u)(Y-p)^2\}
\le\exp\{u(Y-\nu)\}\le\exp(1),
\]
because $\psi_E(u)\ge0$, $Y-\nu\le1$, and $u<1$. It is therefore
integrable. Taking conditional expectations in
(\ref{eq:proof-centering-bound}) gives
\[
\begin{aligned}
&\E\!\left[\exp\{u(Y-\nu)-\psi_E(u)(Y-p)^2\}
              \middle|\mathcal A\right]\\
&\quad\le\exp\{u(p-\nu)\}\{1+u(\nu-p)\}=\exp(z)(1-z),
\qquad z=u(p-\nu)<1.
\end{aligned}
\]
The inequality $\log t\le t-1$ for $t>0$, applied to $t=1-z$,
gives $z+\log(1-z)\le0$. Thus $\exp(z)(1-z)\le1$, as required.
The argument does not require $p=\nu$.
\end{proof}

\subsection{Proof of Theorem~\ref{thm:bedge-main-iclr}}
\label{app:proof-main-theorem}

\begin{proof}
Recall that the inference filtration is
\[
\cF_r
=
\sigma(\cB,\pi,\mathcal Q,Z_1,\ldots,Z_r),
\qquad r=0,1,\ldots,R,
\]
with the convention
\[
\cF_0=\sigma(\cB,\pi,\mathcal Q).
\]
Here, $\mathcal Q$ collects all information determined before the confirmatory inference runs, including the directional testing family, the benchmark partition, the stake grid, the mixture weights, and any independent pilot data used to select these quantities. Thus, $\cF_{r-1}$ represents all information available immediately before the $r$th complete replicate $Z_r$ is observed.
We first work conditionally on $\cF_0$. Under this conditional law, the benchmark, protocol, and all design quantities contained in $\mathcal Q$ are treated as fixed, while the confirmatory panels $Z_1,\ldots,Z_R$ remain random. We prove the supermartingale property and the corresponding FWER bound conditional on $\cF_0$. The resulting conditional bound holds almost surely with respect to the random pre-inference design information. We then use the tower property to average over this initial randomness and obtain the stated guarantee conditional on $(\cB,\pi)$.

Fix $e=(a\to b)$ and $1\le r\le R$. We first prove
\[
\sum_{m=1}^M w_m\E[Y_{e,m,r}\mid\cF_{r-1}]
=\frac{1+\Delta_{a,b}^{\cB}}2.
\]
Write $\mathcal L(U\mid\mathcal A)$ for the conditional distribution
of a random variable $U$ given a sigma-field $\mathcal A$.
Since $\cF_{r-1}=\sigma(\cF_0,Z_1,\ldots,Z_{r-1})$,
Assumption~\ref{ass:design-iclr}(1) and (3) give, respectively,
\begin{equation}
\mathcal L(Z_r\mid\cF_{r-1})
=\mathcal L(Z_r\mid\cF_0)
=\mathcal L(Z_r\mid\cB,\pi),
\label{eq:proof-conditional-law}
\end{equation}
where the first equality uses independence across runs conditional on
$\cF_0$; the second uses independence of $\mathcal Q$ and the confirmatory panels conditional on $(\cB,\pi)$. 
Integrating any bounded measurable panel function $h$ against these equal distributions yields
\begin{equation}
\E[h(Z_r)\mid\cF_{r-1}]
=\E[h(Z_r)\mid\cF_0]
=\E[h(Z_r)\mid\cB,\pi].
\label{eq:proof-execution-law}
\end{equation}
Identical distribution across runs in Assumption~\ref{ass:design-iclr}(1)
makes the last expectation independent of $r$.

Recall that $w_m=|C_m|/N$. Section~\ref{sec:theory}, immediately
before (\ref{eq:conditional-average-null}), defines
$\nu_{e,m,r}=\E[Y_{e,m,r}\mid\cF_{r-1}]$.
For each realized partition, (\ref{eq:block-average-identity}) holds, and each $w_m$ is $\cF_0$-measurable, hence $\cF_{r-1}$-measurable. Linearity of conditional expectation and (\ref{eq:proof-execution-law}) give
\begin{align}
\sum_m w_m\nu_{e,m,r}
&=\E\!\left[\sum_m w_mY_{e,m,r}\middle|\cF_{r-1}\right]
=\E\!\left[\frac1N\sum_i X^e_{i,r}\middle|\cF_{r-1}\right]
\nonumber\\
&=\frac12+\frac12\E\!\left[
 \frac1N\sum_i(S_{i,r}^{(a)}-S_{i,r}^{(b)})
 \middle|\cF_{r-1}\right]
=\frac{1+\Delta_{a,b}^{\cB}}2.
\label{eq:proof-average-identity}
\end{align}
Under $H_e^{(\tau)}$, $\Delta_{a,b}^{\cB}\le\tau$, so
$(1+\Delta_{a,b}^{\cB})/2\le(1+\tau)/2=\mu_0$.

For a grid point $\lambda_g$, each stake $u_m=\lambda_{g,m}=\lambda_gw_m/w_\star$ lies in $[0,1)$ and is fixed before the run. 
Decompose the factor as
\begin{align}
F_{e,r}^{\cC}(\lambda_g)
&=\exp\!\left\{\sum_m u_m(\nu_{e,m,r}-\mu_0)\right\}\prod_m Q_{m,r},\nonumber\\
Q_{m,r}&=\exp\!\left\{u_m(Y_{e,m,r}-\nu_{e,m,r})-\psi_E(u_m)(Y_{e,m,r}-\widehat Y_{e,m,r})^2\right\}.
\label{eq:proof-factor-decomposition}
\end{align}
Recall the block vectors in Assumption~\ref{ass:design-iclr}(2),
\[
V_{m,r}=(W_{i,r}:i\in C_m),\qquad
W_{i,r}=(S_{i,r}^{(1)},\ldots,S_{i,r}^{(L)}).
\]
The observation entering $Q_{m,r}$ is
\[
Y_{e,m,r}=\frac1{|C_m|}\sum_{i\in C_m}\frac{S_{i,r}^{(a)}-S_{i,r}^{(b)}+1}{2},
\]
which is a function of $V_{m,r}$. In contrast,
$u_m$ is $\cF_0$-measurable,
$\widehat Y_{e,m,r}$ is $\cF_{r-1}$-measurable by
Assumption~\ref{ass:design-iclr}(4), and
$\nu_{e,m,r}=\E[Y_{e,m,r}\mid\cF_{r-1}]$ is
$\cF_{r-1}$-measurable by the definition of conditional expectation.
Thus, after conditioning on $\cF_{r-1}$, the parameters
$u_m$, $\widehat Y_{e,m,r}$, and $\nu_{e,m,r}$ in $Q_{m,r}$ are
fixed. Only $Y_{e,m,r}$ depends on the new run. We do not assume that $Y_{e,m,r}$ is $\cF_{r-1}$-measurable.
Mutual conditional independence of $V_{1,r},\ldots,V_{M,r}$ implies mutual conditional independence of $Q_{1,r},\ldots,Q_{M,r}$.
Applying Lemma~\ref{lem:conditional-eb-proof} with $\mathcal A=\cF_{r-1}$, $Y=Y_{e,m,r}$, $p=\widehat Y_{e,m,r}$, and $u=u_m$ gives $\E[Q_{m,r}\mid\cF_{r-1}]\le1$ for each $m$. 
Also,
\[
\sum_m u_m(\nu_{e,m,r}-\mu_0)
=\frac{\lambda_g}{w_\star}
 \left(\sum_m w_m\nu_{e,m,r}-\mu_0\right)\le0
\]
under the null. Taking conditional expectations in
(\ref{eq:proof-factor-decomposition}) consequently yields
\begin{equation}
\E[F_{e,r}^{\cC}(\lambda_g)\mid\cF_{r-1}]
=\exp\!\left\{\sum_m u_m(\nu_{e,m,r}-\mu_0)\right\}
 \prod_m\E[Q_{m,r}\mid\cF_{r-1}]
\le1.
\label{eq:proof-one-run-validity}
\end{equation}

For each $g$ and $m$, $0\le u_m<1$ implies $0\le\psi_E(u_m)<\infty$. 
Since $|Y_{e,m,r}-\mu_0|\le1$ and $(Y_{e,m,r}-\widehat Y_{e,m,r})^2\le1$, the exponent in
(\ref{eq:block-factor}) satisfies
\[
-\sum_{m=1}^M\{u_m+\psi_E(u_m)\}\le\log F_{e,r}^{\cC}(\lambda_g)\le\sum_{m=1}^M u_m\le M\le N.
\]
Both bounds are finite for each initial design. 
Exponentiating gives $0<F_{e,r}^{\cC}(\lambda_g)\le\exp(N)$.
The component product and the mixture in (\ref{eq:block-mixture}) consequently obey, for every finite $r$,
\begin{equation}
\begin{aligned}
0<E_{e,r}^{\cC}(\lambda_g)
&=\prod_{s=1}^rF_{e,s}^{\cC}(\lambda_g)\le\exp(rN),\\
0<E_{e,r}^{\cC}
&=\sum_{g=1}^G\rho_gE_{e,r}^{\cC}(\lambda_g)
 \le\exp(rN)\sum_{g=1}^G\rho_g=\exp(rN).
\end{aligned}
\label{eq:proof-evidence-bounds}
\end{equation}
Here $\rho_1,\ldots,\rho_G$ are the mixture coefficients, not the
benchmark block weights $w_m$. Since $\rho_g\ge0$ and
$\sum_g\rho_g=1$,  at least one coefficient is positive.
Since $N$ is finite and fixed given the benchmark,
(\ref{eq:proof-evidence-bounds}) implies
\[
\E[|E_{e,r}^{\cC}(\lambda_g)|\mid\cF_0]\le\exp(rN)<\infty,
\qquad
\E[|E_{e,r}^{\cC}|\mid\cF_0]\le\exp(rN)<\infty.
\]
The same bounds hold conditional on $(\cB,\pi)$ by the tower
property, which proves the required finite-time integrability under the fixed-benchmark law.

The recursion of $E_{e,r}^{\cC}(\lambda_g) =E_{e,r-1}^{\cC}(\lambda_g)F_{e,r}^{\cC}(\lambda_g)$ and the $\cF_{r-1}$-measurability of the preceding product yield
\begin{align*}
\E[E_{e,r}^{\cC}(\lambda_g)\mid\cF_{r-1}]
&=E_{e,r-1}^{\cC}(\lambda_g)
  \E[F_{e,r}^{\cC}(\lambda_g)\mid\cF_{r-1}]\\
&\le E_{e,r-1}^{\cC}(\lambda_g).
\end{align*}
Assumption~\ref{ass:design-iclr}(3) makes each $\rho_g$
$\cF_0$-measurable and keeps it unchanged across runs. As
$\cF_0\subseteq\cF_{r-1}$ and the sum has $G<\infty$ terms,
\begin{align*}
\E[E_{e,r}^{\cC}\mid\cF_{r-1}]
&=\sum_{g=1}^G\rho_g
       \E[E_{e,r}^{\cC}(\lambda_g)\mid\cF_{r-1}]\\
&\le\sum_{g=1}^G\rho_gE_{e,r-1}^{\cC}(\lambda_g)
=E_{e,r-1}^{\cC}.
\end{align*}
Moreover, $E_{e,0}^{\cC}=\sum_{g=1}^G\rho_g=1$.
This proves the directional supermartingale assertion. Zero stakes and
zero mixture weights cause no difficulty.

For the FWER claim (\ref{eq:main-anytime-fwer}), define
\[
\mathcal N_0
=\{e=(a\to b)\in\overrightarrow{\cP}:\Delta_{a,b}^{\cB}\le\tau\},
\qquad
F_R=\bigcup_{r=1}^R\{\cE_r\cap\mathcal N_0\ne\varnothing\}.
\]
We will show $\Pp(F_R\mid\cF_0)\le\alpha$ and then average over
the initial design to obtain $\Pp(F_R\mid\cB,\pi)\le\alpha$.
Membership in $\mathcal N_0$ depends only on the tested directional family $\overrightarrow{\cP}$ and on the benchmark means $\theta_a^{\cB}$ and $\theta_b^{\cB}$.
The tested directional family is $\cF_0$-measurable by Assumption~\ref{ass:design-iclr}(3), while the benchmark means are determined by $(\cB,\pi)$. Therefore, for each direction $e$, $\mathbbm1\{e\in\mathcal N_0\}$ is $\cF_0$-measurable, and so is $h_0=|\mathcal N_0|$.
If $h_0=0$, $F_R$ is empty. For an initial design with $h_0\ge1$, write $t=1/\alpha$ and define
\[
A_r=\frac1{h_0}\sum_{e\in\mathcal N_0}E_{e,r}^{\cC}.
\]
All directional processes use the same filtration $\cF_r$, so
\[
\E[A_r\mid\cF_{r-1}]
=\frac1{h_0}\sum_{e\in\mathcal N_0}
       \E[E_{e,r}^{\cC}\mid\cF_{r-1}]
\le\frac1{h_0}\sum_{e\in\mathcal N_0}E_{e,r-1}^{\cC}
=A_{r-1}.
\]
Hence $A_r$ is a nonnegative supermartingale with $A_0=1$. The displayed equality is linearity of conditional expectation for a finite sum; the directional processes may be dependent. The cutoff in (\ref{eq:eholm-rejections}) is
\[
\cJ_r=\{e\in\overrightarrow{\cP}:E_{e,r}^{\cC}<t\},
\qquad
c_r=t+\sum_{e\in\cJ_r}(t-E_{e,r}^{\cC}).
\]

Suppose some $e_\star\in\mathcal N_0$ is rejected at time $r$. Since $c_r\ge t$, this direction does not belong to $\cJ_r$. For every $e\in\cJ_r$, the definition of $\cJ_r$ gives $t-E_{e,r}^{\cC}>0$. Splitting the cutoff sum therefore gives
\begin{align*}
E_{e_\star,r}^{\cC}\ge c_r
&=t+\sum_{e\in\mathcal N_0\cap\cJ_r}(t-E_{e,r}^{\cC})
    +\sum_{e\in\cJ_r\setminus\mathcal N_0}(t-E_{e,r}^{\cC})\\
&\ge t+\sum_{e\in\mathcal N_0\cap\cJ_r}(t-E_{e,r}^{\cC}).
\end{align*}
Add the e-values in $\mathcal N_0\cap\cJ_r$ to both sides. Each remaining true-null direction other than $e_\star$ is outside $\cJ_r$ and thus has e-value at least $t$. Adding those e-values gives
\begin{equation}
\sum_{e\in\mathcal N_0}E_{e,r}^{\cC}
\ge(1+|\mathcal N_0\cap\cJ_r|)t
 +(h_0-|\mathcal N_0\cap\cJ_r|-1)t
=h_0t.
\label{eq:proof-rejection-implication}
\end{equation}
Dividing (\ref{eq:proof-rejection-implication}) by $h_0$ proves the pointwise event inclusion
\begin{equation}
\{\cE_r\cap\mathcal N_0\ne\varnothing\}
\subseteq\{A_r\ge1/\alpha\},\qquad 1\le r\le R.
\label{eq:proof-false-edge-crossing}
\end{equation}
This is the deterministic implication underlying the direct e-Holm rule of \citet[Theorem~4.2]{hartoglei2025}.

For completeness, the needed maximal inequality follows by a finite stopping argument. Let $T=\inf\{1\le r\le R:A_r\ge1/\alpha\}$, with $T=\infty$ if there is no crossing. For any finite $n\le R$,
\[
A_{T\wedge n}
=A_0+\sum_{s=1}^n\mathbbm1\{T\ge s\}(A_s-A_{s-1}).
\]
For $1\le s\le n$, the indicator multiplying $A_s-A_{s-1}$ is
\[
\mathbbm1\{T\ge s\}
=\prod_{q=1}^{s-1}\mathbbm1\{A_q<1/\alpha\}\in\{0,1\},
\]
where the empty product for $s=1$ equals one. Since each $A_q$ is $\cF_q$-measurable and $\cF_q\subseteq\cF_{s-1}$ for $q<s$, $\mathbbm1\{T\ge s\}$ is $\cF_{s-1}$-measurable. Applying the tower property to the finite stopped sum gives
\begin{align*}
\E[A_{T\wedge n}\mid\cF_0]
&=1+\sum_{s=1}^n\E\!\left[
 \mathbbm1\{T\ge s\}
 \E[A_s-A_{s-1}\mid\cF_{s-1}]
 \middle|\cF_0\right]\\
&\le1,
\end{align*}
because the indicator is nonnegative and each inner conditional expectation is nonpositive. Nonnegativity of $A_{T\wedge n}$ and the definition of $T$ imply
\[
A_{T\wedge n}\ge\alpha^{-1}\mathbbm1\{T\le n\},
\qquad
\Pp(T\le n\mid\cF_0)\le\alpha.
\]
By (\ref{eq:proof-false-edge-crossing}), $F_R\subseteq\{T\le R\}$ when $h_0>0$; when $h_0=0$, $F_R$ is empty. Therefore $\Pp(F_R\mid\cF_0)\le\alpha$ almost surely in both cases.
Since $\sigma(\cB,\pi)\subseteq\cF_0$, the tower property gives
\[
\Pp(F_R\mid\cB,\pi)
=\E[\Pp(F_R\mid\cF_0)\mid\cB,\pi]\le\alpha.
\]
This is exactly (\ref{eq:main-anytime-fwer}), which averages over any independent initial design selection rather than changing the target. If the assumptions hold for every run, apply this bound for each positive integer $n$. The events $\{\max_{1\le r\le n}A_r\ge1/\alpha\}$ increase to the event of any finite-time crossing, giving the same
bound over all times. This is Ville's argument in the nonnegative-supermartingale framework \citep{howard2020}.

For the path inequalities, acyclicity, rank coverage, and top-set certificates in Theorem~\ref{thm:bedge-main-iclr}, consider
\[
\Omega_R=F_R^c
=\bigcap_{r=1}^R\{\cE_r\cap\mathcal N_0=\varnothing\}.
\]
We have proved $\Pp(\Omega_R\mid\cB,\pi)\ge1-\alpha$. Fix an outcome in $\Omega_R$ and a completed run $1\le r\le R$. By the definitions of $\Omega_R$ and $\mathcal N_0$, each $(a\to b)\in\cE_r$ satisfies $\theta_a^{\cB}-\theta_b^{\cB}>\tau$. For a path $a=v_0\to v_1\to\cdots\to v_d=b$ of $d\ge1$ such edges, summing the $d$ inequalities gives
\begin{equation}
\theta_a^{\cB}-\theta_b^{\cB}
=\sum_{j=0}^{d-1}
 (\theta_{v_j}^{\cB}-\theta_{v_{j+1}}^{\cB})>d\tau.
\label{eq:proof-path-margin}
\end{equation}
For a directed cycle, $v_d=v_0$, and (\ref{eq:proof-path-margin}) would give $0>d\tau\ge0$, a contradiction. Thus $(\cV,\cE_r)$ is acyclic. For any distinct vertices linked by a path, $d\ge1$ and $\tau\ge0$ imply $\Delta_{a,b}^{\cB}>d\tau\ge\tau$, so the derived edge is a true margin claim as well.

For a model $\ell$, let $B_\ell=\{j\ne\ell:\theta_j^{\cB}>\theta_\ell^{\cB}\}$. Recall that $A_{\ell,r}$ contains the vertices with a directed path to $\ell$, and $D_{\ell,r}$ contains the vertices reached by a directed path from $\ell$, in both cases excluding $\ell$ itself. Equation (\ref{eq:proof-path-margin}) gives
\[
A_{\ell,r}\subseteq B_\ell,\qquad
D_{\ell,r}\subseteq\{j\ne\ell:\theta_j^{\cB}<\theta_\ell^{\cB}\}.
\]
The second set is disjoint from $B_\ell$, so $|A_{\ell,r}|\le|B_\ell|$ and $|B_\ell|+|D_{\ell,r}|\le L-1$. With the true rank defined by $\operatorname{rank}_{B}(\ell)=1+|B_\ell|$, these inequalities give
\begin{equation}
1+|A_{\ell,r}|
\le\operatorname{rank}_{B}(\ell)
\le L-|D_{\ell,r}|,
\qquad
\operatorname{rank}_{B}(\ell)\in\cI_{\ell,r},
\label{eq:proof-main-rank-coverage}
\end{equation}
where $\cI_{\ell,r}=[1+|A_{\ell,r}|,L-|D_{\ell,r}|]$ is the interval in (\ref{eq:graph-rank-interval}). This is the rank definition used in Corollary~\ref{cor:anytime_rank_intervals}, which counts strictly better means and permits ties.

Finally, if a nonempty proper set $\mathcal T$ is certified, every $a\in\mathcal T$ reaches every $b\notin\mathcal T$. Applying (\ref{eq:proof-path-margin}) to every such pair and using $\tau\ge0$ gives
\[
\min_{a\in\mathcal T}\theta_a^{\cB}
>\max_{b\notin\mathcal T}\theta_b^{\cB}.
\]
Thus $\mathcal T$ is the top set of its stated size. Ties within either set are allowed. A tie across their boundary cannot satisfy this certificate on $\Omega_R$. All these implications hold pointwise on $\Omega_R$ for every $1\le r\le R$, model, and reported set, so no additional multiplicity correction is needed for these graph consequences.
\end{proof}

\subsection{Proof of Theorem~\ref{thm:pilot-topk}}
\label{app:proof-pilot-topk}

\begin{proof}
In Section~\ref{sec:pilot-topk}, the pilot data are $\mathcal D_0=(Z_1^{(0)},\ldots,Z_{R_0}^{(0)})$, and the confirmatory data are the independent subsequent runs $Z_1,\ldots,Z_R$. The time-zero sigma-field $\cF_0$ of the confirmatory filtration contains the pilot and the fixed design, so $\sigma(\mathcal D_0,\cB,\pi)\subseteq\cF_0$, and $\cF_r=\sigma(\cF_0,Z_1,\ldots,Z_r)$. Let $(i_1,\ldots,i_L)$ be the permutation obtained by sorting the pairs $(-\widehat\theta_\ell^{(0)},\ell)$ in increasing lexicographic order, with pilot means defined in (\ref{eq:pilot-score}). Thus larger pilot means come first, and a smaller model index comes first when two pilot means are equal. The rule gives a unique permutation by finitely many comparisons of pilot-measurable scores. Consequently, $\mathcal T_0=\{i_1,\ldots,i_k\}$ is measurable with respect to $\sigma(\mathcal D_0,\cB,\pi)$ and has exactly $k$ members. Define
\[
\mathcal U_0=\{(a\to b):a\in\mathcal T_0,\ b\notin\mathcal T_0\},
\qquad
\mathcal U_{00}=\{e\in\mathcal U_0:\Delta_e^{\cB}\le\tau\},
\]
where $\Delta_e^{\cB}=\Delta_{a,b}^{\cB}$ for $e=(a\to b)$. The family $\cH(\mathcal T_0)$ consists of the hypotheses indexed by $\mathcal U_0$, and $\mathcal U_{00}$ contains precisely its true nulls. Both sets are measurable with respect to $\cF_0$, i.e.,  membership in $\mathcal U_0$ is determined by the pilot, and the benchmark means are fixed given $(\cB,\pi)$. In particular, $|\mathcal U_0|=k(L-k)>0$, while $h_0=|\mathcal U_{00}|$ is allowed to be zero. We first work under the conditional law given $\cF_0$, for almost every initial design.

For any bounded measurable function $h$ of a complete panel, Assumption~\ref{ass:design-iclr}(1) and then (3) imply
\begin{align}
\E[h(Z_r)\mid\cF_{r-1}]
&=\E[h(Z_r)\mid\cF_0]
 =\E[h(Z_r)\mid\cB,\pi],
\label{eq:pilot-execution-law-detail}
\end{align}
where the first equality uses conditional independence of the new run from the preceding confirmatory runs, and the second uses independence of the entire initial information, including the pilot, from the confirmatory panels given $(\cB,\pi)$. Conditional identical distribution across runs makes the last expectation independent of $r$. Thus, for each fixed direction $e=(a\to b)$ and its block means $\nu_{e,m,r}=\E[Y_{e,m,r}\mid\cF_{r-1}]$,
\begin{align}
\sum_{m=1}^M w_m\nu_{e,m,r}
&=\E\!\left[\sum_{m=1}^M w_mY_{e,m,r}\middle|\cF_{r-1}\right]
 =\E\!\left[\frac1N\sum_{i=1}^N X_{i,r}^{e}
             \middle|\cF_{r-1}\right]\nonumber\\
&=\frac12+\frac{1}{2N}\sum_{i=1}^N
 \E[S_{i,r}^{(a)}-S_{i,r}^{(b)}\mid\cB,\pi]
 =\frac{1+\Delta_e^{\cB}}2,
\label{eq:pilot-mean-detail}
\end{align}
where the first equality uses predictable weights, the second the partition identity, and the third (\ref{eq:pilot-execution-law-detail}).
These equalities hold simultaneously for all directions in the finite model universe, hence also for the pilot-selected directions. Under $e\in\mathcal U_{00}$, their common value is at most $\mu_0$. No blockwise inequality $\nu_{e,m,r}\le\mu_0$ has been imposed.

We verify the evidence property for this conditional null family.
For a fixed grid component set $u_m=\lambda_gw_m/w_\star\in[0,1)$ and $p_{m,r}=\widehat Y_{e,m,r}$. Define
\[
Q_{m,r}=\exp\{u_m(Y_{e,m,r}-\nu_{e,m,r})
                   -\psi_E(u_m)(Y_{e,m,r}-p_{m,r})^2\}.
\]
Assumption~\ref{ass:design-iclr}(4) makes $p_{m,r}$ predictable and bounded, so Lemma~\ref{lem:conditional-eb-proof} gives $\E[Q_{m,r}\mid\cF_{r-1}]\le1$. Conditional on this history, the $Q_{m,r}$ are functions of the separate block vectors. Their mutual conditional independence follows from Assumption~\ref{ass:design-iclr}(2). Consequently, for $e\in\mathcal U_{00}$,
\begin{align}
\E[F_{e,r}^{\cC}(\lambda_g)\mid\cF_{r-1}]
&=\exp\!\left\{\frac{\lambda_g}{w_\star}
       \left(\sum_m w_m\nu_{e,m,r}-\mu_0\right)\right\}
       \prod_m\E[Q_{m,r}\mid\cF_{r-1}]\nonumber\\
&\le1.
\label{eq:pilot-factor-detail}
\end{align}
For clarity, $F_{e,r}^{\cC}(\lambda_g)$ is the exponential factor in (\ref{eq:block-factor}), evaluated on confirmatory run $r$. The bounds preceding (\ref{eq:proof-evidence-bounds}) apply here and give
\[
-\sum_{m=1}^M\{u_m+\psi_E(u_m)\}
\le\log F_{e,r}^{\cC}(\lambda_g)\le N.
\]
Since $u_m<1$, the lower bound is finite, and hence $0<F_{e,r}^{\cC}(\lambda_g)\le\exp(N)$. For a single grid value $\lambda_g$, the corresponding process and the mixture over the grid are, respectively,
\[
E_{e,r}^{\cC}(\lambda_g)=\prod_{s=1}^rF_{e,s}^{\cC}(\lambda_g),
\qquad
E_{e,r}^{\cC}=\sum_{g=1}^G\rho_gE_{e,r}^{\cC}(\lambda_g).
\]
The mixture coefficients satisfy $\rho_g\ge0$ and $\sum_g\rho_g=1$. Thus both process values lie in $(0,\exp(rN)]$, which gives
\[
\E[|E_{e,r}^{\cC}(\lambda_g)|\mid\cF_0]\le\exp(rN)<\infty,
\qquad
\E[|E_{e,r}^{\cC}|\mid\cF_0]\le\exp(rN)<\infty.
\]
This is the finite-time integrability required for the conditional expectations below. Each $\rho_g$ is $\cF_0$-measurable, and $E_{e,r-1}^{\cC}(\lambda_g)$ is $\cF_{r-1}$-measurable. For $e\in\mathcal U_{00}$, the product recursion and (\ref{eq:pilot-factor-detail}) therefore give
\begin{align*}
\E[E_{e,r}^{\cC}\mid\cF_{r-1}]
&=\sum_g\rho_gE_{e,r-1}^{\cC}(\lambda_g)\E[F_{e,r}^{\cC}(\lambda_g)\mid\cF_{r-1}]\\
&\le\sum_g\rho_gE_{e,r-1}^{\cC}(\lambda_g)=E_{e,r-1}^{\cC}.
\end{align*}
At $r=0$, the empty products equal one, so $E_{e,0}^{\cC}=1$. For $s\ge1$, the observed block mean $Y_{e,m,s}$ in each factor is computed from $Z_s$, not from any pilot run $Z_j^{(0)}$. In particular, no factor computed from a pilot run is multiplied into $E_{e,r}^{\cC}(\lambda_g)$. The pilot-selected family and any allowed initial design information are fixed when conditioning on $\cF_0$, and they are not additional confirmatory observations.

If $h_0=0$, no true cross-set null can be rejected. Suppose henceforth that $h_0\ge1$, and define
\[
A_r=\frac1{h_0}\sum_{e\in\mathcal U_{00}}E_{e,r}^{\cC}.
\]
Because $h_0$ and $\mathcal U_{00}$ are $\cF_0$-measurable, linearity of conditional expectation gives
\[
\E[A_r\mid\cF_{r-1}]
=\frac1{h_0}\sum_{e\in\mathcal U_{00}}
       \E[E_{e,r}^{\cC}\mid\cF_{r-1}]
\le\frac1{h_0}\sum_{e\in\mathcal U_{00}}E_{e,r-1}^{\cC}
=A_{r-1}.
\]
Also $A_0=h_0^{-1}\sum_{e\in\mathcal U_{00}}1=1$. The calculation concerns a sum, not a product, so dependence among the directional processes does not change the equality. For $t=1/\alpha$, the restricted-family cutoff is
\[
J_r=\{e\in\mathcal U_0:E_{e,r}^{\cC}<t\},\qquad
c_r=t+\sum_{e\in J_r}(t-E_{e,r}^{\cC}).
\]
Suppose a true-null direction $e_*\in\mathcal U_{00}$ is rejected at $r$. Then $E_{e_*,r}^{\cC}\ge c_r\ge t$, so $e_*\notin J_r$. Set $J_{0,r}=J_r\cap\mathcal U_{00}$ and $j_0=|J_{0,r}|$. For every $e\in J_r$, $E_{e,r}^{\cC}<t$, so $t-E_{e,r}^{\cC}>0$. Hence
\begin{align*}
E_{e_*,r}^{\cC}\ge c_r
&=t+\sum_{e\in J_{0,r}}(t-E_{e,r}^{\cC})
    +\sum_{e\in J_r\setminus J_{0,r}}(t-E_{e,r}^{\cC})\\
&\ge t+\sum_{e\in J_{0,r}}(t-E_{e,r}^{\cC}),
\end{align*}
where the second sum is nonnegative, including when it is empty. There are $h_0-j_0-1$ directions in $\mathcal U_{00}\setminus(J_r\cup\{e_*\})$. Each satisfies $E_{e,r}^{\cC}\ge t$ because it is outside $J_r$. Therefore
\begin{align}
\sum_{e\in\mathcal U_{00}}E_{e,r}^{\cC}
&\ge E_{e_*,r}^{\cC}+\sum_{e\in J_{0,r}}E_{e,r}^{\cC}+(h_0-j_0-1)t\nonumber\\
&\ge t+j_0t+(h_0-j_0-1)t=h_0t.
\label{eq:pilot-cutoff-detail}
\end{align}
Thus any false rejection requires $A_r\ge t$. This derives the needed direct e-Holm implication rather than assuming independent tests \citep[Theorem~4.2]{hartoglei2025}.

Let $T=\inf\{1\le r\le R:A_r\ge t\}$, with $\inf\varnothing=\infty$. For $1\le n\le R$, the stopped process satisfies the pathwise identity
\[
A_{T\wedge n}=1+
 \sum_{s=1}^n\mathbbm1\{T\ge s\}(A_s-A_{s-1}).
\]
Since $\{T\ge s\}=\bigcap_{j=1}^{s-1}\{A_j<t\}\in\cF_{s-1}$, the tower property and the supermartingale inequalities give
\begin{align*}
\E[A_{T\wedge n}\mid\cF_0]
&=1+\sum_{s=1}^n\E\!\left[
 \mathbbm1\{T\ge s\}\E[A_s-A_{s-1}\mid\cF_{s-1}]
 \middle|\cF_0\right]\le1.
\end{align*}
Nonnegativity gives $A_{T\wedge n}\ge t\mathbbm1\{T\le n\}$, and hence $\Pp(T\le n\mid\cF_0)\le1/t=\alpha$. Let $F_R$ denote the false-rejection event in (\ref{eq:topk-anytime-fwer}). Taking $n=R$ and using (\ref{eq:pilot-cutoff-detail}) gives $\Pp(F_R\mid\cF_0)\le\alpha$. This finite stopped-sum argument is the form of Ville's inequality needed here \citep{howard2020}.
Since $\sigma(\mathcal D_0,\cB,\pi)\subseteq\cF_0$, the tower property gives
\[
\Pp(F_R\mid\mathcal D_0,\cB,\pi)
=\E[\Pp(F_R\mid\cF_0)\mid\mathcal D_0,\cB,\pi]\le\alpha,
\]
which proves (\ref{eq:topk-anytime-fwer}) for almost every pilot realization. Averaging once more over the pilot data yields
\[
\Pp(F_R\mid\cB,\pi)
=\E[\Pp(F_R\mid\mathcal D_0,\cB,\pi)\mid\cB,\pi]\le\alpha,
\]
including pilot outcomes for which $h_0=0$.

For the Top-$k$ certificate, let $C_r=\{\mathcal U_0\subseteq\cE_r\}$. On $F_R^c\cap C_r$, every cross-set gap exceeds $\tau$. The two sets are finite and nonempty, so choose $a_*\in\arg\min_{a\in\mathcal T_0}\theta_a^{\cB}$ and $b_*\in\arg\max_{b\notin\mathcal T_0}\theta_b^{\cB}$. Then
\begin{equation}
\min_{a\in\mathcal T_0}\theta_a^{\cB}
-\max_{b\notin\mathcal T_0}\theta_b^{\cB}
=\Delta_{a_*,b_*}^{\cB}>\tau\ge0.
\label{eq:pilot-topset-detail}
\end{equation}
The same argument applies to every $1\le r\le R$ on $F_R^c$.
Conversely, if $\min_{a\in\mathcal T_0}\theta_a^{\cB}  \le\max_{b\notin\mathcal T_0}\theta_b^{\cB}+\tau$,
then $\Delta_{a_*,b_*}^{\cB}\le\tau$. If $C_r$ also occurs, $(a_*\to b_*)\in\mathcal U_{00}\cap\cE_r$, so $F_R$ occurs.
This proves
\begin{equation}
\left\{\exists r\le R:C_r\ \text{and}\quad
 \min_{a\in\mathcal T_0}\theta_a^{\cB}
 \le\max_{b\notin\mathcal T_0}\theta_b^{\cB}+\tau\right\}
\subseteq F_R.
\label{eq:pilot-certificate-inclusion}
\end{equation}
Taking complements and using the conditional error bound proves the precise certificate assertion
\[
\Pp\!\left(
 \forall\,1\le r\le R:\ C_r\ \Longrightarrow\quad
 \min_{a\in\mathcal T_0}\theta_a^{\cB}
 >\max_{b\notin\mathcal T_0}\theta_b^{\cB}+\tau
 \,\middle|\,\mathcal D_0,\cB,\pi\right)\ge1-\alpha.
\]
Since $\tau\ge0$, every member of a certified set has a larger benchmark mean than every model outside it on $F_R^c$.

For $1\le r\le R$, write
$C_{\le r}=\bigcup_{s=1}^r C_s$, where
$C_s=\{\mathcal U_0\subseteq\cE_s\}$. Thus $C_{\le R}$ is the event
that a certificate is issued by run $R$. Define the candidate's
margin-failure event by
\[
B_0=\left\{
\min_{a\in\mathcal T_0}\theta_a^{\cB}
\le\max_{b\notin\mathcal T_0}\theta_b^{\cB}+\tau\right\}.
\]
The inclusion (\ref{eq:pilot-certificate-inclusion}) gives
\begin{equation}
\Pp(C_{\le R}\cap B_0\mid\mathcal D_0,\cB,\pi)\le\alpha.
\label{eq:pilot-certificate-error}
\end{equation}
Given $(\mathcal D_0,\cB,\pi)$, the candidate and its true
benchmark means are fixed, so $B_0$ is already determined. For any
such pilot outcome in $B_0$, (\ref{eq:pilot-certificate-error}) reduces to
$\Pp(C_{\le R}\mid\mathcal D_0,\cB,\pi)\le\alpha$.
An incorrect pilot selection is therefore allowed without assuming that
the selected set satisfies the required margin.

For this fixed-candidate procedure, the event $A_{\le r}$ in
(\ref{eq:metric-certificate-events}) is exactly $C_{\le r}\cap B_0$.
Since $C_{\le r}\subseteq C_{\le R}$, monotonicity of conditional
expectation and (\ref{eq:pilot-certificate-error}) give, almost surely,
\[
\Pp(C_{\le r}\cap B_0\mid\mathcal D_0,\cB,\pi)
\le\Pp(C_{\le R}\cap B_0\mid\mathcal D_0,\cB,\pi)\le\alpha.
\]
Averaging over the pilot data with the tower property gives the
error metric in (\ref{eq:metric-fcp}),
\begin{align*}
\operatorname{FCP}_r
&=\Pp(C_{\le r}\cap B_0\mid\cB,\pi)\\
&=\E\!\left[\Pp(C_{\le r}\cap B_0\mid\mathcal D_0,\cB,\pi)
\,\middle|\,\cB,\pi\right]\le\alpha.
\end{align*}
If $\operatorname{CertProb}_r=\Pp(C_{\le r}\mid\cB,\pi)>0$, then
\[
\Pp(B_0\mid C_{\le r},\cB,\pi)
=\frac{\Pp(C_{\le r}\cap B_0\mid\cB,\pi)}
{\Pp(C_{\le r}\mid\cB,\pi)}
=\frac{\operatorname{FCP}_r}{\operatorname{CertProb}_r}.
\]
Therefore, the guarantee $\operatorname{FCP}_r\le\alpha$ controls the probability of issuing an incorrect certificate, rather than the conditional error probability given that certification occurs. This distinction is particularly clear after conditioning on the pilot. That is, given $(\mathcal D_0,\cB,\pi)$, the event $B_0$ is fixed. Hence, if $B_0$ holds for a particular pilot realization and certification occurs with positive conditional probability, then
\[
\Pp(B_0\mid C_{\le R},\mathcal D_0,\cB,\pi)=1.
\]

Finally, if some $a\in\mathcal T_0$ and $b\notin\mathcal T_0$
have $\theta_a^{\cB}=\theta_b^{\cB}$, then
$\Delta_{a,b}^{\cB}=0\le\tau$. On $F_R^c$, this direction is
never rejected, so the requirement $\mathcal U_0\subseteq\cE_r$
cannot be met at any $r\le R$. This concerns a tie in true benchmark
means across the two groups, not a tie in the pilot scores.
\end{proof}

\subsection{Proof of Corollary~\ref{cor:anytime_rank_intervals}}
\label{app:proof-rank-intervals}

\begin{proof}
Fix an integer $R\ge1$ and consider the first $R$ completed runs.
For the true-null set $\mathcal N_0$ in
Theorem~\ref{thm:bedge-main-iclr}, define
\begin{equation}
\begin{aligned}
\Omega_R
&=\bigcap_{r=1}^R\{\cE_r\cap\mathcal N_0=\varnothing\}\\
&=\left\{\forall\,1\le r\le R,\ \forall(a\to b)\in\cE_r:
\theta_a^{\cB}-\theta_b^{\cB}=\Delta_{a,b}^{\cB}>\tau\right\}.
\end{aligned}
\label{eq:rank-true-edge-event}
\end{equation}
By (\ref{eq:main-anytime-fwer}),
$\Pp(\Omega_R\mid\cB,\pi)\ge1-\alpha$.
For the interval in (\ref{eq:graph-rank-interval}), write
\[
\mathcal I_{\ell,r}=\{x\in\mathbb R:1+|\mathcal A_{\ell,r}|\le x\le L-|\mathcal D_{\ell,r}|\}.
\]
Define the joint coverage event
\[
C_R=\bigcap_{r=1}^R\bigcap_{\ell\in\mathcal V}\{\operatorname{rank}_{B}(\ell)\in\mathcal I_{\ell,r}\}.
\]
We will prove $\Omega_R\subseteq C_R$.

Fix $\omega\in\Omega_R$, $1\le r\le R$, and $\ell\in\mathcal V$.
All graph quantities in the following calculation are evaluated at
$\omega$. A directed cycle
$v_0\to v_1\to\cdots\to v_d=v_0$ with $d\ge1$ would imply
\[
0=\theta_{v_0}^{\cB}-\theta_{v_d}^{\cB}
 =\sum_{j=0}^{d-1}
  (\theta_{v_j}^{\cB}-\theta_{v_{j+1}}^{\cB})
 >d\tau\ge0,
\]
This gives $0>0$, a contradiction. Thus $(\mathcal V,\cE_r)$ is acyclic.

Define the true sets of strictly better, strictly worse, and tied models by
\begin{align*}
\mathcal B_\ell&=\{j\ne\ell:\theta_j^{\cB}>\theta_\ell^{\cB}\},\\
\mathcal W_\ell&=\{j\ne\ell:\theta_j^{\cB}<\theta_\ell^{\cB}\},\\
\mathcal O_\ell&=\{j\ne\ell:\theta_j^{\cB}=\theta_\ell^{\cB}\}.
\end{align*}
They are pairwise disjoint and satisfy
\[
\mathcal B_\ell\cup\mathcal W_\ell\cup\mathcal O_\ell
=\mathcal V\setminus\{\ell\},\qquad
|\mathcal B_\ell|+|\mathcal W_\ell|+|\mathcal O_\ell|=L-1.
\]
If
$j\in\mathcal A_{\ell,r}$, there exists a path
$j=v_0\to\cdots\to v_d=\ell$ of length $d\ge1$.
Summing its true edge inequalities gives
\[
\theta_j^{\cB}-\theta_\ell^{\cB}
=\sum_{q=0}^{d-1}
 (\theta_{v_q}^{\cB}-\theta_{v_{q+1}}^{\cB})>d\tau\ge0.
\]
Thus $j\in\mathcal B_\ell$. If instead
$j\in\mathcal D_{\ell,r}$, choose a path
$\ell=u_0\to\cdots\to u_q=j$, where $q\ge1$. Then
\[
\theta_\ell^{\cB}-\theta_j^{\cB}
=\sum_{s=0}^{q-1}
 (\theta_{u_s}^{\cB}-\theta_{u_{s+1}}^{\cB})>q\tau\ge0,
\]
so $j\in\mathcal W_\ell$. We have proved the two inclusions
\begin{equation}
\mathcal A_{\ell,r}\subseteq\mathcal B_\ell,
\qquad
\mathcal D_{\ell,r}\subseteq\mathcal W_\ell
\subseteq(\mathcal V\setminus\{\ell\})\setminus\mathcal B_\ell.
\label{eq:rank-set-inclusions-detail}
\end{equation}
In particular,
\[
\mathcal A_{\ell,r}\cap\mathcal D_{\ell,r}=\varnothing,
\qquad
(\mathcal A_{\ell,r}\cup\mathcal D_{\ell,r})
                  \cap\mathcal O_\ell=\varnothing.
\]
The rank in the corollary is
$\operatorname{rank}_{B}(\ell)=1+|\mathcal B_\ell|$.
Using (\ref{eq:rank-set-inclusions-detail}) and the partition above,
the differences between this rank and the two interval endpoints are
\begin{align}
\operatorname{rank}_{B}(\ell)-(1+|\mathcal A_{\ell,r}|)
&=|\mathcal B_\ell|-|\mathcal A_{\ell,r}|
 =|\mathcal B_\ell\setminus\mathcal A_{\ell,r}|\ge0,\nonumber\\
(L-|\mathcal D_{\ell,r}|)-\operatorname{rank}_{B}(\ell)
&=L-1-|\mathcal B_\ell|-|\mathcal D_{\ell,r}|\nonumber\\
&=|\mathcal W_\ell\setminus\mathcal D_{\ell,r}|
   +|\mathcal O_\ell|\ge0.
\label{eq:rank-endpoint-differences}
\end{align}
Consequently,
\begin{equation}
\begin{aligned}
1+|\mathcal A_{\ell,r}|
\le1+|\mathcal B_\ell|
=\operatorname{rank}_{B}(\ell)
\le L-|\mathcal D_{\ell,r}|,\\
\operatorname{rank}_{B}(\ell)\in
\mathcal I_{\ell,r}
=[1+|\mathcal A_{\ell,r}|,\ L-|\mathcal D_{\ell,r}|]
\ne\varnothing.
\end{aligned}
\label{eq:rank-counting-detail}
\end{equation}
The calculation holds for every $r\le R$ and every $\ell$ at the
fixed outcome $\omega$. Since $\omega\in\Omega_R$ was arbitrary,
\begin{equation}
\Omega_R\subseteq C_R,
\qquad
\mathbbm1\{\Omega_R\}\le\mathbbm1\{C_R\}.
\label{eq:rank-coverage-event-detail}
\end{equation}

To justify taking conditional expectations in
(\ref{eq:rank-coverage-event-detail}), note first that
$E_{e,r}^{\cC}$ and the cutoff $c_r$ are $\cF_r$-measurable.
Thus $\{e\in\cE_r\}=\{E_{e,r}^{\cC}\ge c_r\}\in\cF_r$
for every tested direction; an untested direction has an empty edge
event. For distinct vertices $a,b$, reachability satisfies
\begin{equation}
\{a\leadsto_r b\}
=\bigcup_{d=1}^{L-1}
 \ \bigcup_{\substack{v_0=a,\ v_d=b\\
                       v_0,\ldots,v_d\ \mathrm{distinct}}}
       \ \bigcap_{j=0}^{d-1}
       \{(v_j\to v_{j+1})\in\cE_r\}
\in\cF_r.
\label{eq:rank-reachability-measurable}
\end{equation}
Indeed, deleting a segment between repeated vertices shortens a
walk without changing its endpoints; successive deletions give a
path with at most $L-1$ edges. The unions in
(\ref{eq:rank-reachability-measurable}) are therefore finite. It follows that
\[
|\mathcal A_{\ell,r}|=\sum_{j\ne\ell}\mathbbm1\{j\leadsto_r\ell\},
\qquad
|\mathcal D_{\ell,r}|=\sum_{j\ne\ell}\mathbbm1\{\ell\leadsto_r j\}
\]
are $\cF_r$-measurable. The true rank is $\cF_0$-measurable, so
\begin{align*}
\{\operatorname{rank}_{B}(\ell)\in\mathcal I_{\ell,r}\}
&=\{1+|\mathcal A_{\ell,r}|\le\operatorname{rank}_{B}(\ell)\}\\
&\quad\cap\{\operatorname{rank}_{B}(\ell)\le L-|\mathcal D_{\ell,r}|\}
\in\cF_r.
\end{align*}
Hence $C_R,\Omega_R\in\cF_R$. Taking conditional expectations of
the indicator inequality in (\ref{eq:rank-coverage-event-detail}) gives
\begin{align}
\Pp(C_R\mid\cB,\pi)
&=\E[\mathbbm1\{C_R\}\mid\cB,\pi]\nonumber\\
&\ge\E[\mathbbm1\{\Omega_R\}\mid\cB,\pi]
 =\Pp(\Omega_R\mid\cB,\pi)\ge1-\alpha.
\label{eq:rank-finite-coverage}
\end{align}
This proves the corollary for the first $R$ completed runs.

For the additional statement following the corollary, suppose
Assumption~\ref{ass:design-iclr} holds for one sequence $(Z_r)_{r\ge1}$.
For every integer $n\ge1$, (\ref{eq:rank-finite-coverage}) applies to
its first $n$ runs. The corresponding events satisfy
\[
C_{n+1}\subseteq C_n,\qquad
C_\infty:=\bigcap_{n=1}^\infty C_n
=\bigcap_{r=1}^\infty\bigcap_{\ell\in\mathcal V}
       \{\operatorname{rank}_{B}(\ell)\in\mathcal I_{\ell,r}\}.
\]
Since $\mathbbm1\{C_n^c\}\uparrow\mathbbm1\{C_\infty^c\}$,
conditional monotone convergence gives
\begin{align*}
\Pp(C_\infty\mid\cB,\pi)
&=1-\E[\mathbbm1\{C_\infty^c\}\mid\cB,\pi]\\
&=1-\lim_{n\to\infty}\E[\mathbbm1\{C_n^c\}\mid\cB,\pi]
 =\lim_{n\to\infty}\Pp(C_n\mid\cB,\pi)\ge1-\alpha
 \quad\text{a.s.}
\end{align*}
The bounds $\Pp(C_n\mid\cB,\pi)\ge1-\alpha$ in
(\ref{eq:rank-finite-coverage}) hold simultaneously for all $n$
outside a countable union of probability-zero sets.
\end{proof}

\subsection{Proof of Theorem~\ref{thm:update_benchmark}}
\label{app:proof-epochs}

\begin{proof}
For an epoch $v$ that is entered, $\mathcal K_{v-1}$ is the sigma-field containing all information available before its fresh runs, as defined in Section~\ref{sec:epoch-reset-procedure}.
Its within-epoch filtration is
\[
\mathcal F^{(v)}_0=\mathcal K_{v-1},\qquad
\mathcal F^{(v)}_r
=\sigma(\mathcal K_{v-1},Z^{(v)}_1,\ldots,Z^{(v)}_r).
\]
Here $Z^{(v)}_r$ contains all scores on the frozen benchmark
$B^{(v)}$. In this proof, $S_{i,r}^{(\ell)}$ denotes its component
for item $i$ and model $\ell$ in epoch $v$.
We first condition on $\mathcal K_{v-1}$. By the design in
Section~\ref{sec:epoch-reset-procedure},
\[
\sigma\!\left(B^{(v)},\mathcal H^{(v)},\mathcal C^{(v)},
\Lambda^{(v)},(\rho_g^{(v)})_g,\alpha_v\right)
\subseteq\mathcal K_{v-1},
\]
where $\mathcal C^{(v)}$ is the partition and $\Lambda^{(v)}$ is
the stake grid in epoch $v$. Define
\[
\mathcal U^{(v)}=\{e:H_e^{(v,\tau)}\in\mathcal H^{(v)}\},
\qquad
\mathcal N^{(v)}_0
=\{e\in\mathcal U^{(v)}:\Delta_e^{B^{(v)}}\le\tau\},
\]
\[
\Delta_e^{B^{(v)}}=\Delta_{a,b}^{B^{(v)}}\quad(e=(a\to b)),
\qquad h_v=|\mathcal N^{(v)}_0|.
\]
Since the benchmark means are functions of $(B^{(v)},\pi)$,
\[
\mathbbm1\{e\in\mathcal N^{(v)}_0\}
=\mathbbm1\{e\in\mathcal U^{(v)}\}
 \mathbbm1\{\Delta_e^{B^{(v)}}\le\tau\}
\]
is $\mathcal K_{v-1}$-measurable for each direction $e$, as is $h_v$.

Write $P_{b,\pi}$ for the fixed execution law of a fresh complete
panel on benchmark $b$. The execution-law condition in the theorem
means that, for every bounded measurable panel function $h$,
\begin{equation}
\E[h(Z^{(v)}_r)\mid\mathcal F^{(v)}_{r-1}]
=\E[h(Z^{(v)}_r)\mid\mathcal K_{v-1}]
=\int h(z)\,P_{B^{(v)},\pi}(dz).
\label{eq:epoch-execution-law-detail}
\end{equation}
The first equality uses conditional run independence in
Assumption~\ref{ass:design-iclr}(1); the second uses the fixed
execution law in Theorem~\ref{thm:update_benchmark}.
Because $\sigma(B^{(v)},\pi)\subseteq\mathcal K_{v-1}$, the
tower property also gives
\begin{align*}
\E[h(Z^{(v)}_r)\mid B^{(v)},\pi]
&=\E\!\left[\E[h(Z^{(v)}_r)\mid\mathcal K_{v-1}]
                       \middle|B^{(v)},\pi\right]\\
&=\E\!\left[\int h(z)P_{B^{(v)},\pi}(dz)
                       \middle|B^{(v)},\pi\right]
 =\int h(z)P_{B^{(v)},\pi}(dz).
\end{align*}
Thus the conditional expectations in
(\ref{eq:epoch-execution-law-detail}) use the same law as the
benchmark means $\theta_\ell^{B^{(v)}}$.

Let the epoch partition have $M_v$ nonempty blocks with
$w^{(v)}_m=|C^{(v)}_m|/N_v$ and
$w^{(v)}_\star=\max_m w^{(v)}_m>0$. For $e=(a\to b)$, define
\[
Y^{(v)}_{e,m,r}
=\frac1{|C_m^{(v)}|}\sum_{i\in C_m^{(v)}}
          \frac{S_{i,r}^{(a)}-S_{i,r}^{(b)}+1}{2},
\qquad
\nu^{(v)}_{e,m,r}
=\E[Y^{(v)}_{e,m,r}\mid\mathcal F^{(v)}_{r-1}].
\]
The partition identity and
(\ref{eq:epoch-execution-law-detail}) give
\begin{align}
\sum_{m=1}^{M_v}w^{(v)}_m\nu^{(v)}_{e,m,r}
&=\E\!\left[\frac1{N_v}\sum_{i\in B^{(v)}}
 \frac{S_{i,r}^{(a)}-S_{i,r}^{(b)}+1}{2}
 \middle|\mathcal F^{(v)}_{r-1}\right]\nonumber\\
&=\frac{1+\Delta_e^{B^{(v)}}}{2}\le\mu_0,
\qquad e\in\mathcal N^{(v)}_0.
\label{eq:epoch-average-null-detail}
\end{align}
Here $e\in\mathcal N^{(v)}_0$ implies
$(1+\Delta_e^{B^{(v)}})/2\le(1+\tau)/2=\mu_0$.

For each grid value $\lambda_g^{(v)}$, define
$u^{(v)}_{g,m}=\lambda^{(v)}_gw^{(v)}_m/w^{(v)}_\star\in[0,1)$.
The prediction $\widehat Y^{(v)}_{e,m,r}$ is
$\mathcal F^{(v)}_{r-1}$-measurable and belongs to $[0,1]$ by
Assumption~\ref{ass:design-iclr}(4). Write
\[
Q^{(v)}_{g,m,r}
=\exp\!\left\{u^{(v)}_{g,m}(Y^{(v)}_{e,m,r}-\nu^{(v)}_{e,m,r})
 -\psi_E(u^{(v)}_{g,m})
    (Y^{(v)}_{e,m,r}-\widehat Y^{(v)}_{e,m,r})^2\right\}.
\]
Lemma~\ref{lem:conditional-eb-proof} gives
$\E[Q^{(v)}_{g,m,r}\mid\mathcal F^{(v)}_{r-1}]\le1$.
Given $\mathcal F^{(v)}_{r-1}$, each $Q^{(v)}_{g,m,r}$ is a
function of the block score vector
$V^{(v)}_{m,r}=(S_{i,r}^{(\ell)}:i\in C_m^{(v)},\ 1\le\ell\le L)$.
Assumption~\ref{ass:design-iclr}(2) therefore yields
\[
\E\!\left[\prod_{m=1}^{M_v}Q^{(v)}_{g,m,r}
               \middle|\mathcal F^{(v)}_{r-1}\right]
=\prod_{m=1}^{M_v}
   \E[Q^{(v)}_{g,m,r}\mid\mathcal F^{(v)}_{r-1}]\le1.
\]
For $e\in\mathcal N^{(v)}_0$, (\ref{eq:epoch-average-null-detail}) gives
\begin{equation}
\sum_m u^{(v)}_{g,m}(\nu^{(v)}_{e,m,r}-\mu_0)
=\frac{\lambda^{(v)}_g}{w^{(v)}_\star}
  \left(\sum_m w^{(v)}_m\nu^{(v)}_{e,m,r}-\mu_0\right)\le0.
\label{eq:epoch-centering-sum}
\end{equation}
Decomposing the factor in (\ref{eq:block-factor}) accordingly gives
\begin{align}
\E[F^{(v)}_{e,r}(\lambda^{(v)}_g)\mid\mathcal F^{(v)}_{r-1}]
&=\exp\!\left\{\sum_m u^{(v)}_{g,m}
                      (\nu^{(v)}_{e,m,r}-\mu_0)\right\}
 \prod_m\E[Q^{(v)}_{g,m,r}\mid\mathcal F^{(v)}_{r-1}]
\le1.
\label{eq:epoch-factor-validity}
\end{align}
For grid value $\lambda^{(v)}_g$, the process is
$E^{(v)}_{e,r}(\lambda^{(v)}_g)
=\prod_{s=1}^rF^{(v)}_{e,s}(\lambda^{(v)}_g)$, and its mixture
coefficient $\rho^{(v)}_g\ge0$ is fixed at the start of epoch $v$,
with $\sum_g\rho^{(v)}_g=1$.
The bounds in (\ref{eq:proof-evidence-bounds}), applied to the
epoch partition, give
\[
0<E^{(v)}_{e,r}(\lambda^{(v)}_g)\le\exp(rN_v),\qquad
0<E^{(v)}_{e,r}=\sum_g\rho^{(v)}_g
E^{(v)}_{e,r}(\lambda^{(v)}_g)\le\exp(rN_v).
\]
Since $N_v$ is finite and $\mathcal K_{v-1}$-measurable,
\[
\E[|E^{(v)}_{e,r}(\lambda^{(v)}_g)|\mid\mathcal K_{v-1}]
\le\exp(rN_v)<\infty,\qquad
\E[|E^{(v)}_{e,r}|\mid\mathcal K_{v-1}]\le\exp(rN_v)<\infty.
\]
The product $E^{(v)}_{e,r-1}(\lambda^{(v)}_g)$ and the coefficient
$\rho^{(v)}_g$ are $\mathcal F^{(v)}_{r-1}$-measurable. The
product recursion and (\ref{eq:epoch-factor-validity}) give, for
$e\in\mathcal N^{(v)}_0$,
\begin{align}
\E[E^{(v)}_{e,r}\mid\mathcal F^{(v)}_{r-1}]
&=\sum_g\rho^{(v)}_g E^{(v)}_{e,r-1}(\lambda^{(v)}_g)
 \E[F^{(v)}_{e,r}(\lambda^{(v)}_g)
                         \mid\mathcal F^{(v)}_{r-1}]\nonumber\\
&\le\sum_g\rho^{(v)}_g E^{(v)}_{e,r-1}(\lambda^{(v)}_g)
 =E^{(v)}_{e,r-1},\qquad
E^{(v)}_{e,0}=\sum_g\rho^{(v)}_g=1.
\label{eq:epoch-supermartingale-detail}
\end{align}
Define the epoch error event by
\[
F_v=\bigcup_{r=1}^\infty
       \{\mathcal E_r^{(v)}\cap\mathcal N_0^{(v)}\ne\varnothing\}.
\]
If $\alpha_v=0$, then $\mathcal E_r^{(v)}=\varnothing$ for all $r$,
so $F_v=\varnothing$. If $h_v=0$, then
$\mathcal N_0^{(v)}=\varnothing$, again giving $F_v=\varnothing$.
It remains to consider $\alpha_v>0$ and $h_v\ge1$. Define
\[
A^{(v)}_r=\frac1{h_v}\sum_{e\in\mathcal N^{(v)}_0}E^{(v)}_{e,r},
\qquad t_v=1/\alpha_v.
\]
Since $h_v$ and $\mathcal N^{(v)}_0$ are
$\mathcal K_{v-1}$-measurable, (\ref{eq:epoch-supermartingale-detail}) gives
\[
\E[A^{(v)}_r\mid\mathcal F^{(v)}_{r-1}]
=\frac1{h_v}\sum_{e\in\mathcal N^{(v)}_0}
    \E[E^{(v)}_{e,r}\mid\mathcal F^{(v)}_{r-1}]
\le A^{(v)}_{r-1},\qquad A^{(v)}_0=1.
\]
Set
\[
J^{(v)}_r=\{e\in\mathcal U^{(v)}:E^{(v)}_{e,r}<t_v\},
\qquad
c^{(v)}_r=t_v+\sum_{e\in J^{(v)}_r}(t_v-E^{(v)}_{e,r}).
\]
If some $e_*\in\mathcal N^{(v)}_0$ is rejected, it lies outside
$J^{(v)}_r$ and satisfies $E^{(v)}_{e_*,r}\ge c^{(v)}_r$. With
$J^{(v)}_{0,r}=J^{(v)}_r\cap\mathcal N^{(v)}_0$, we have
\begin{align*}
E^{(v)}_{e_*,r}\ge c^{(v)}_r
&=t_v+\sum_{e\in J^{(v)}_{0,r}}(t_v-E^{(v)}_{e,r})
     +\sum_{e\in J^{(v)}_r\setminus J^{(v)}_{0,r}}
                                      (t_v-E^{(v)}_{e,r})\\
&\ge t_v+\sum_{e\in J^{(v)}_{0,r}}(t_v-E^{(v)}_{e,r}),
\end{align*}
because $E^{(v)}_{e,r}<t_v$ for every $e\in J^{(v)}_r$.
Each of the $h_v-|J^{(v)}_{0,r}|-1$ directions in
$\mathcal N^{(v)}_0\setminus(J^{(v)}_r\cup\{e_*\})$
satisfies $E^{(v)}_{e,r}\ge t_v$. Therefore
\begin{align*}
\sum_{e\in\mathcal N^{(v)}_0}E^{(v)}_{e,r}
&\ge E^{(v)}_{e_*,r}
 +\sum_{e\in J^{(v)}_{0,r}}E^{(v)}_{e,r}
 +(h_v-|J^{(v)}_{0,r}|-1)t_v\\
&\ge t_v+\sum_{e\in J^{(v)}_{0,r}}(t_v-E^{(v)}_{e,r})
 +\sum_{e\in J^{(v)}_{0,r}}E^{(v)}_{e,r}
 +(h_v-|J^{(v)}_{0,r}|-1)t_v=h_vt_v.
\end{align*}
Dividing by $h_v$ proves
\begin{equation}
\{\mathcal E^{(v)}_r\cap\mathcal N^{(v)}_0\ne\varnothing\}
\subseteq\{A^{(v)}_r\ge1/\alpha_v\}.
\label{eq:epoch-error-inclusion}
\end{equation}

Fix an integer $n\ge1$ through which the epoch assumptions hold, and let
\[
T_{v,n}=\inf\{1\le r\le n:A^{(v)}_r\ge t_v\},
\qquad\inf\varnothing=\infty.
\]
For $1\le s\le n$,
\[
\{T_{v,n}\ge s\}
=\bigcap_{j=1}^{s-1}\{A^{(v)}_j<t_v\}
\in\mathcal F^{(v)}_{s-1},
\qquad
A^{(v)}_{T_{v,n}\wedge n}
=1+\sum_{s=1}^n\mathbbm1\{T_{v,n}\ge s\}
                    (A^{(v)}_s-A^{(v)}_{s-1}).
\]
Taking conditional expectations of this identity gives
\begin{align*}
\E[A^{(v)}_{T_{v,n}\wedge n}\mid\mathcal K_{v-1}]
&=1+\sum_{s=1}^n\E\!\left[
 \mathbbm1\{T_{v,n}\ge s\}
 \E[A^{(v)}_s-A^{(v)}_{s-1}\mid\mathcal F^{(v)}_{s-1}]
 \middle|\mathcal K_{v-1}\right]\le1.
\end{align*}
Nonnegativity and $t_v=1/\alpha_v$ imply
\[
\mathbbm1\{T_{v,n}\le n\}
\le\alpha_v A^{(v)}_{T_{v,n}\wedge n}.
\]
Since $\alpha_v$ is $\mathcal K_{v-1}$-measurable,
\begin{equation}
\begin{aligned}
\Pp\!\left(\max_{1\le r\le n}A^{(v)}_r\ge1/\alpha_v
              \middle|\mathcal K_{v-1}\right)
&=\Pp(T_{v,n}\le n\mid\mathcal K_{v-1})\\
&\le\alpha_v\E[A^{(v)}_{T_{v,n}\wedge n}\mid\mathcal K_{v-1}]
\le\alpha_v.
\end{aligned}
\label{eq:epoch-crossing-detail}
\end{equation}

For a finite run budget $n_v\ge1$ fixed at the start of epoch $v$,
the reporting convention
gives $\mathcal E_r^{(v)}=\varnothing$ for $r>n_v$. Hence
\[
F_v=\bigcup_{r=1}^{n_v}
       \{\mathcal E_r^{(v)}\cap\mathcal N_0^{(v)}\ne\varnothing\}
\subseteq\left\{\max_{1\le r\le n_v}A^{(v)}_r\ge t_v\right\},
\]
by (\ref{eq:epoch-error-inclusion}). Applying
(\ref{eq:epoch-crossing-detail}) with $n=n_v$ bounds
$\Pp(F_v\mid\mathcal K_{v-1})$ by $\alpha_v$.

If the assumptions instead hold for the entire potential sequence
$(Z_r^{(v)})_{r\ge1}$, define
\[
D_{v,n}=\bigcup_{r=1}^n\{A^{(v)}_r\ge t_v\},\qquad
D_{v,n}\uparrow D_{v,\infty}:=\bigcup_{r=1}^\infty
                                     \{A^{(v)}_r\ge t_v\}.
\]
Again (\ref{eq:epoch-error-inclusion}) gives $F_v\subseteq D_{v,\infty}$,
including when evaluation stops after a data-dependent number of runs,
because no reports are issued after that stop. Conditional monotone
convergence and (\ref{eq:epoch-crossing-detail}) yield
\begin{align*}
\Pp(F_v\mid\mathcal K_{v-1})
&\le\E[\mathbbm1\{D_{v,\infty}\}\mid\mathcal K_{v-1}]\\
&=\lim_{n\to\infty}\E[\mathbbm1\{D_{v,n}\}\mid\mathcal K_{v-1}]
\le\alpha_v.
\end{align*}
If $h_v=0$, $F_v=\varnothing$. If $\alpha_v=0$ or epoch $v$
is not entered, $\mathcal E_r^{(v)}=\varnothing$ for all $r$,
so $F_v=\varnothing$ as well. Combining these cases proves
\begin{equation}
\E[\mathbbm1\{F_v\}\mid\mathcal K_{v-1}]
=\Pp(F_v\mid\mathcal K_{v-1})\le\alpha_v
\quad\text{almost surely}.
\label{eq:epoch-conditional-error-detail}
\end{equation}
Now remove the conditioning separately for each epoch. The tower
property, countable subadditivity, and Tonelli's theorem for the
nonnegative budgets give
\begin{align}
\Pp\!\left(\bigcup_{v\ge1}F_v\right)
&\le\sum_{v\ge1}\Pp(F_v)
 =\sum_{v\ge1}\E\!\left[
       \E[\mathbbm1\{F_v\}\mid\mathcal K_{v-1}]\right]
 \nonumber\\
&\le\sum_{v\ge1}\E[\alpha_v]
 =\E\!\left[\sum_{v\ge1}\alpha_v\right]
 \le\alpha.
\label{eq:epoch-global-error-detail}
\end{align}
The final inequality follows from $\sum_{v\ge1}\alpha_v\le\alpha$
almost surely. The calculation uses no factorization of probabilities
across epochs.

By the definition of $\mathcal N_0^{(v)}$,
\[
\bigcup_{v\ge1}F_v
=\left\{\exists v\ge1,\ \exists r\ge1,\quad
 \exists(a\to b)\in\mathcal E_r^{(v)}:
 \Delta_{a,b}^{B^{(v)}}\le\tau\right\}.
\]
Substituting this identity into (\ref{eq:epoch-global-error-detail})
proves Theorem~\ref{thm:update_benchmark}.
\end{proof}

\subsection{Proof of Corollary~\ref{cor:new-models}}
\label{app:proof-new-models}

\begin{proof}
For the initial stage $v=0$, the initial panels satisfy
Assumption~\ref{ass:design-iclr}. Therefore,
Theorem~\ref{thm:bedge-main-iclr} gives $\Pp(F_0\mid\cB,\pi)\le\alpha_0,$ where $F_0=\left\{\exists r\le R^{(0)},\exists e\in\cE_{r,\mathrm{new}}^{(0)}\text{ such that }\Delta_e^{\cB}\le\tau\right\}.$

Fix a stage $v\ge1$. Let $\mathcal Q_v$ collect the model-inclusion rule,
pairing rule, directional family, partition, stake grid, mixture weights,
and all other stage-$v$ design choices. Define the replay filtration
\[
\cF_r^{(v)}
=
\sigma\!\left(
\cB,\pi,\mathcal Q_v,
\widetilde Z_1^{(v)},\ldots,\widetilde Z_r^{(v)}
\right),
\qquad r=0,\ldots,R^{(v)},
\]
with $\cF_0^{(v)}=\sigma(\cB,\pi,\mathcal Q_v).$

By assumption, $\mathcal Q_v$ is independent of the augmented inference panels conditional on $(\cB,\pi)$, and the augmented panels satisfy
Assumption~\ref{ass:design-iclr} with respect to $(\cF_r^{(v)})_{r=0}^{R^{(v)}}$.

Although the stored outcomes are available when stage $v$ begins, the new $e$-processes replay the augmented panels in replicate order. 
Hence every prediction and stake used at replicate $r$ is
$\cF_{r-1}^{(v)}$-measurable and does not depend on stored outcomes from replicate $r$ or later. 
The supermartingale argument in Theorem~\ref{thm:bedge-main-iclr} therefore applies with respect to the replay filtration.

The proof of Theorem~\ref{thm:bedge-main-iclr} applies to any directional family fixed before inference and therefore applies to
$\cH_{\mathrm{new}}^{(v)}$. Defining
\[
F_v=\left\{\exists r\le R^{(v)},\ \exists e\in\cE_{r,\mathrm{new}}^{(v)}\text{ such that }\Delta_e^{\cB}\le\tau\right\},
\]
we obtain $\Pp(F_v\mid\cB,\pi)\le\alpha_v.$

The events $(F_v)_{v\ge0}$ may be arbitrarily dependent because stored
outcomes can be reused across stages. Countable subadditivity nevertheless
gives
\[
\Pp\!\left(\bigcup_{v\ge0}F_v\,\middle|\,\cB,\pi\right)\le\sum_{v\ge0}\Pp(F_v\mid\cB,\pi)\le\sum_{v\ge0}\alpha_v\le\alpha.
\]
On the complementary event, every retained or newly certified edge is a true margin claim. Therefore, all path implications, simultaneous rank intervals, and Top-$k$ certificates derived from the combined graph are valid.
\end{proof}

\subsection{The grid-dependent power statement}
\label{app:proof-power}

This subsection supplies the condition and proof for the power statement
already used in the main text, which concerns the same fixed grid and
predictions, not a different evidence construction. For a tested
alternative $e=(a\to b)$, set $\delta_e=(\Delta_{a,b}^{\cB}-\tau)/2>0$
and define
\begin{align}
q_{e,m,r}
&=\E[(Y_{e,m,r}-\widehat Y_{e,m,r})^2\mid\cF_{r-1}],\nonumber\\
a_{e,g,r}
&=\frac{\lambda_g}{w_\star}\delta_e
 -\sum_m\psi_E(\lambda_{g,m})q_{e,m,r}.
\label{eq:proof-growth-drift}
\end{align}
The quantity $a_{e,g,r}$ is the conditional expected log increment,
as verified in the proof below.

\begin{theorem}
\label{thm:power-consistency-iclr}
Suppose Assumption~\ref{ass:design-iclr} holds for the entire infinite
sequence, with a finite family and a fixed finite grid. Fix a tested
direction $e$ with $\Delta_{a,b}^{\cB}>\tau$. If for some grid
component $g$ with $\rho_g>0$,
\begin{equation}
\liminf_{r\to\infty}\frac1r\sum_{s=1}^r a_{e,g,s}>0
\quad\text{almost surely},
\label{eq:proof-positive-growth-condition}
\end{equation}
then $E_{e,r}^{\cC}\to\infty$ almost surely, and direct e-Holm
rejects $H_e^{(\tau)}$ at every sufficiently large completed replicate
almost surely. In particular, its probability of rejection tends to one.
\end{theorem}

\begin{proof}
Let $\cF_0=\sigma(\cB,\pi,\mathcal Q)$, and fix the grid index $g$
in the theorem, for which $\rho_g>0$. Define
\[
G_g=\left\{\liminf_{r\to\infty}\frac1r
                        \sum_{s=1}^r a_{e,g,s}>0\right\}.
\]
Assumption (\ref{eq:proof-positive-growth-condition}) gives
$\Pp(G_g^c)=0$. By the tower property,
\[
0=\Pp(G_g^c)=\E[\Pp(G_g^c\mid\cF_0)],\qquad
\Pp(G_g^c\mid\cF_0)\ge0,
\]
so $\Pp(G_g\mid\cF_0)=1$ almost surely. We work under the
conditional law given $\cF_0$ for such initial designs. Define
$\ell_{g,r}=\log F_{e,r}^{\cC}(\lambda_g)$. Taking conditional
expectations in (\ref{eq:block-factor}) gives
\begin{align}
\E[\ell_{g,r}\mid\cF_{r-1}]
&=\sum_m\lambda_{g,m}(\nu_{e,m,r}-\mu_0)
  -\sum_m\psi_E(\lambda_{g,m})q_{e,m,r}\nonumber\\
&=\frac{\lambda_g}{w_\star}\delta_e
  -\sum_m\psi_E(\lambda_{g,m})q_{e,m,r}
=a_{e,g,r},
\label{eq:power-log-mean}
\end{align}
where the second equality is (\ref{eq:proof-average-identity}).
Both $\nu_{e,m,r}$ and $\widehat Y_{e,m,r}$ are
$\cF_{r-1}$-measurable. Since
$\E[Y_{e,m,r}-\nu_{e,m,r}\mid\cF_{r-1}]=0$,
\begin{align*}
q_{e,m,r}
&=\E\!\left[
 \{(Y_{e,m,r}-\nu_{e,m,r})
       +(\nu_{e,m,r}-\widehat Y_{e,m,r})\}^2
 \middle|\cF_{r-1}\right]\\
&=\E[(Y_{e,m,r}-\nu_{e,m,r})^2\mid\cF_{r-1}]\\
&\quad+2(\nu_{e,m,r}-\widehat Y_{e,m,r})
      \E[Y_{e,m,r}-\nu_{e,m,r}\mid\cF_{r-1}]
      +(\nu_{e,m,r}-\widehat Y_{e,m,r})^2\\
&=\operatorname{Var}(Y_{e,m,r}\mid\cF_{r-1})
      +(\nu_{e,m,r}-\widehat Y_{e,m,r})^2.
\end{align*}

To pass from conditional expected growth to realized growth, write
\[
\xi_r=\ell_{g,r}-a_{e,g,r},\qquad \mathsf M_r=\sum_{s=1}^r\xi_s,
\qquad
C_g=\sum_m\{\lambda_{g,m}+\psi_E(\lambda_{g,m})\}.
\]
Since $Y_{e,m,r},\widehat Y_{e,m,r},\mu_0\in[0,1]$, we have
\[
|\ell_{g,r}|\le C_g,\qquad
|a_{e,g,r}|=|\E[\ell_{g,r}\mid\cF_{r-1}]|
\le C_g,\qquad |\xi_r|\le2C_g.
\]
Here $C_g<\infty$ is $\cF_0$-measurable because the grid is
finite and $\lambda_{g,m}<1$ for all $m$. By (\ref{eq:power-log-mean}),
\[
\E[\xi_r\mid\cF_{r-1}]
=\E[\ell_{g,r}\mid\cF_{r-1}]-a_{e,g,r}=0.
\]
On $\{C_g=0\}$, $a_{e,g,r}=0$ for every $r$, so
$\{C_g=0\}\subseteq G_g^c$. Thus $C_g>0$ almost surely under
the designs considered here.

The conditional form of Hoeffding's bounded-variable lemma
\citep{hoeffding1963}, applied to the centered variable
$\xi_r\in[-2C_g,2C_g]$, gives, for each real $t$,
\[
\E[\exp(t\xi_r)\mid\cF_{r-1}]\le\exp(2t^2C_g^2).
\]
Since $\mathsf M_{r-1}$ is $\cF_{r-1}$-measurable, the tower property gives
\begin{align*}
\E[\exp(t\mathsf M_r)\mid\cF_0]
&=\E\!\left[\exp(t\mathsf M_{r-1})
             \E[\exp(t\xi_r)\mid\cF_{r-1}]\middle|\cF_0\right]\\
&\le\exp(2t^2C_g^2)\E[\exp(t\mathsf M_{r-1})\mid\cF_0]
\le\cdots\le\exp(2rt^2C_g^2),
\end{align*}
where $M_0=0$. For $\varepsilon>0$, exponential Markov's inequality
and minimization over $t>0$ give
\[
\Pp(\mathsf M_r\ge r\varepsilon\mid\cF_0)
\le\inf_{t>0}\exp\{-tr\varepsilon+2rt^2C_g^2\}
=\exp\!\left\{-\frac{r\varepsilon^2}{8C_g^2}\right\}.
\]
The minimizing value is $t=\varepsilon/(4C_g^2)$, with $C_g$
fixed under the conditional law. Applying the same bound to $-\mathsf M_r$
and adding the two tail probabilities yields
\begin{equation}
\Pp(|\mathsf M_r|\ge r\varepsilon\mid\cF_0)
\le2\exp\!\left\{-\frac{r\varepsilon^2}{8C_g^2}\right\}.
\label{eq:power-martingale-tail}
\end{equation}
For every integer $j\ge1$, (\ref{eq:power-martingale-tail}) implies
\begin{align*}
\Pp\!\left(\bigcup_{r=n}^\infty
             \{|\mathsf M_r|\ge r/j\}\middle|\cF_0\right)
&\le2\sum_{r=n}^\infty
           \exp\!\left\{-\frac{r}{8j^2C_g^2}\right\}\\
&=\frac{2\exp\{-n/(8j^2C_g^2)\}}
        {1-\exp\{-1/(8j^2C_g^2)\}}\longrightarrow0.
\end{align*}
The events on the left decrease to
$\bigcap_{n\ge1}\bigcup_{r\ge n}\{|\mathsf M_r|\ge r/j\}$,
which therefore has conditional probability zero. Taking the
countable union over $j$ gives
\[
\Pp\!\left(\limsup_{r\to\infty}\frac{|\mathsf M_r|}{r}>0
                     \middle|\cF_0\right)=0,
\qquad
\frac{\mathsf M_r}{r}\longrightarrow0\quad\text{a.s.}
\]

The product definition of the component process now gives
\[
\frac1r\log E_{e,r}^{\cC}(\lambda_g)
=\frac1r\sum_{s=1}^r a_{e,g,s}+\frac{\mathsf M_r}{r}.
\]
On $G_g\cap\{\mathsf M_r/r\to0\}$, whose conditional probability is one,
\[
d_g:=\liminf_{r\to\infty}\frac1r
                   \log E_{e,r}^{\cC}(\lambda_g)
=\liminf_{r\to\infty}\frac1r\sum_{s=1}^r a_{e,g,s}>0.
\]
The coefficient $\rho_g>0$ is fixed given $\cF_0$, so
\[
\frac1r\log E_{e,r}^{\cC}
\ge\frac{\log\rho_g}{r}
  +\frac1r\log E_{e,r}^{\cC}(\lambda_g).
\]
Since $(\log\rho_g)/r\to0$, the last inequality yields
\[
\liminf_{r\to\infty}\frac1r\log E_{e,r}^{\cC}\ge d_g>0.
\]
For each such outcome, set $\gamma=d_g/2>0$. There exists an
integer $r_0$ such that
$E_{e,r}^{\cC}\ge\exp(\gamma r)$ for every $r\ge r_0$.
Thus $\Pp(E_{e,r}^{\cC}\to\infty\mid\cF_0)=1$.

Let $h=|\cH|\ge1$. Whenever $E_{e,r}^{\cC}\ge h/\alpha$,
the direction $e$ is not in $\cJ_r$. Thus $|\cJ_r|\le h-1$ and
\[
c_r=\frac1\alpha+
 \sum_{j\in\cJ_r}\left(\frac1\alpha-E_{j,r}^{\cC}\right)
\le\frac1\alpha+\frac{h-1}{\alpha}
=\frac h\alpha\le E_{e,r}^{\cC}.
\]
Consequently, for every $r$,
\[
\{E_{e,r}^{\cC}\ge h/\alpha\}\subseteq\{e\in\cE_r\}.
\]
Because $h/\alpha<\infty$ is fixed given $\cF_0$, divergence gives
\[
\Pp\!\left(\bigcup_{m=1}^\infty\bigcap_{r=m}^\infty
                   \{e\in\cE_r\}\middle|\cF_0\right)=1.
\]
In particular, $\mathbbm1\{e\in\cE_r\}\to1$ almost surely.
Since these indicators lie in $[0,1]$, conditional bounded convergence gives
\[
\Pp(e\in\cE_r\mid\cF_0)
=\E[\mathbbm1\{e\in\cE_r\}\mid\cF_0]\longrightarrow1
\quad\text{a.s.}
\]
Applying the tower property and bounded convergence once more yields
\begin{align*}
\Pp(e\in\cE_r\mid\cB,\pi)
&=\E[\Pp(e\in\cE_r\mid\cF_0)\mid\cB,\pi]\longrightarrow1
\quad\text{a.s.},\\
\Pp(e\in\cE_r)
&=\E[\Pp(e\in\cE_r\mid\cF_0)]\longrightarrow1.
\end{align*}
\end{proof}

\section{Ablation Study}

\begin{figure}[htbp]
\centering
\includegraphics[width=0.99\textwidth]{figure/Figure0.pdf}\\[-2ex]
\subfloat[$|C_m|=1$, $h_{i,\ell}=0$]{%
    \includegraphics[width=0.48\textwidth]{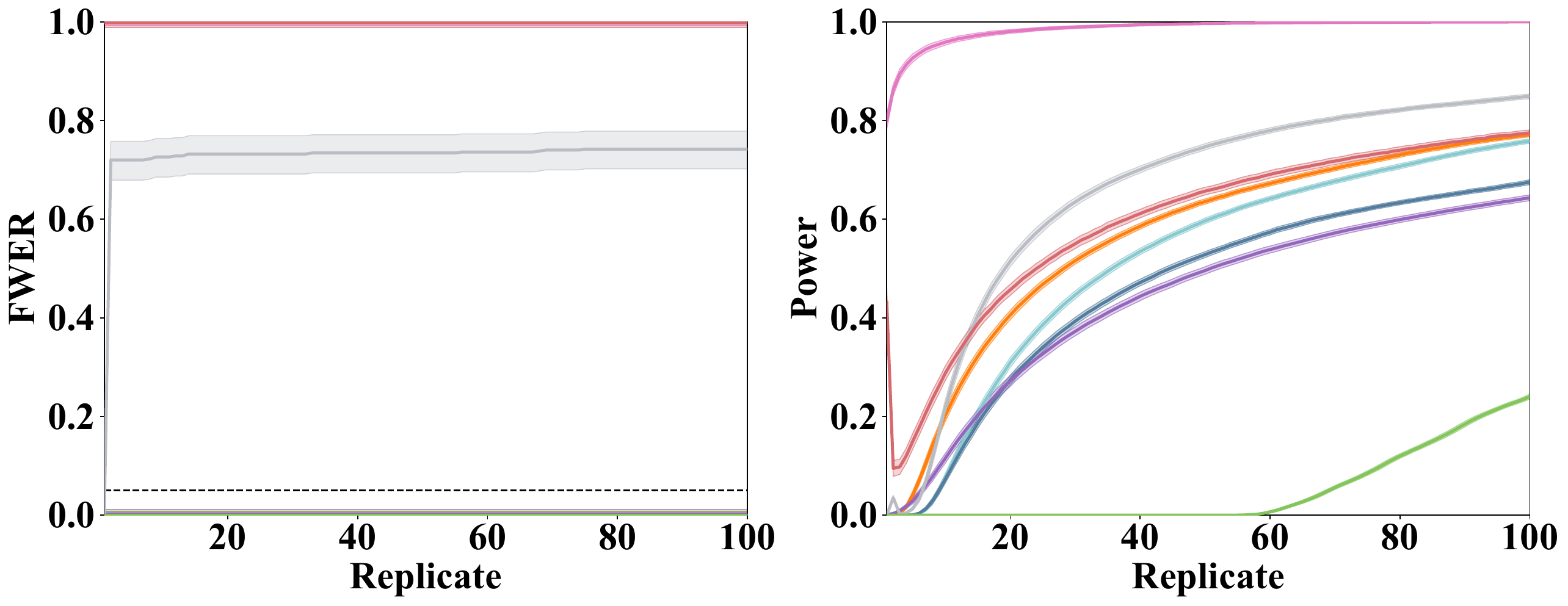}
}
\hfill
\subfloat[$|C_m|=5$, $h_{i,\ell}=0$]{%
    \includegraphics[width=0.48\textwidth]{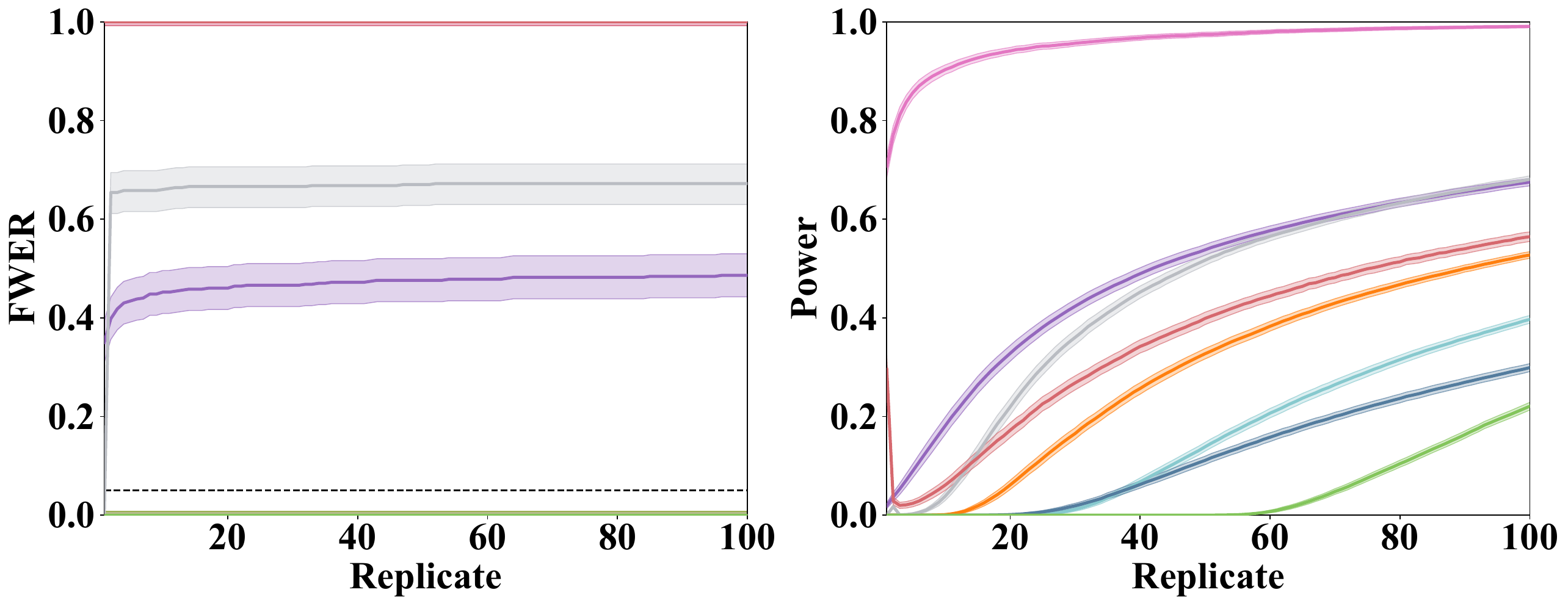}
}
\hfill
\subfloat[$|C_m|=1$, $h_{i,\ell}=0.2$]{%
    \includegraphics[width=0.48\textwidth]{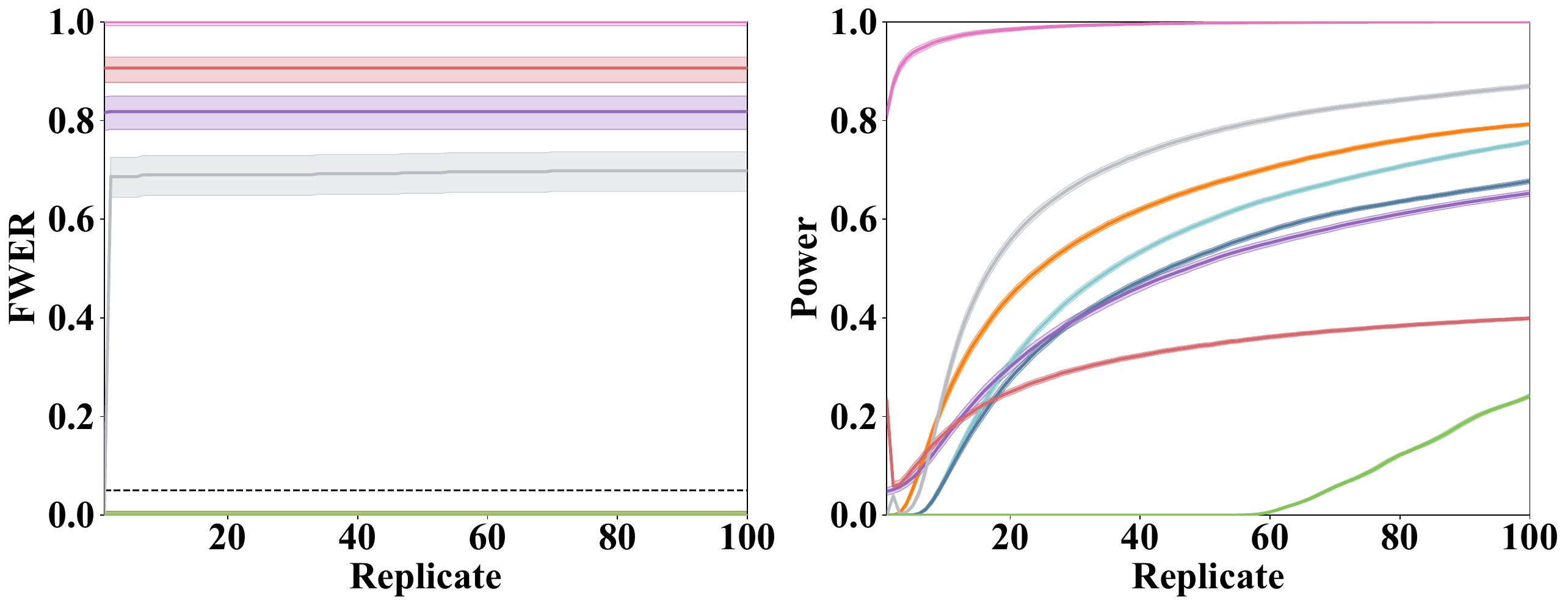}
}
\hfill
\subfloat[$|C_m|=5$, $h_{i,\ell}=0.2$]{%
    \includegraphics[width=0.48\textwidth]{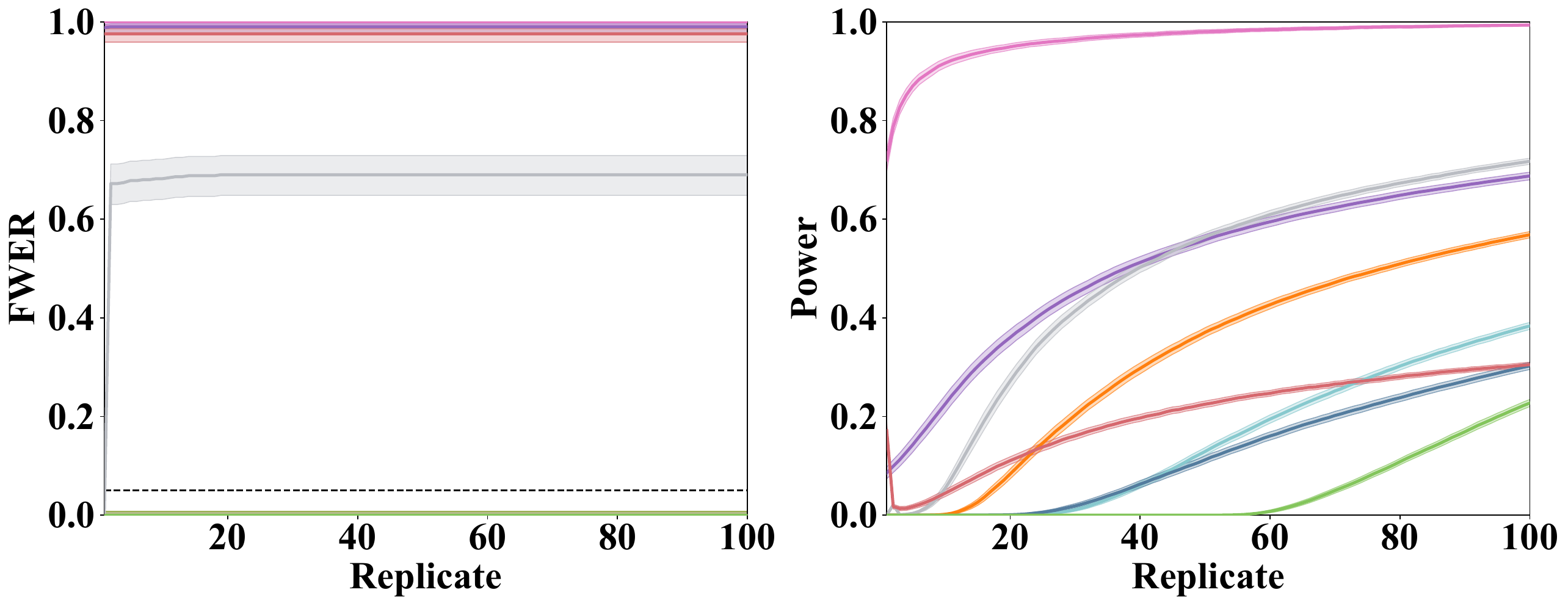}
}
\caption{Ablation study on the number of models ($L=20$).}
\label{figure:a5}
\end{figure}

\begin{figure}[htbp]
\centering
\includegraphics[width=0.99\textwidth]{figure/Figure0.pdf}\\[-2ex]
\subfloat[$|C_m|=1$, $h_{i,\ell}=0$]{%
    \includegraphics[width=0.48\textwidth]{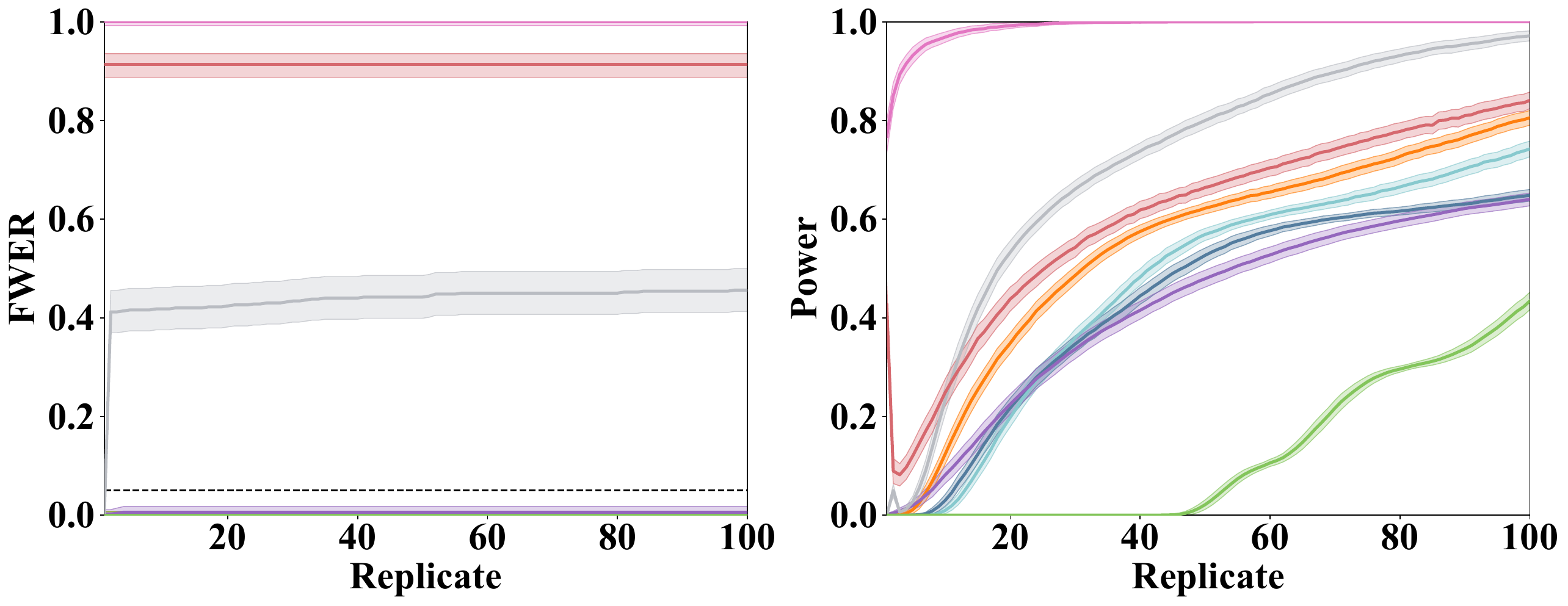}
}
\hfill
\subfloat[$|C_m|=5$, $h_{i,\ell}=0$]{%
    \includegraphics[width=0.48\textwidth]{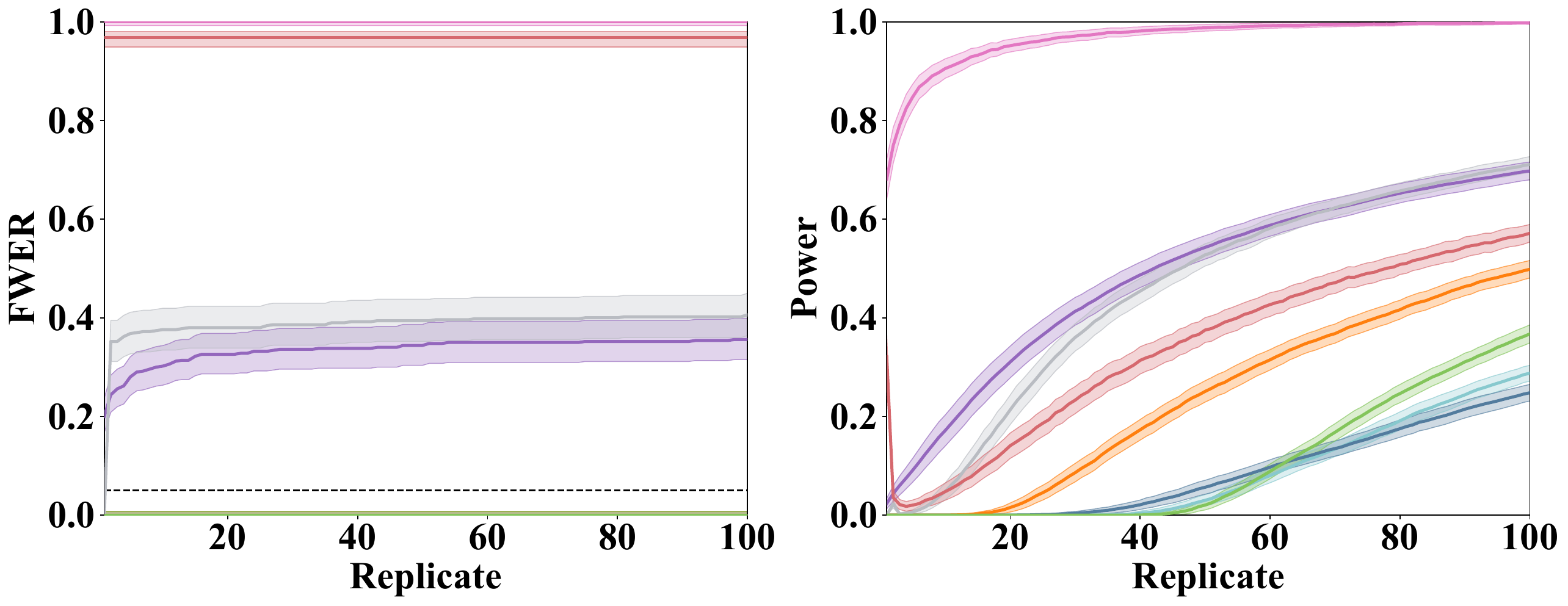}
}
\hfill
\subfloat[$|C_m|=1$, $h_{i,\ell}=0.2$]{%
    \includegraphics[width=0.48\textwidth]{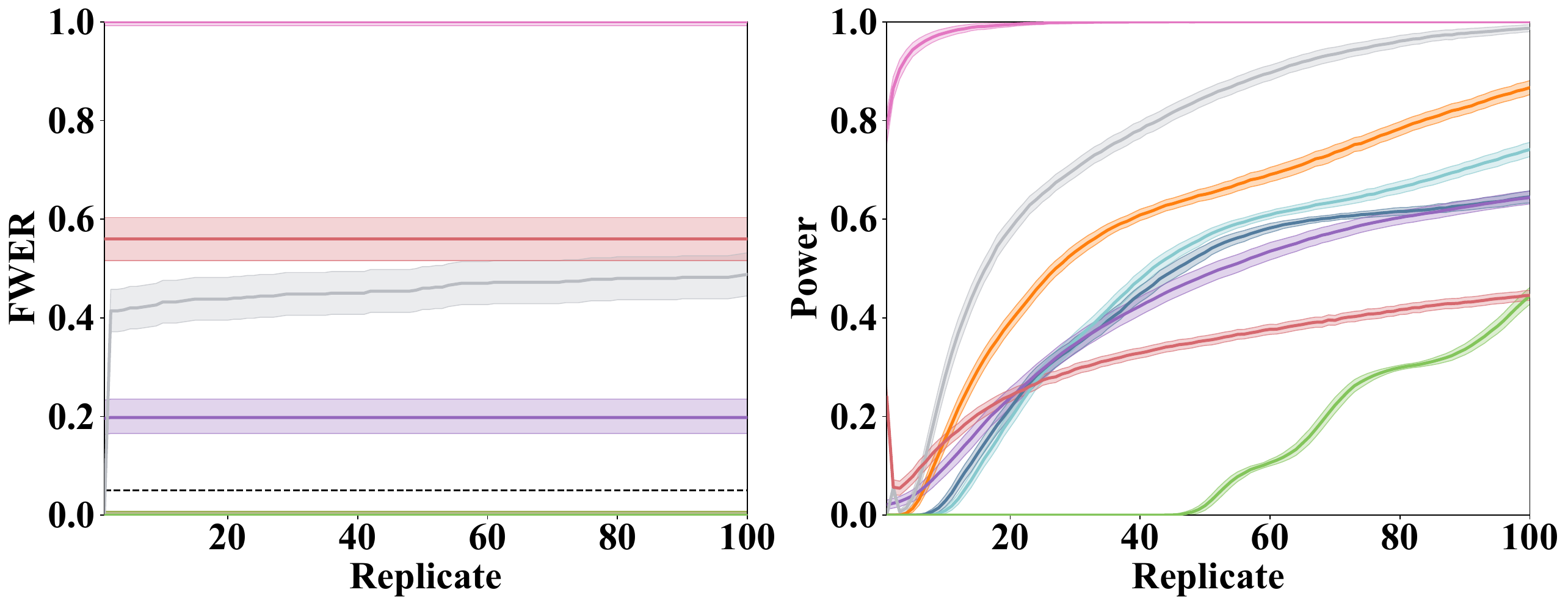}
}
\hfill
\subfloat[$|C_m|=5$, $h_{i,\ell}=0.2$]{%
    \includegraphics[width=0.48\textwidth]{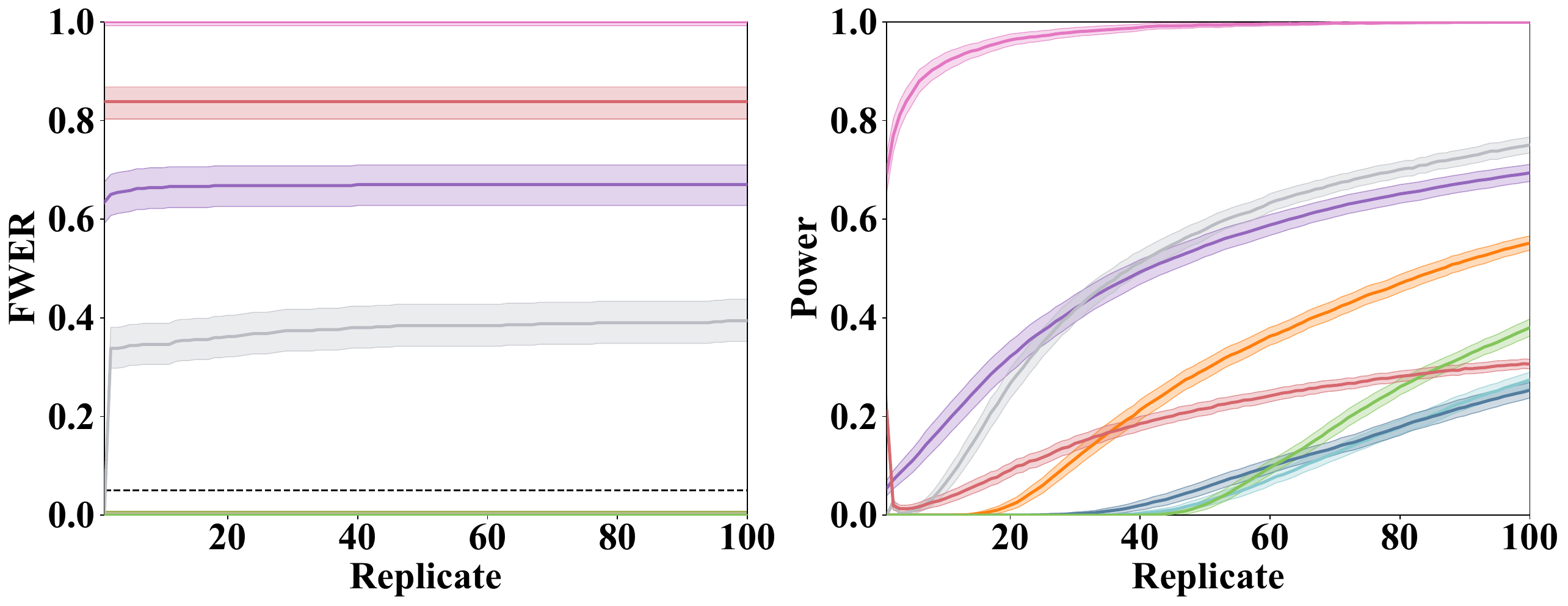}
}
\caption{Ablation study on the effect of item size ($N=50$).}
\label{figure:a3}
\end{figure}

\begin{figure}[htbp]
\centering
\includegraphics[width=0.99\textwidth]{figure/Figure0.pdf}\\[-2ex]
\subfloat[$|C_m|=1$, $h_{i,\ell}=0$]{%
    \includegraphics[width=0.48\textwidth]{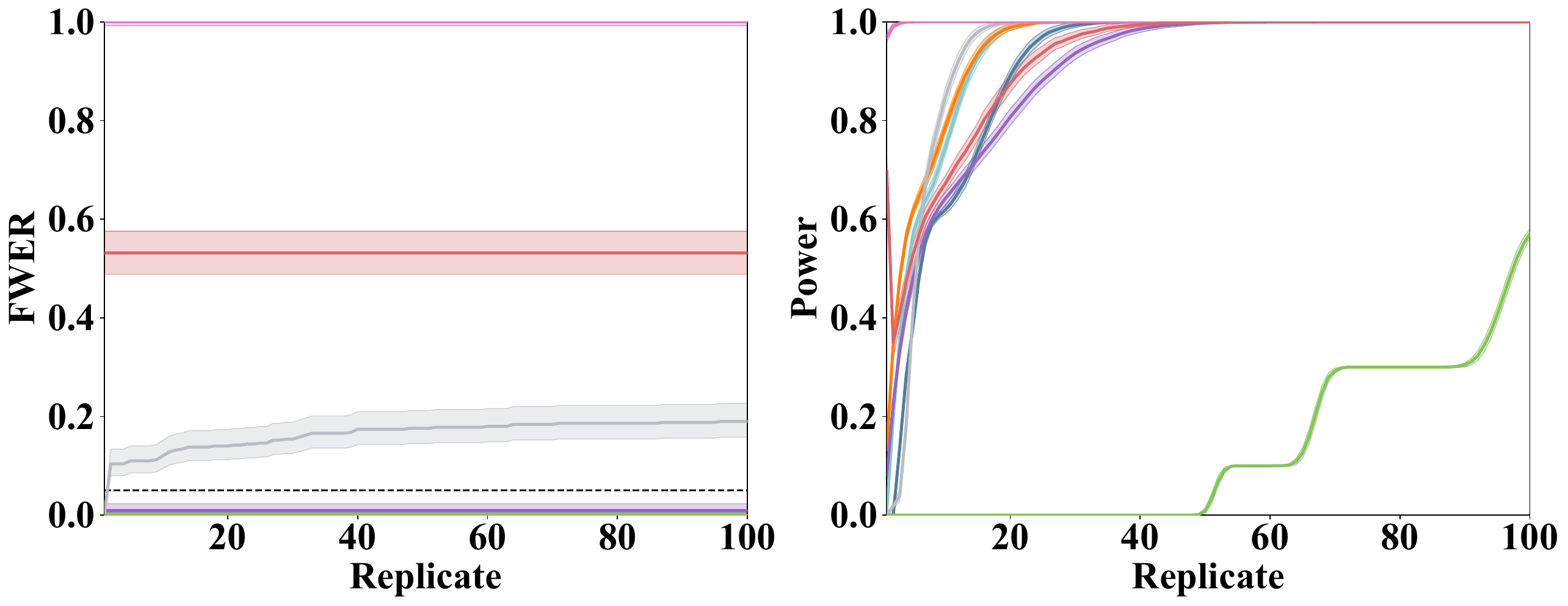}
}
\hfill
\subfloat[$|C_m|=5$, $h_{i,\ell}=0$]{%
    \includegraphics[width=0.48\textwidth]{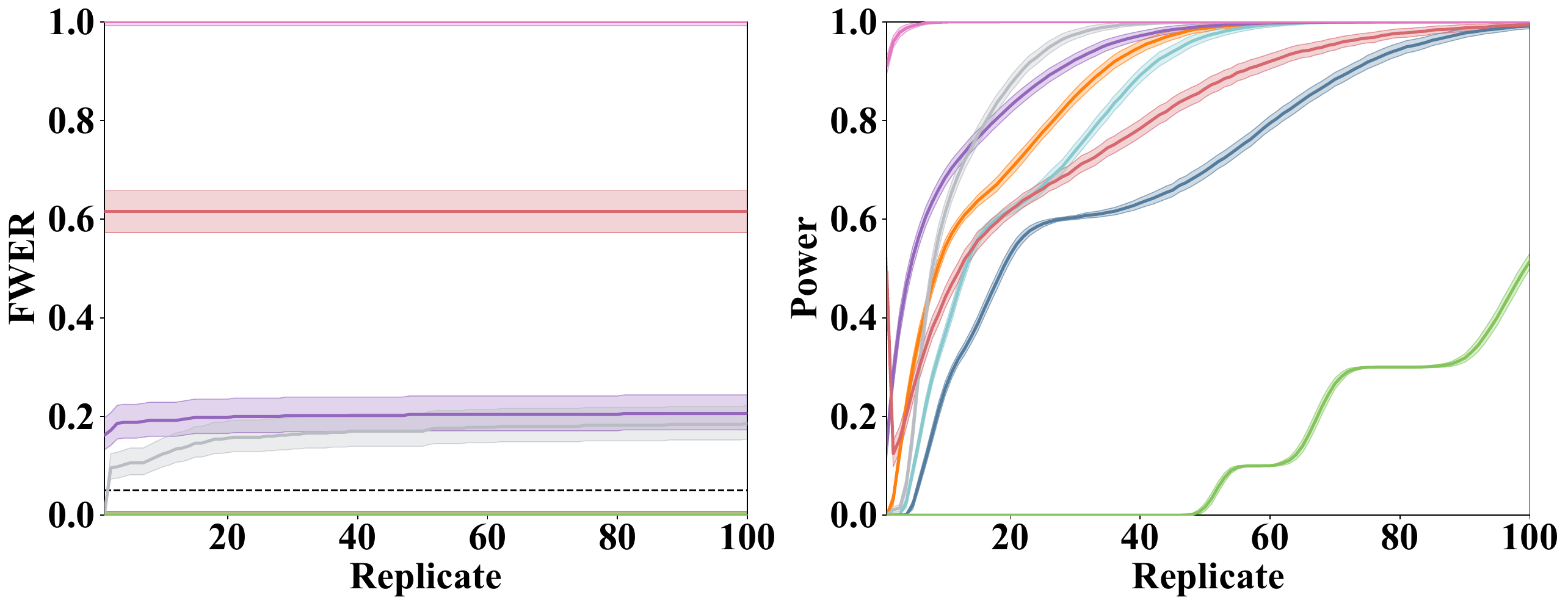}
}
\hfill
\subfloat[$|C_m|=1$, $h_{i,\ell}=0.2$]{%
    \includegraphics[width=0.48\textwidth]{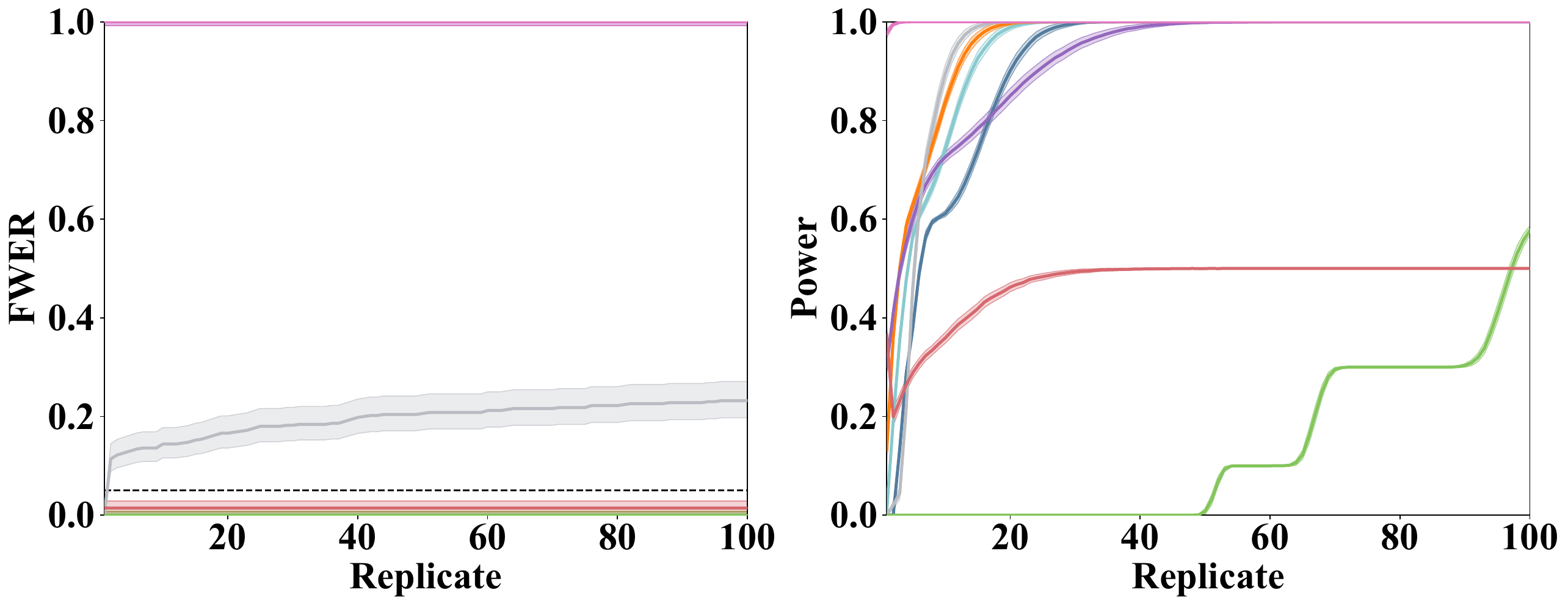}
}
\hfill
\subfloat[$|C_m|=5$, $h_{i,\ell}=0.2$]{%
    \includegraphics[width=0.48\textwidth]{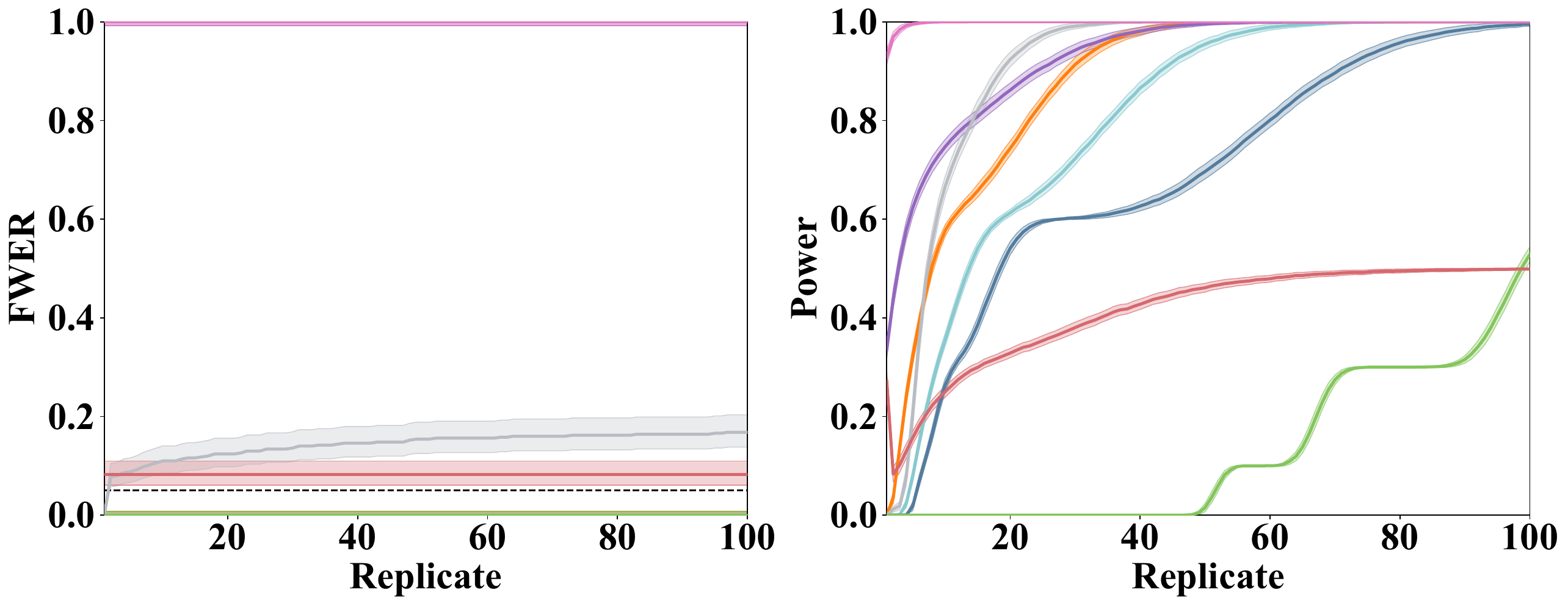}
}
\caption{Ablation study on the effect of item size ($N=500$).}
\label{figure:a4}
\end{figure}


\begin{figure}[htbp]
\centering
\includegraphics[width=0.99\textwidth]{figure/Figure0.pdf}\\[-2ex]
\subfloat[$|C_m|=1$, $h_{i,\ell}=0$]{%
    \includegraphics[width=0.48\textwidth]{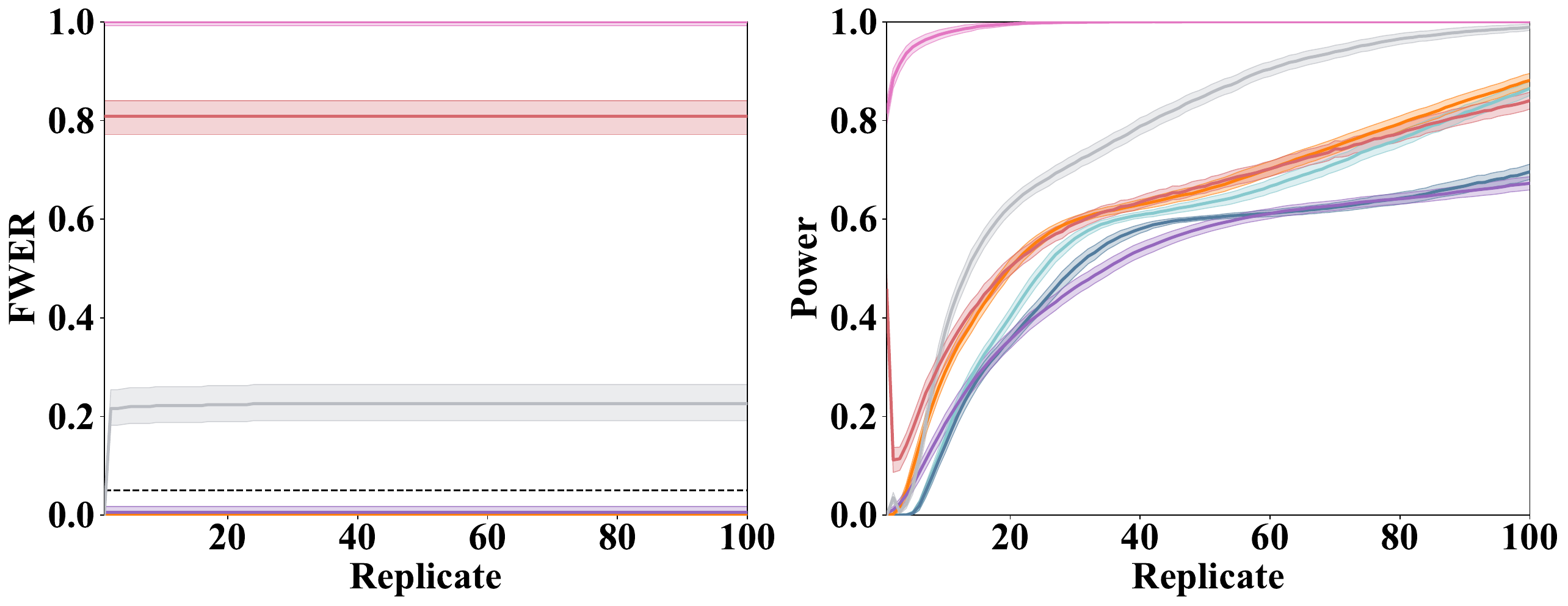}
}
\hfill
\subfloat[$|C_m|=5$, $h_{i,\ell}=0$]{%
    \includegraphics[width=0.48\textwidth]{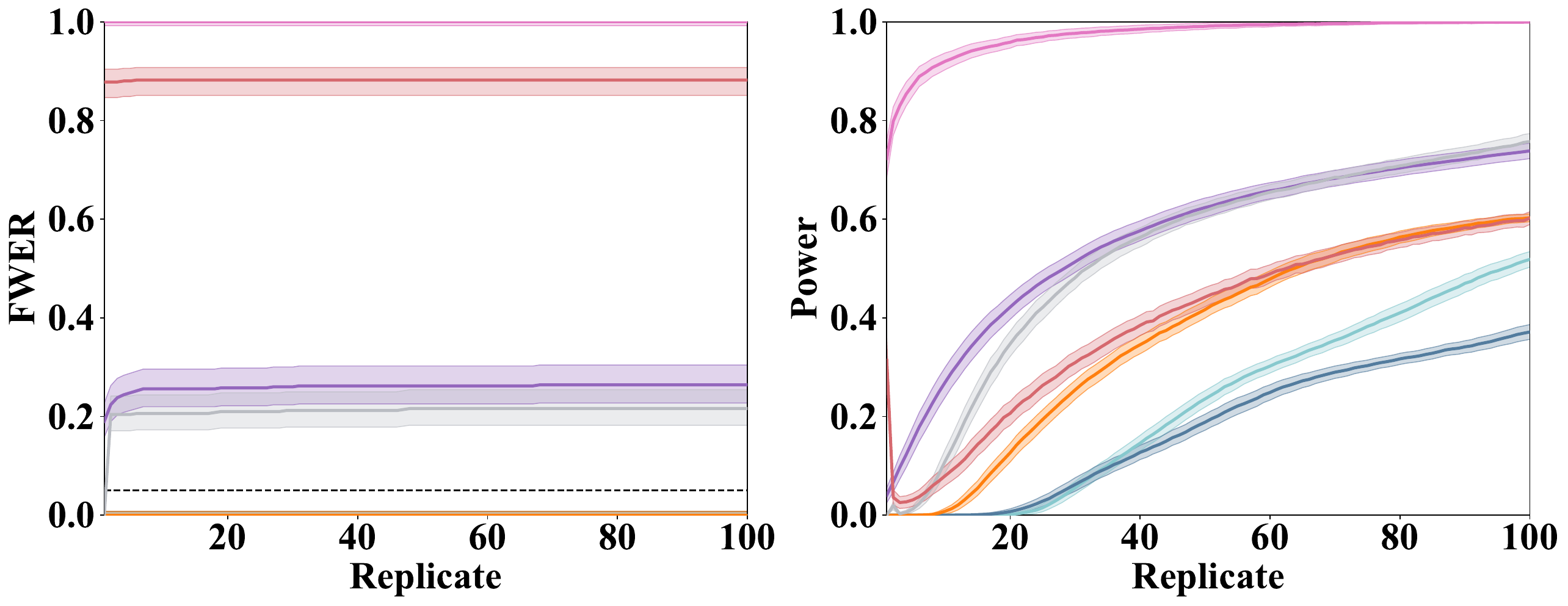}
}
\hfill
\subfloat[$|C_m|=1$, $h_{i,\ell}=0.2$]{%
    \includegraphics[width=0.48\textwidth]{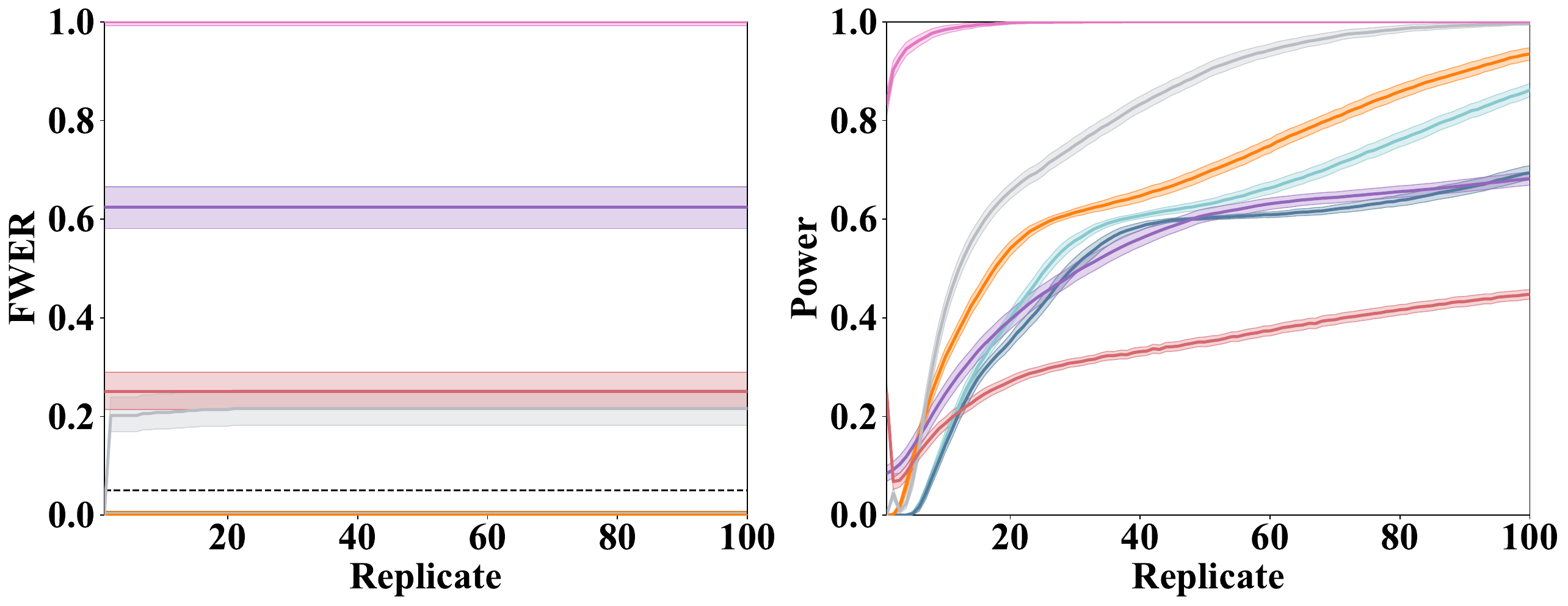}
}
\hfill
\subfloat[$|C_m|=5$, $h_{i,\ell}=0.2$]{%
    \includegraphics[width=0.48\textwidth]{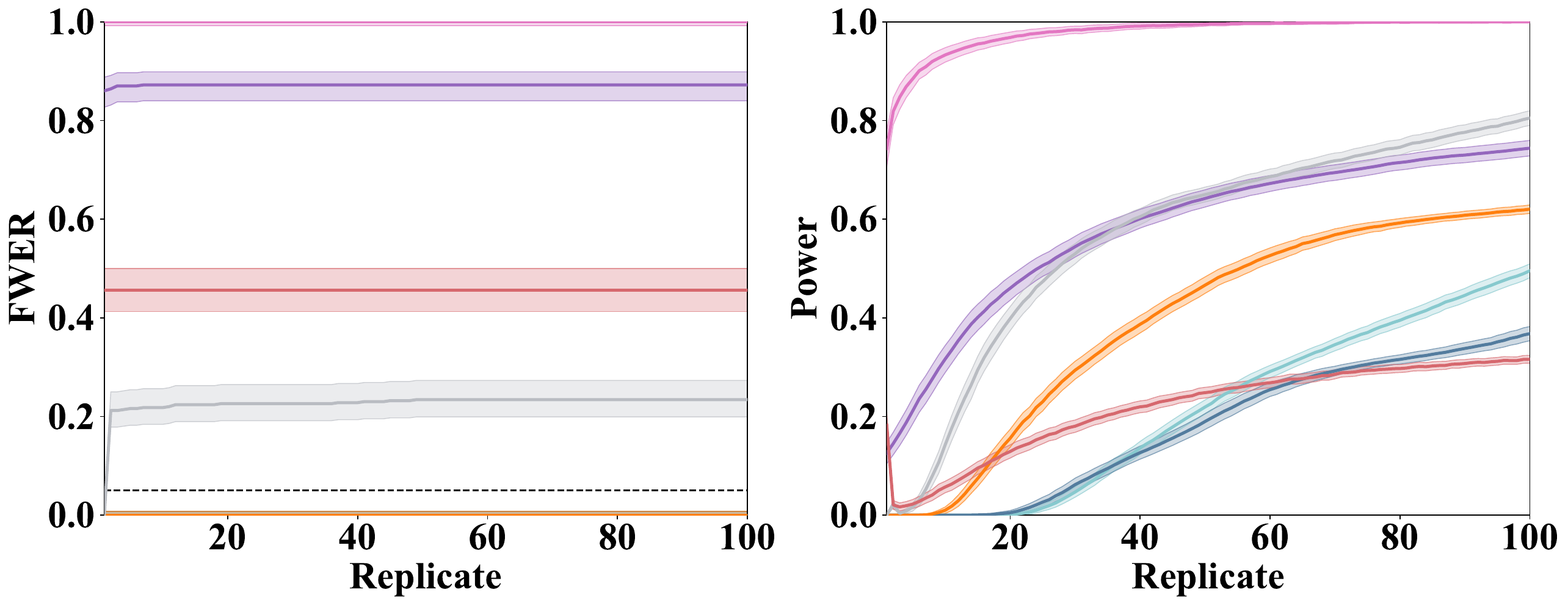}
}
\caption{Ablation study on prespecified margins ($\tau=0.01$).}
\label{figure:tau1}
\end{figure}

\begin{figure}[htbp]
\centering
\includegraphics[width=0.99\textwidth]{figure/Figure0.pdf}\\[-2ex]
\subfloat[$|C_m|=1$, $h_{i,\ell}=0$]{%
    \includegraphics[width=0.48\textwidth]{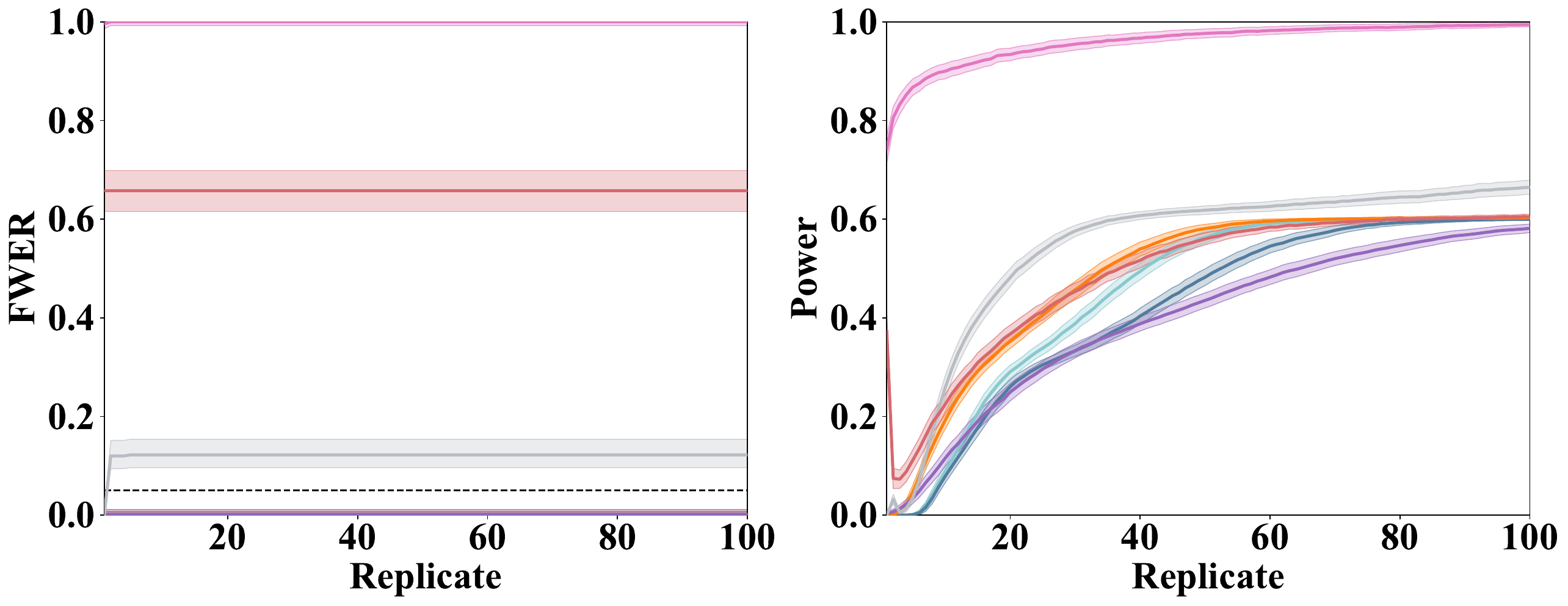}
}
\hfill
\subfloat[$|C_m|=5$, $h_{i,\ell}=0$]{%
    \includegraphics[width=0.48\textwidth]{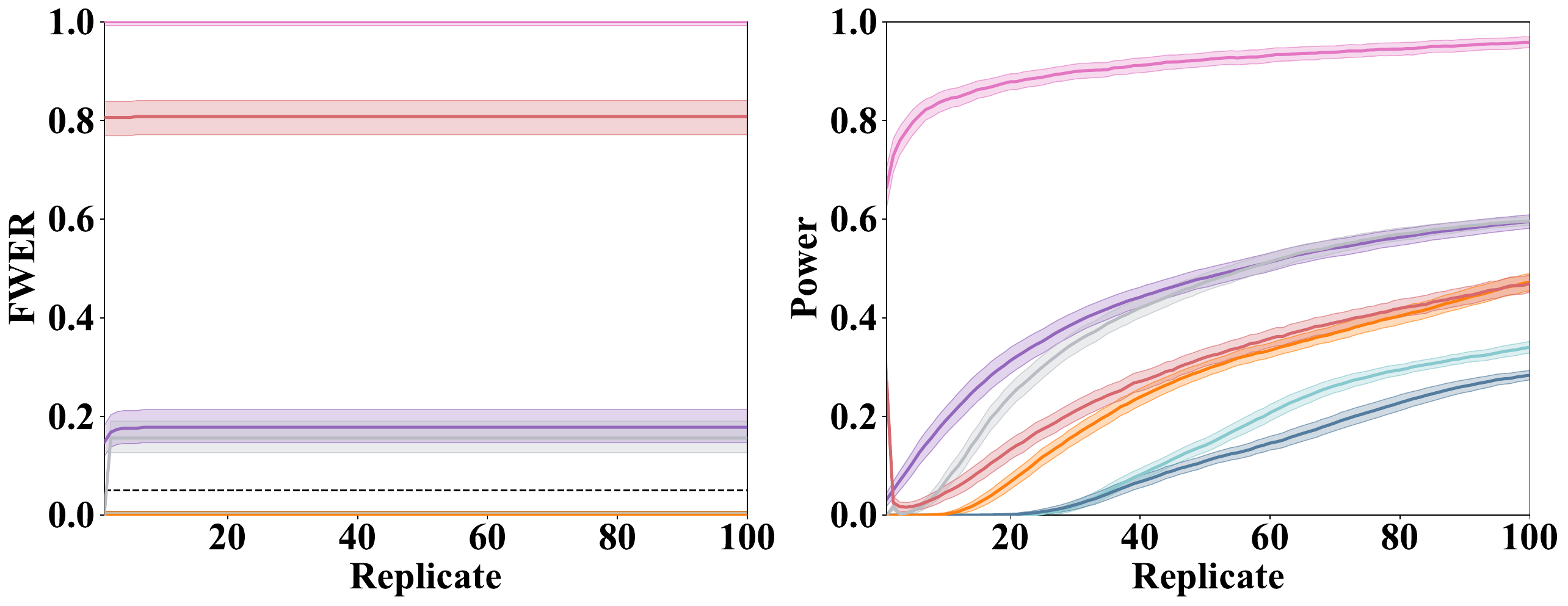}
}
\hfill
\subfloat[$|C_m|=1$, $h_{i,\ell}=0.2$]{%
    \includegraphics[width=0.48\textwidth]{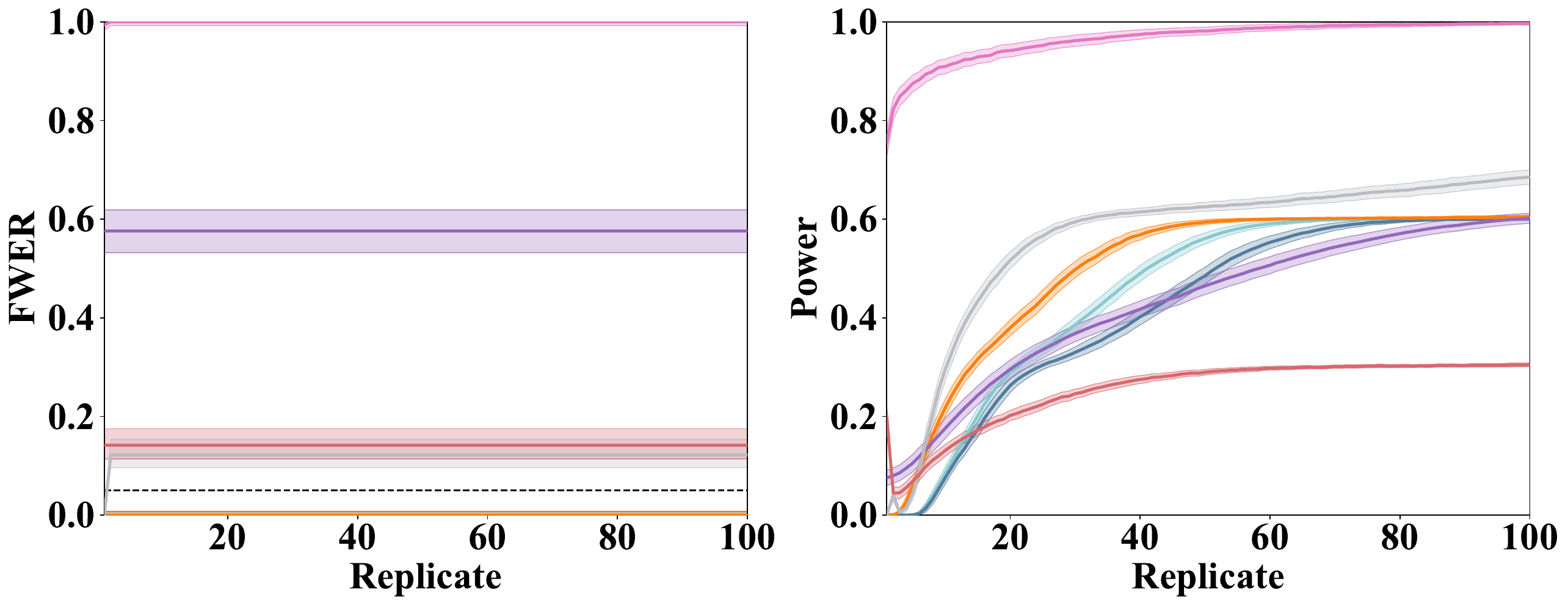}
}
\hfill
\subfloat[$|C_m|=5$, $h_{i,\ell}=0.2$]{%
    \includegraphics[width=0.48\textwidth]{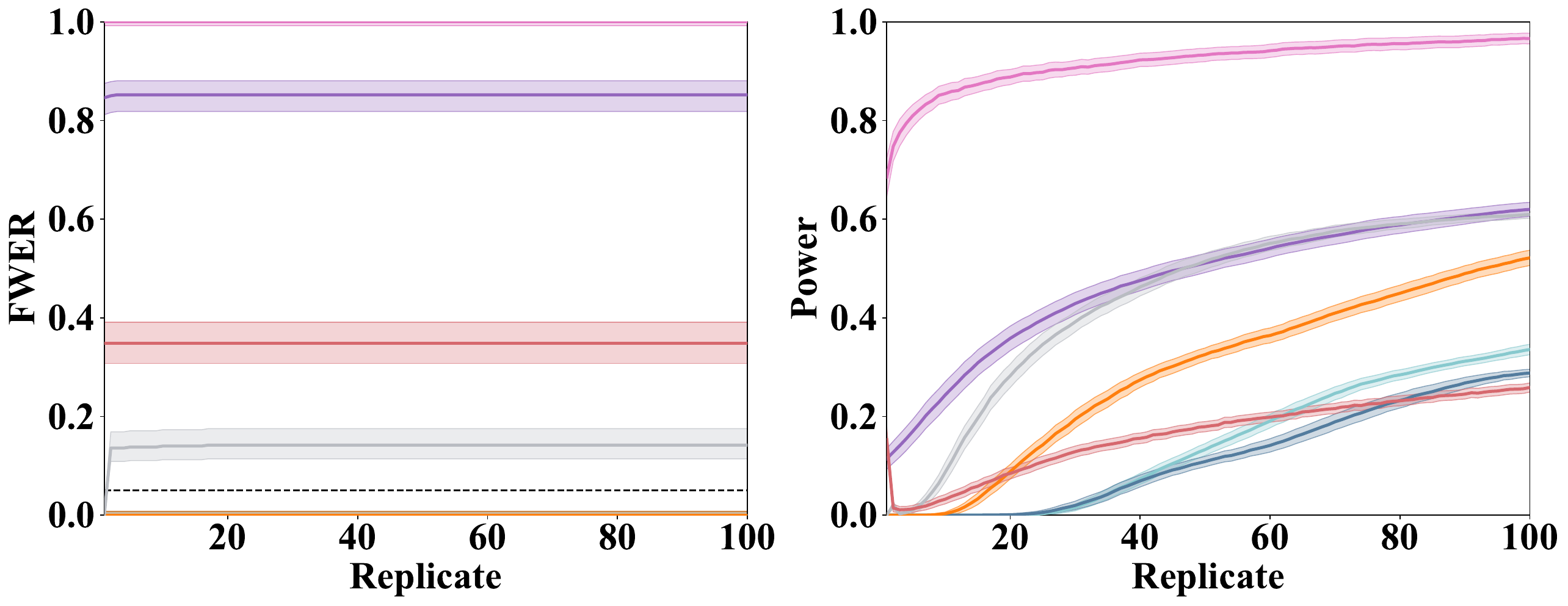}
}
\caption{Ablation study on prespecified margins ($\tau=0.03$).}
\label{figure:tau3}
\end{figure}


\begin{figure}[htbp]
\centering
\subfloat[$|C_m|=1$, $h_{i,\ell}=0$]{%
    \includegraphics[width=0.48\textwidth]{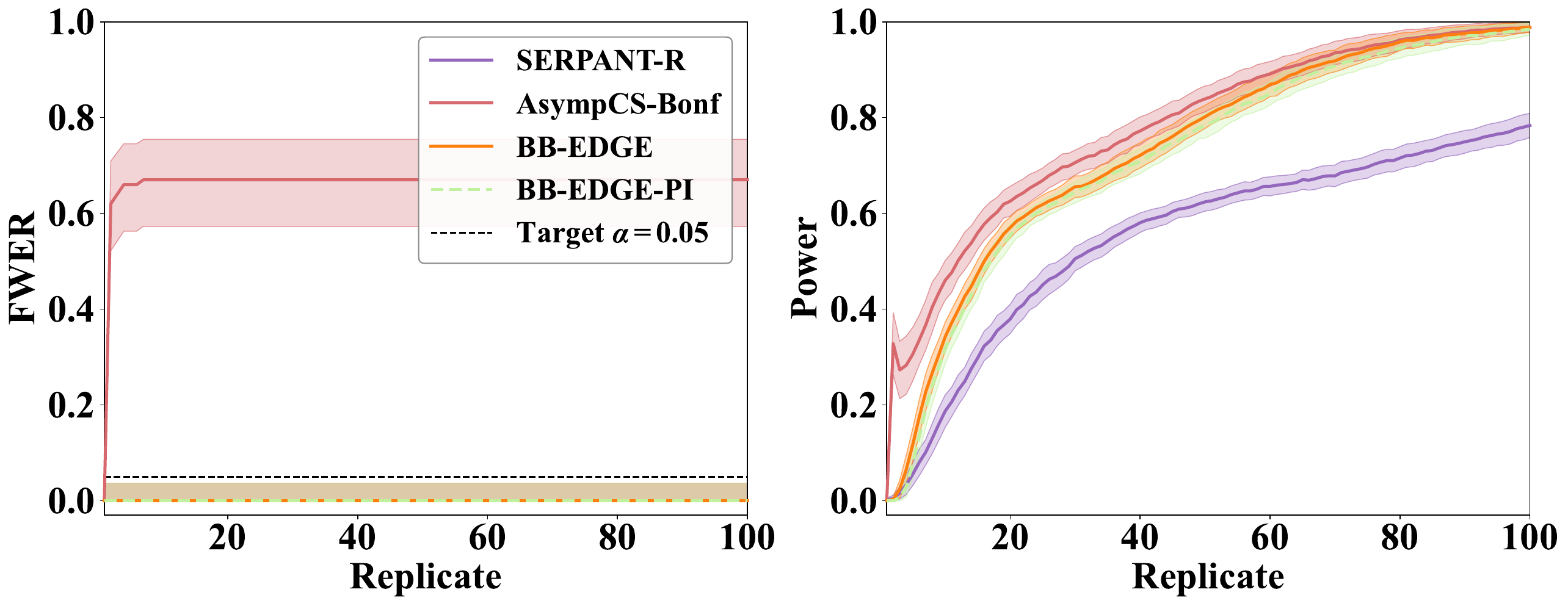}
}
\hfill
\subfloat[$|C_m|=5$, $h_{i,\ell}=0$]{%
    \includegraphics[width=0.48\textwidth]{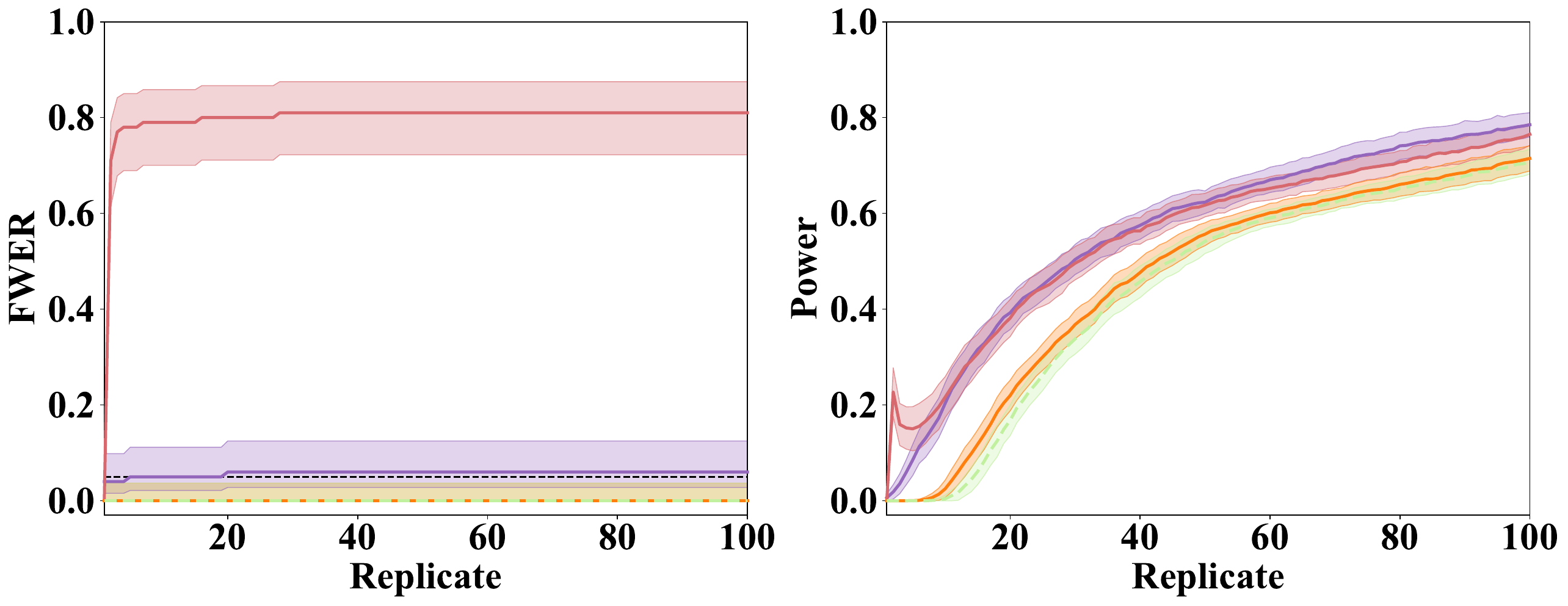}}
\hfill
\subfloat[$|C_m|=1$, $h_{i,\ell}=0.2$]{%
    \includegraphics[width=0.48\textwidth]{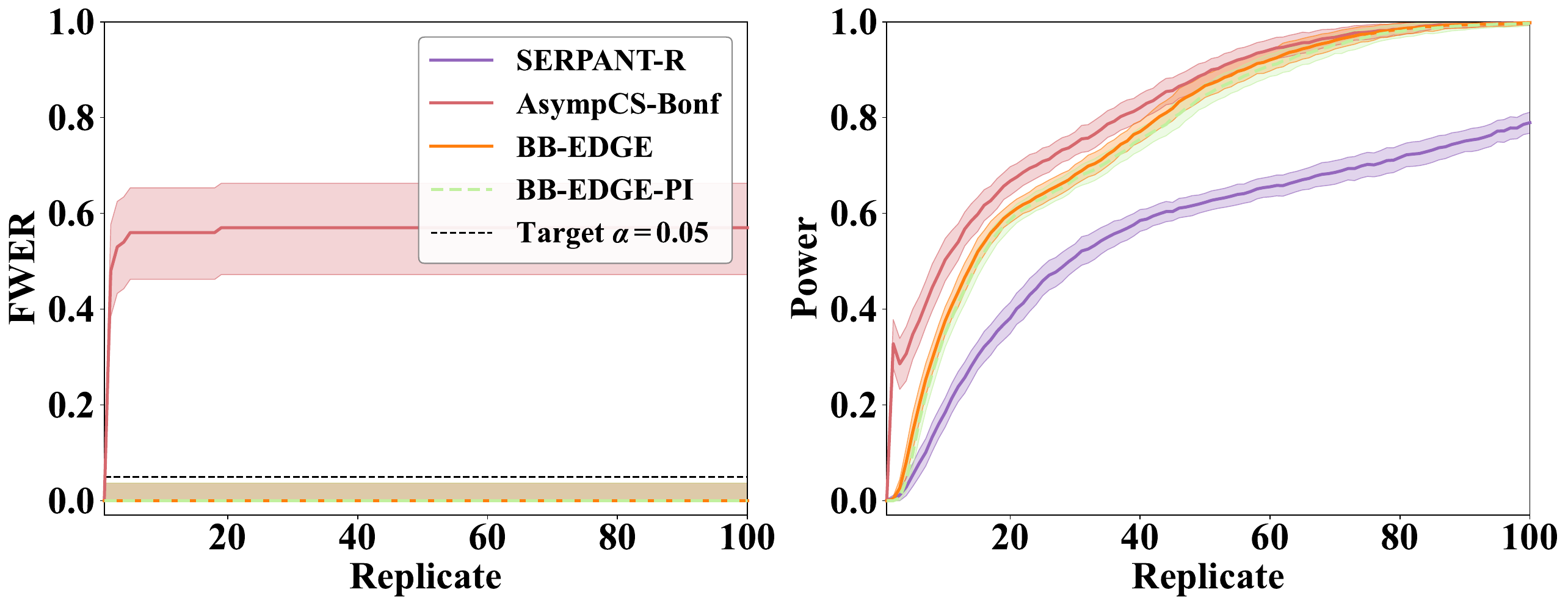}
}
\hfill
\subfloat[$|C_m|=5$, $h_{i,\ell}=0.2$]{%
    \includegraphics[width=0.48\textwidth]{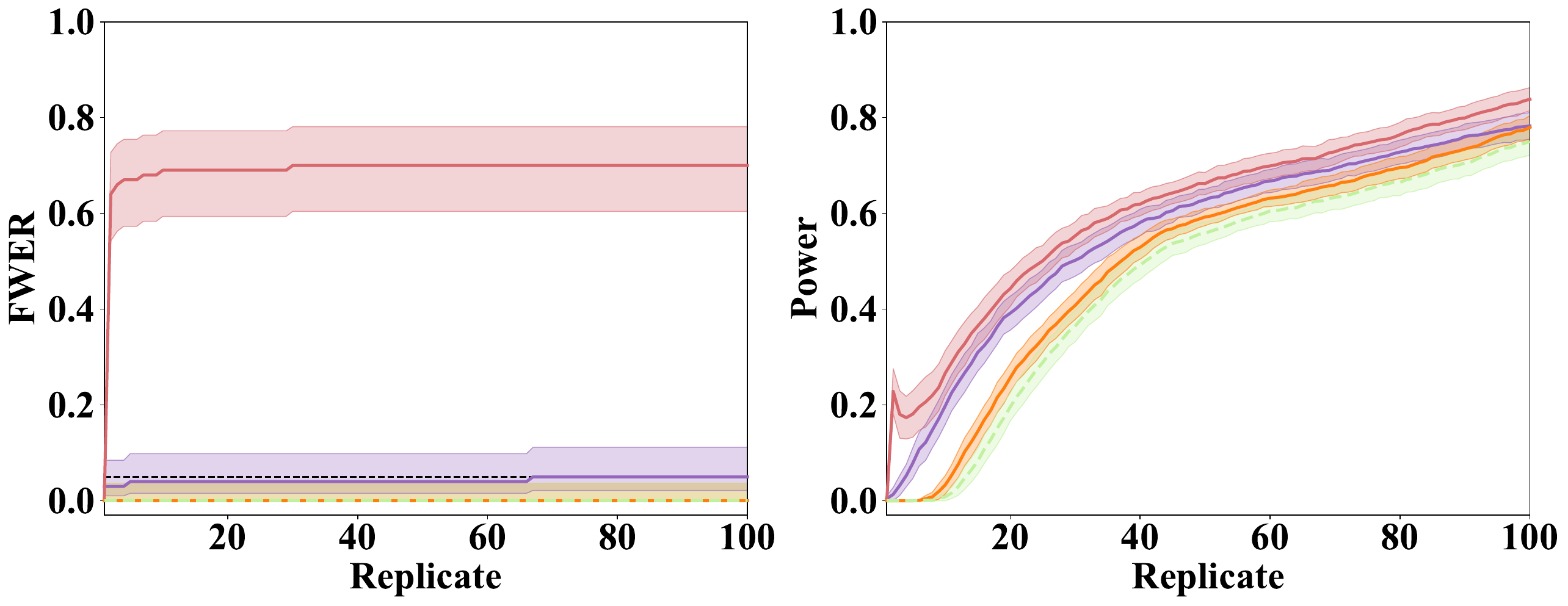}
}
\caption{Ablation study on more baselines: BB-EDGE versus BB-EDGE-PI, AsympCS-Bonf and SERPANT-R.}
\label{figure:grid}
\end{figure}

\begin{figure}[htbp]
\centering
\includegraphics[width=0.99\textwidth]{figure/Figure0.pdf}\\[-2ex]
\subfloat[Heterogeneous Block Correlations]{%
    \includegraphics[width=0.48\textwidth]{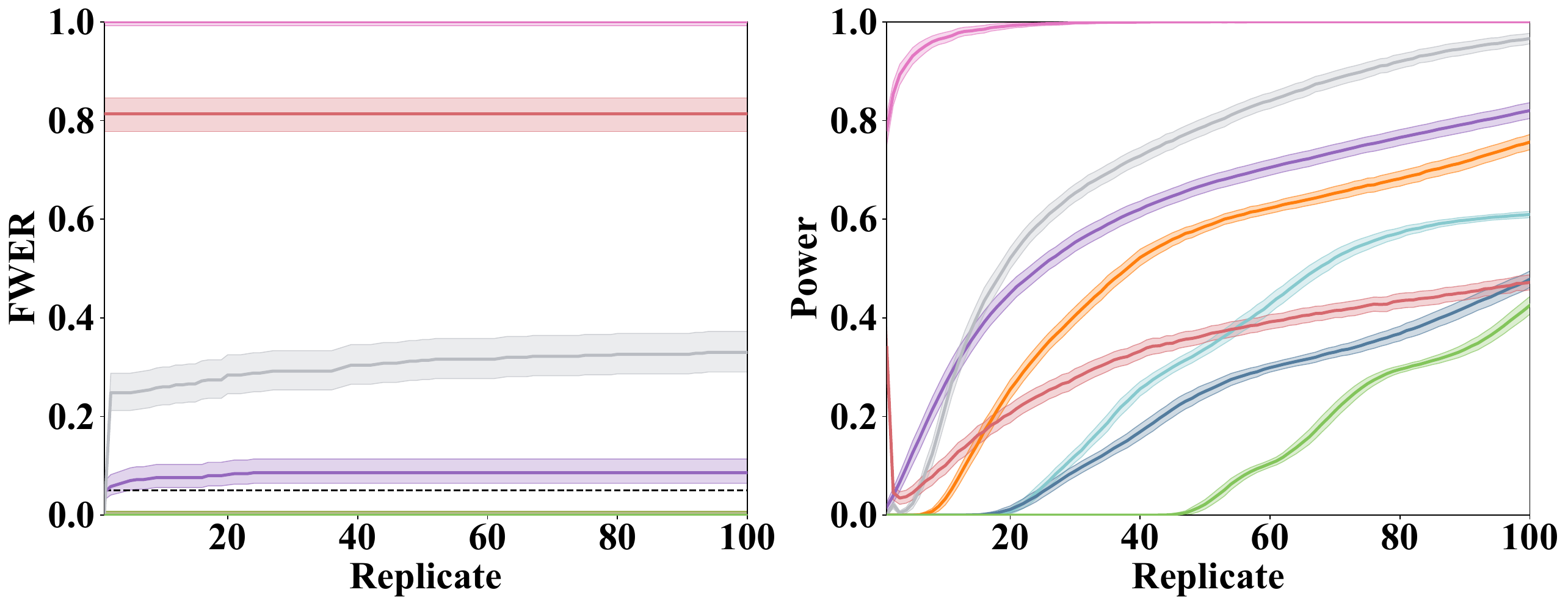}
}
\hfill
\subfloat[Heterogeneous Block Sizes]{%
    \includegraphics[width=0.48\textwidth]{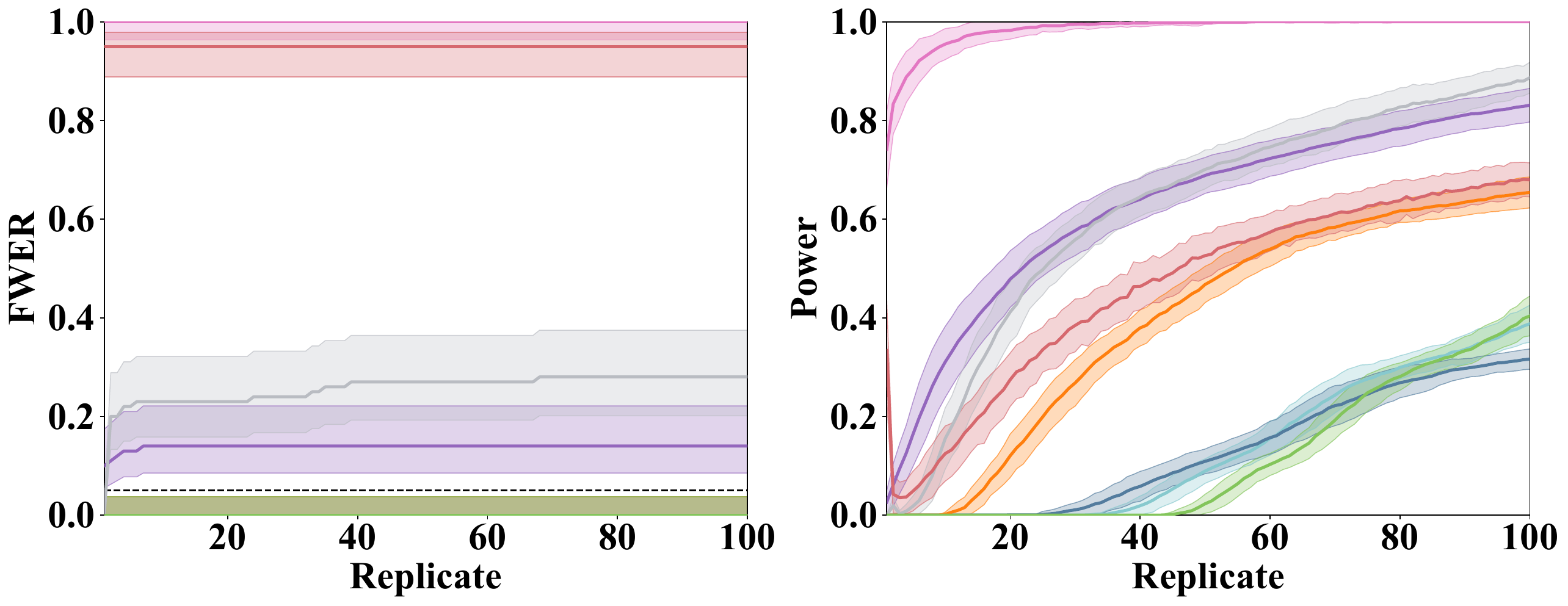}
}
\caption{Ablation study under heterogeneous block structures.
(a) The within-block correlation varies uniformly from $0.1$ to $0.7$ across blocks.
(b) The prespecified block sizes vary from $1$ to $10$.}
\label{figure:block_heterogeneity}
\end{figure}

\clearpage
\section{Additional Real Data Results}
\label{section_read_data_results}
\subsection{Leaderboard}
Tables~\ref{table:a2}--\ref{table:a5} report the mean benchmark score and its standard deviation for ten LLMs on four real datasets across the $20$ replicates.
Leaderboards offer an intuitive ranking of models, but a ranking based on point estimates alone cannot separate models whose mean scores are nearly identical.
For example, Gemma-4-E4B-It and Qwen3.5-9B score $0.7751$ and $0.7790$ on ACEBench, a gap smaller than either of their standard deviations across replicates ($0.0104$ and $0.0085$).
Consistently, BB-EDGE leaves this pair unresolved at $r=20$ (Figure~\ref{figure:2} (b)).
Additionally, Figure~\ref{figure:heterogeneity} also shows that item-level advantages are positive for some items and negative for others despite near-zero benchmark-average differences: $\widehat{\Delta}_{e}^{\mathcal B}=-0.008$ for Gemma-4-E4B-it versus MiMo-7B-RL on GSM8K.

\begin{table}[htbp]
    \centering
    \begin{tabular}{cc}
        \toprule
        Model & Accuracy \\
        \midrule
        Qwen3.5-9B                  & $0.6436 \pm 0.0031$ \\
        Phi-4                       & $0.6267 \pm 0.0043$ \\
        Granite-4.1-8B              & $0.6124 \pm 0.0028$ \\
        Falcon-H1R-7B               & $0.5492 \pm 0.0035$ \\
        Intern-S1-Mini              & $0.5427 \pm 0.0030$ \\
        MiMo-7B-RL                  & $0.5291 \pm 0.0030$ \\
        Gemma-4-E4B-It              & $0.5015 \pm 0.0029$ \\
        DeepSeek-R1-Distill-Qwen-7B & $0.4728 \pm 0.0028$ \\
        Llama-3.1-8B-Instruct       & $0.4223 \pm 0.0032$ \\
        Mistral-7B-Instruct-v0.3    & $0.3527 \pm 0.0030$ \\
        \bottomrule
    \end{tabular}
    \caption{Accuracy (mean $\pm$ std) on MMLU-Pro for each model, sorted in descending order of accuracy.}
    \label{table:a2}
\end{table}

\begin{table}[htbp]
    \centering
    \begin{tabular}{cc}
        \toprule
        Model & Exact Match \\
        \midrule
        Granite-4.1-8B               & $0.9138 \pm 0.0043$ \\
        Phi-4                        & $0.9067 \pm 0.0065$ \\
        Qwen3.5-9B                   & $0.8488 \pm 0.0079$ \\
        Intern-S1-Mini               & $0.7971 \pm 0.0076$ \\
        MiMo-7B-RL                   & $0.7444 \pm 0.0074$ \\
        Gemma-4-E4B-It               & $0.7362 \pm 0.0076$ \\
        Llama-3.1-8B-Instruct        & $0.7235 \pm 0.0090$ \\
        DeepSeek-R1-Distill-Qwen-7B  & $0.7102 \pm 0.0091$ \\
        Mistral-7B-Instruct-v0.3     & $0.5054 \pm 0.0097$ \\
        Falcon-H1R-7B                & $0.4765 \pm 0.0099$ \\
        \bottomrule
    \end{tabular}
    \caption{Exact Match (mean $\pm$ std) on GSM8K for each model, sorted in descending order of accuracy.}
    \label{table:a3}
\end{table}

\begin{table}[htbp]
    \centering
    \begin{tabular}{cc}
        \toprule
        Model & LLM Judge \\
        \midrule
        Gemma-4-E4B-It                & $0.7422 \pm 0.0076$ \\
        Qwen3.5-9B                    & $0.7418 \pm 0.0139$ \\
        Granite-4.1-8B                & $0.6960 \pm 0.0098$ \\
        Phi-4                         & $0.6614 \pm 0.0117$ \\
        Intern-S1-Mini                & $0.5945 \pm 0.0150$ \\
        Llama-3.1-8B-Instruct         & $0.5714 \pm 0.0146$ \\
        Mistral-7B-Instruct-v0.3      & $0.4961 \pm 0.0113$ \\
        Falcon-H1R-7B                 & $0.4497 \pm 0.0141$ \\
        DeepSeek-R1-Distill-Qwen-7B   & $0.4233 \pm 0.0102$ \\
        MiMo-7B-RL                    & $0.3824 \pm 0.0149$ \\
        \bottomrule
    \end{tabular}
    \caption{LLM judge scores (mean $\pm$ standard deviation) on MT-Bench for each model, sorted in descending order of mean score.}
    \label{table:a4}
\end{table}

\begin{table}[htbp]
    \centering
    \begin{tabular}{cc}
        \toprule
        Model & AST-based Accuracy \\
        \midrule
        Qwen3.5-9B & $0.7790 \pm 0.0104$ \\
        Gemma-4-E4B-It & $0.7751 \pm 0.0085$ \\
        Phi-4  & $0.7308 \pm 0.0093$ \\
        Intern-S1-Mini  & $0.7295 \pm 0.0140$ \\
        Granite-4.1-8B   & $0.6619 \pm 0.0083$ \\
        Mistral-7B-Instruct-v0.3 & $0.5465 \pm 0.0114$ \\
        Llama-3.1-8B-Instruct  & $0.4996 \pm 0.0167$ \\
        DeepSeek-R1-Distill-Qwen-7B   & $0.3769 \pm 0.0213$ \\
        MiMo-7B-RL  & $0.3700 \pm 0.0218$ \\
        Falcon-H1R-7B   & $0.2620 \pm 0.0388$ \\
        \bottomrule
    \end{tabular}
    \caption{AST-based accuracy (mean $\pm$ standard deviation) on ACEBench for each model, sorted in descending order of mean score.}
    \label{table:a5}
\end{table}

\begin{figure}[htbp]
\centering
\subfloat[GSM8K]{%
    \includegraphics[width=0.98\textwidth]{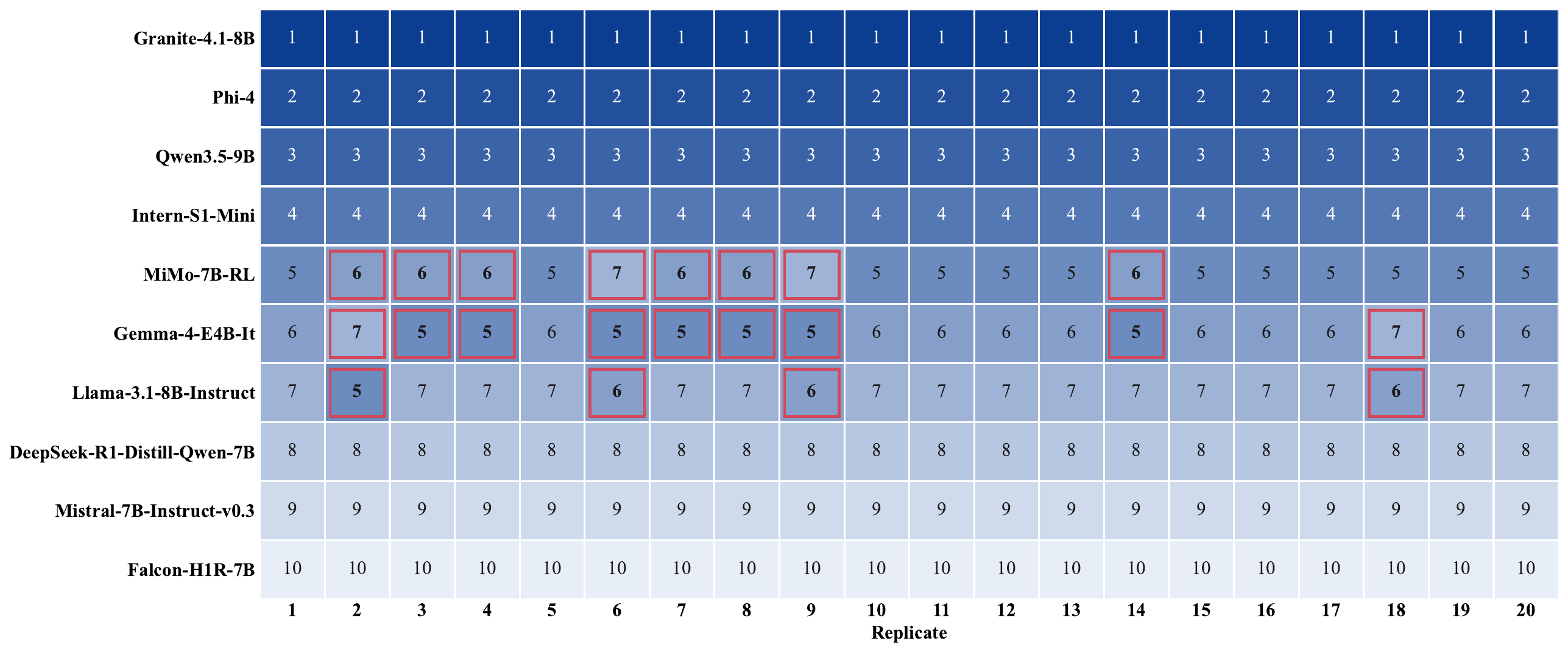}%
}
\caption{Single-run rankings of ten LLMs on GSM8K over 20 replicates. Rows are ordered by the 20-replicate mean score; outlined cells mark ranks that differ from the 20-replicate ranking.}
\label{fig:rank-stability}
\end{figure}

\begin{figure}[htbp]
\centering
\subfloat[GSM8K]{%
    \includegraphics[width=0.48\textwidth]{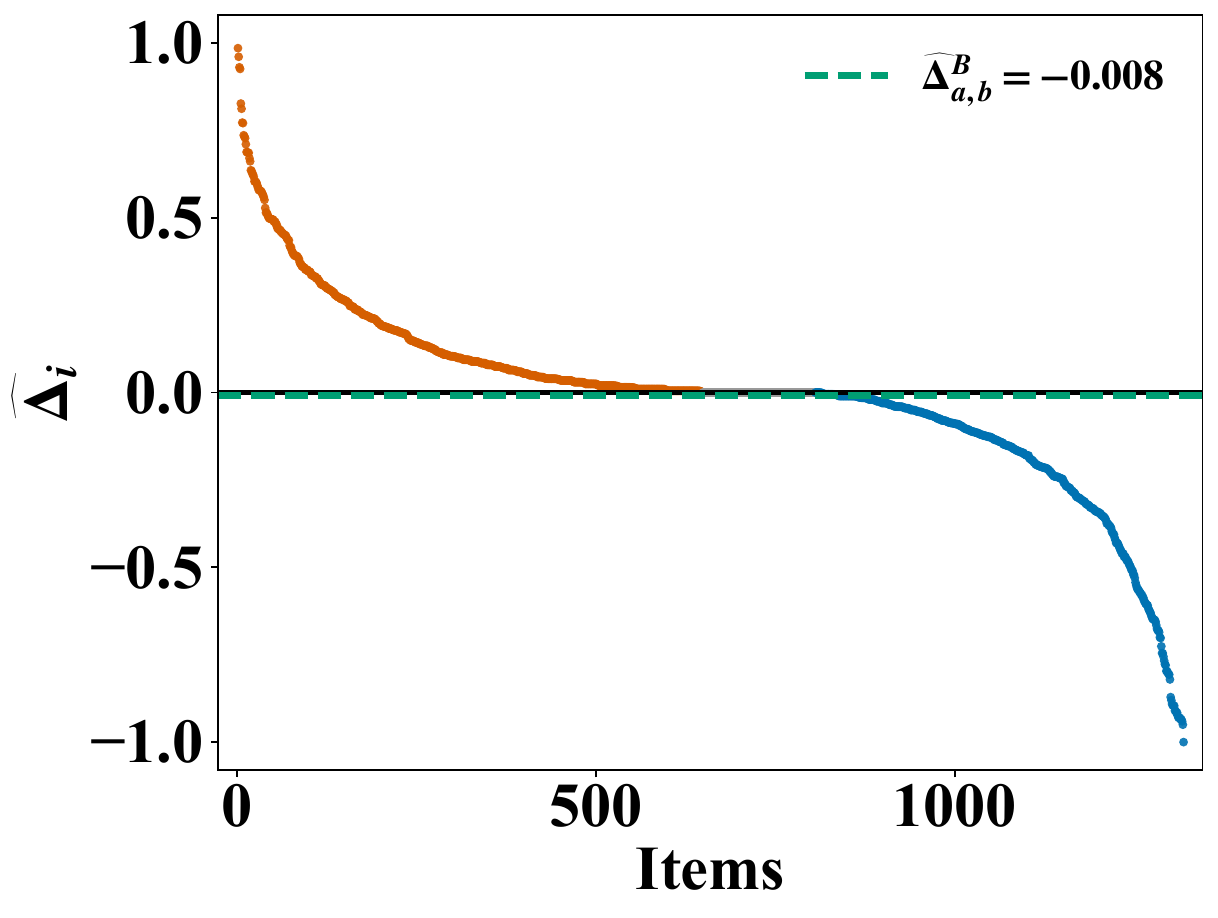}%
}
\caption{
Item-level heterogeneity on GSM8K. Different items favor different models despite near-zero benchmark-average differences: Gemma-4-E4B-it versus MiMo-7B-RL on GSM8K ($\widehat{\Delta}^{\mathcal B}_{\mathrm{Gemma},\mathrm{MiMo}}=-0.008$).
}
\label{figure:heterogeneity}
\end{figure}

\subsection{Full-Ranking Confidence Graph}
Figure~\ref{figure:a8} presents the directed acyclic confidence graphs at replicate $r=20$ for MMLU-Pro and MT-Bench.
The MMLU-Pro graph forms a fully resolved linear chain headed by Qwen3.5-9B, followed by Phi-4 and Granite-4.1-8B, with each displayed edge representing a statistically certified directional comparison.
The MT-Bench graph is nearly a total order, with only three pairs remaining unresolved after multiplicity correction: Gemma-4-E4B-It versus Qwen3.5-9B, Intern-S1-Mini versus Llama-3.1-8B-Instruct, and Falcon-H1R-7B versus DeepSeek-R1-Distill-Qwen-7B.
For the binary-score benchmarks, SERPANT's preference target is equivalent to the mean-score target after fair tie breaking.
On MT-Bench, which uses graded judge scores, SERPANT instead targets the induced pairwise-preference ordering and is included only as a reference.
Compared with $\tau=0$, Figure~\ref{figure:real_data_tau} shows that setting $\tau=0.01$ yields fewer certified edges because each comparison must establish an advantage greater than $0.01$. 
For example, increasing $\tau$ from $0$ to $0.01$ changes the fully resolved MMLU-Pro ranking into a partial order, leaving Phi-4 versus Granite-4.1-8B and Falcon-H1R-7B versus Intern-S1-Mini unresolved.
Nevertheless, BB-EDGE still identifies clear partial ranking structures across all four benchmarks.
\begin{figure*}[htbp]
\centering
\subfloat[MMLU-Pro]{%
    \includegraphics[width=0.43\textwidth]{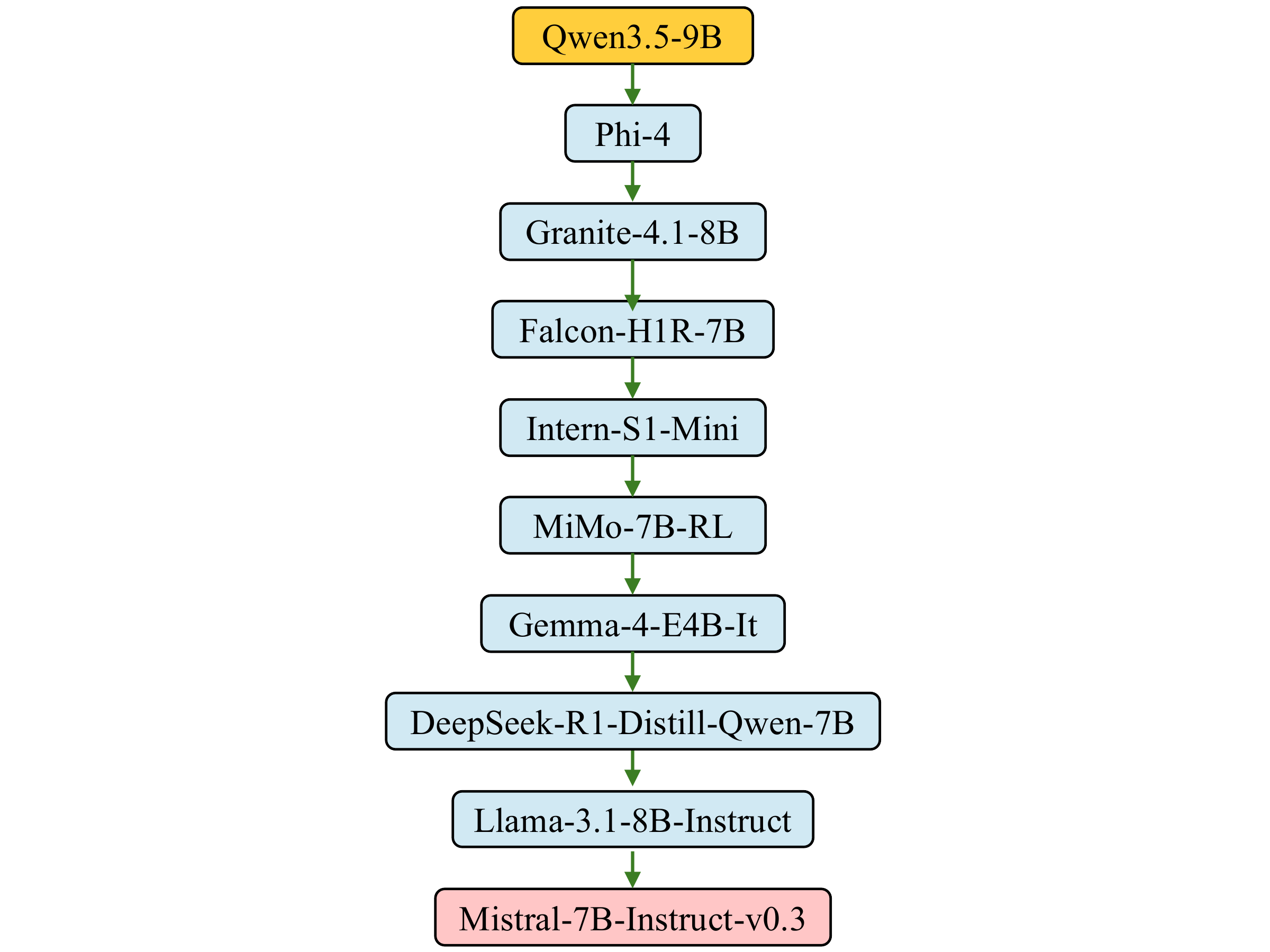}
}
\subfloat[MT-Bench]{%
    \includegraphics[width=0.43\textwidth]{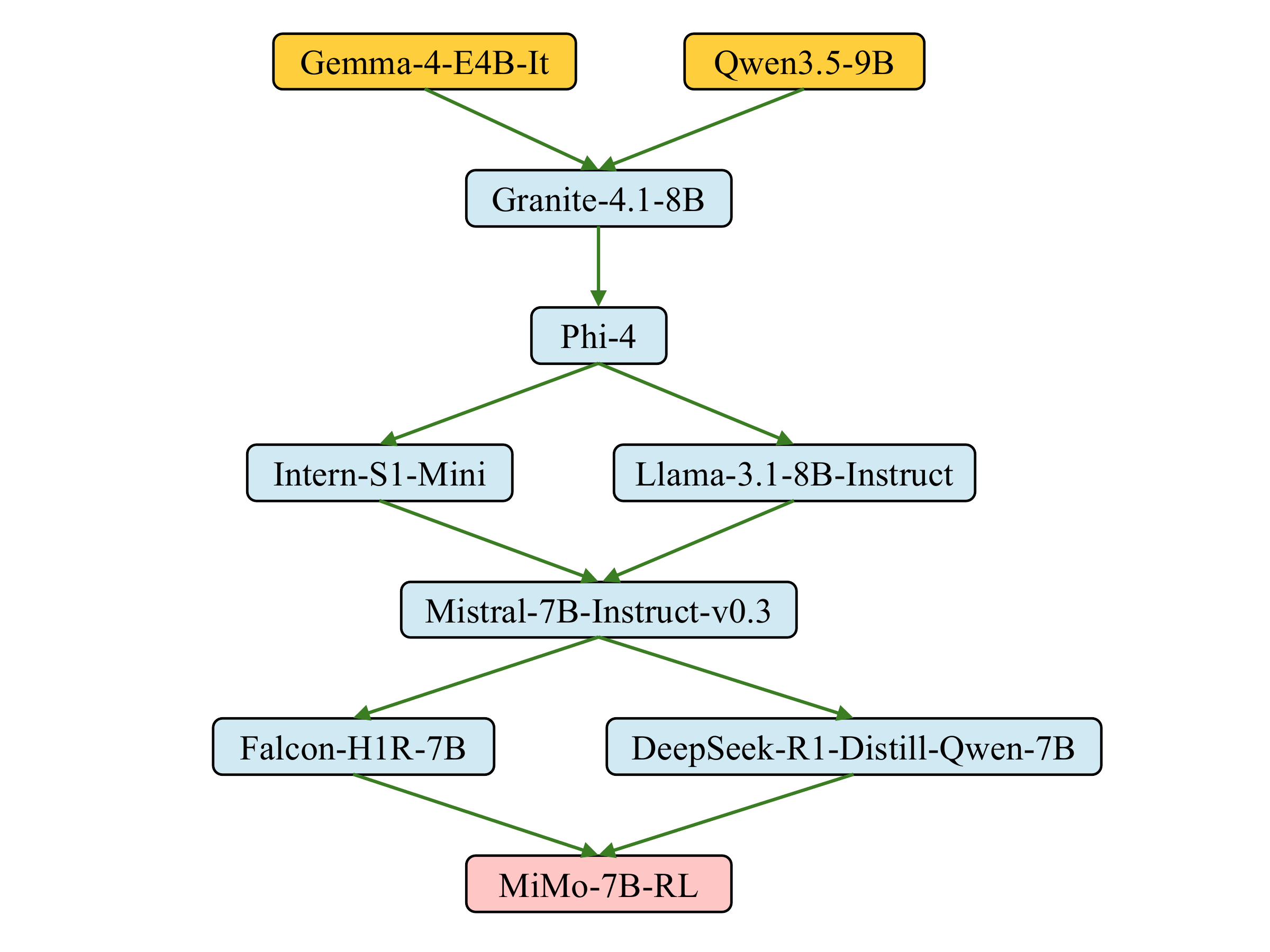}
}
\caption{
Certified comparison graphs at replicate $r=20$ on MMLU-Pro and MT-Bench, respectively.
An edge \(A\to B\) indicates that model $A$ is certified to outperform model $B$ on the benchmark under the evaluation protocol. Pairs not connected by a directed path remain unresolved.}
\label{figure:a8}
\end{figure*}

\begin{figure*}[htbp]
\centering
\subfloat[MMLU-Pro]{%
    \includegraphics[width=0.43\textwidth]{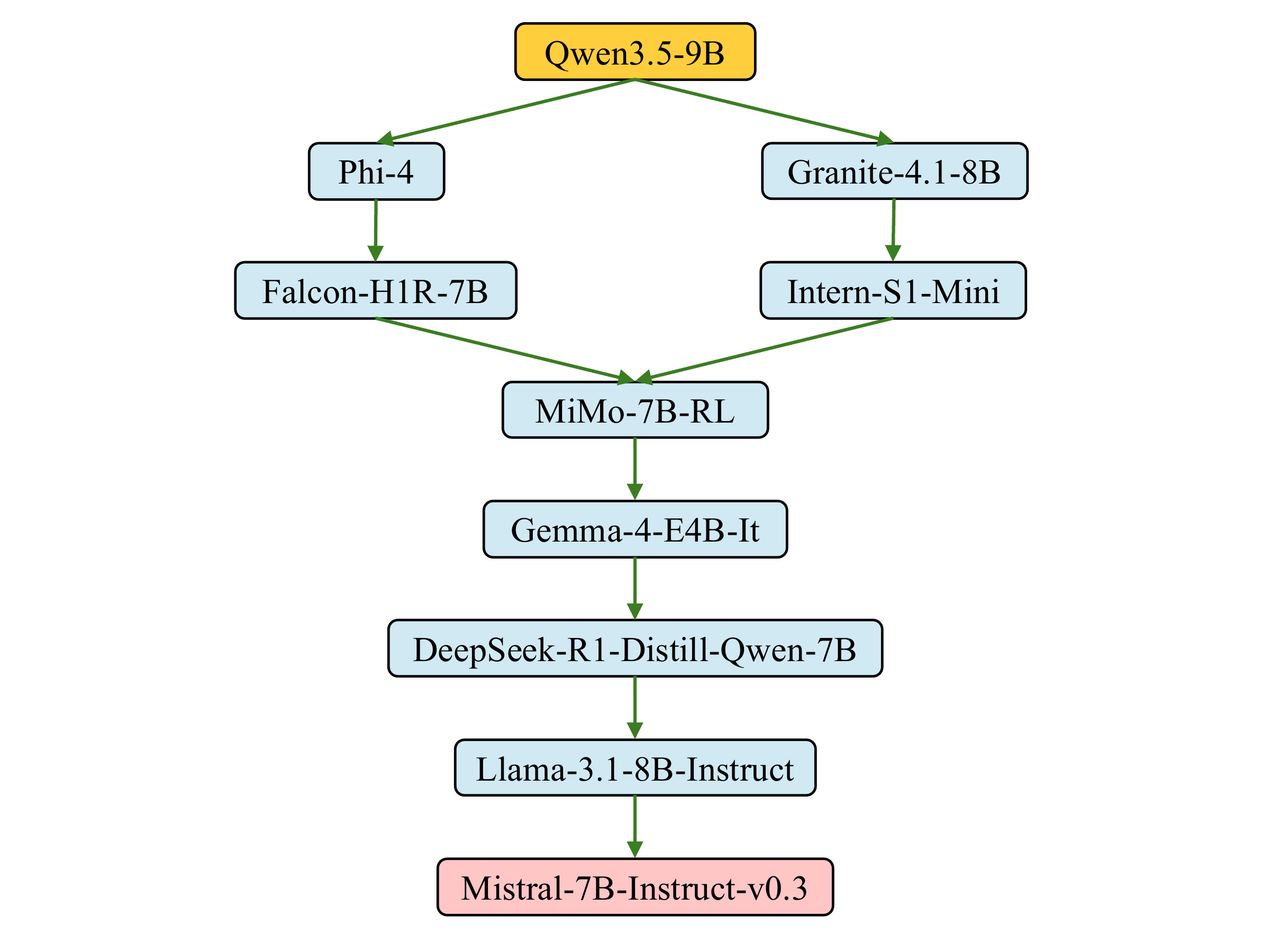}
}
\subfloat[GSM8K]{%
    \includegraphics[width=0.43\textwidth]{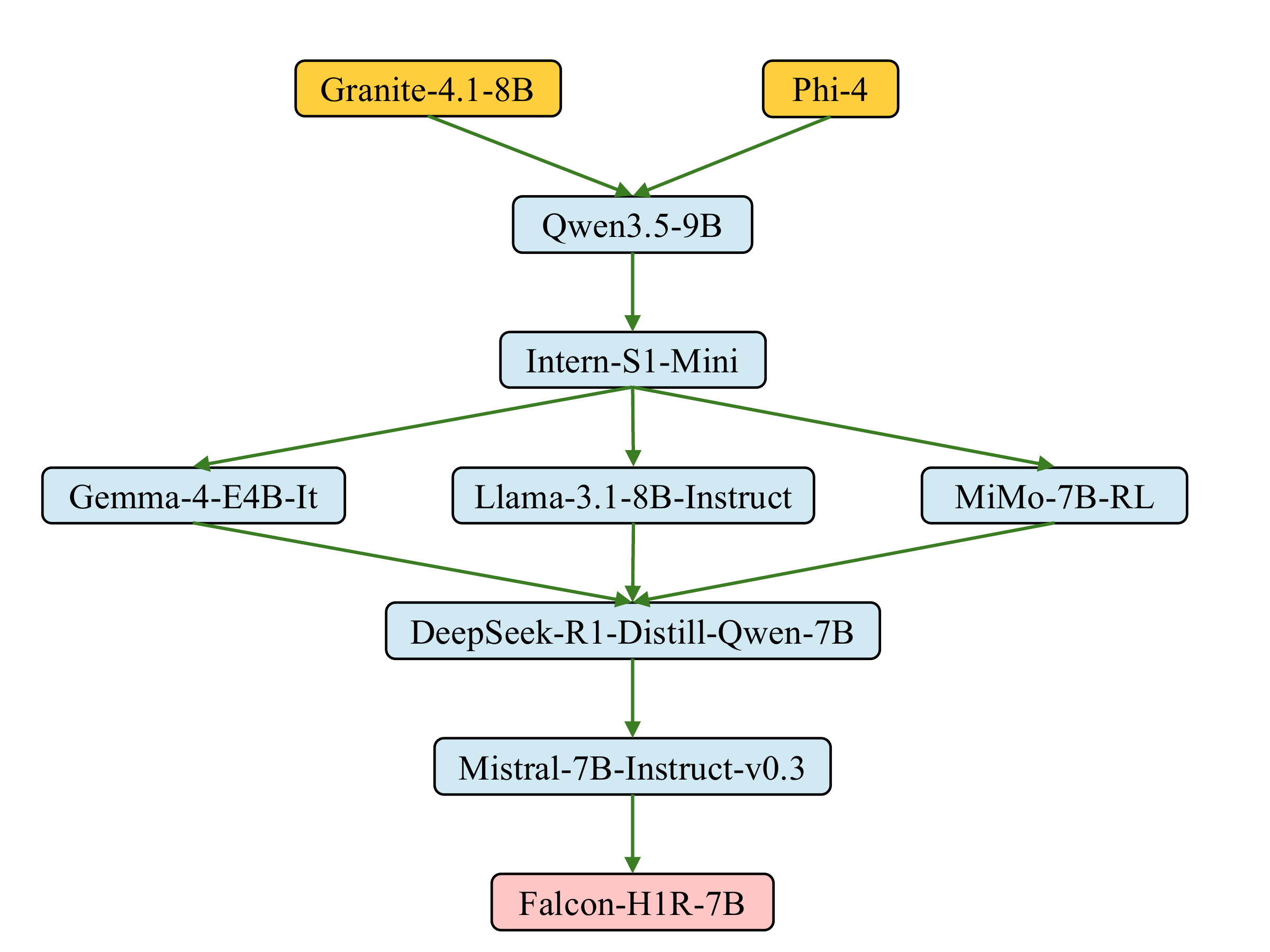}
}\\[1ex]
\subfloat[MT-Bench]{%
    \includegraphics[width=0.43\textwidth]{figure/MTBench_RANKING.pdf}
}
\subfloat[ACEBench]{%
    \includegraphics[width=0.43\textwidth]{figure/ACEBENCH_RANKING.pdf}
}
\caption{Certified comparison graphs at replicate $r=20$ on four datasets, with a prespecified margin of $\tau=0.01$. An edge \(A\to B\) indicates that model $A$ is certified to outperform model $B$ on the benchmark under the evaluation protocol. Pairs not connected by a directed path remain unresolved.}
\label{figure:real_data_tau}
\end{figure*}

Figure~\ref{figure:complete} shows the GSM8K confidence graph at $r = 100$ with a single block ($M = 1$), which requires no within-replicate independence.
It resolves 24 of 45 comparisons, certifying exactly the pairs with mean gaps above $0.14$, whereas the item-factorized graph ($M = N$) resolves 41 after only two replicates.
The one-block resolution requires roughly two orders of magnitude more replicates, so we recommend it only as a fallback when the protocol cannot justify conditional independence across blocks.

\begin{figure*}[htbp]
\centering
    \includegraphics[width=0.99\textwidth]{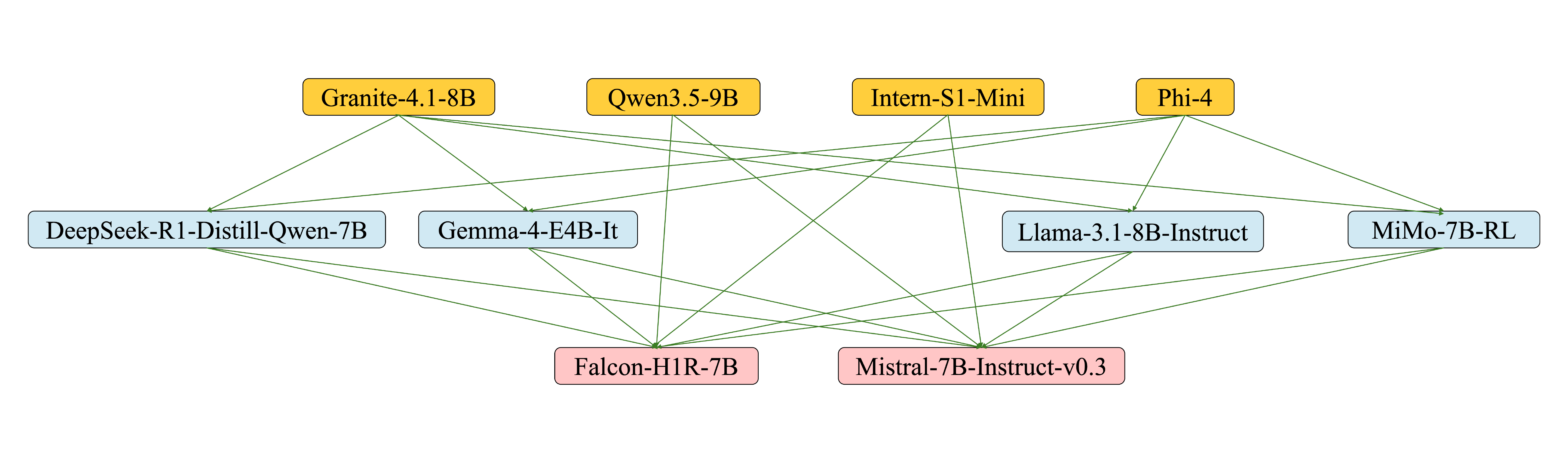}
\caption{Certified comparison graphs at replicate $r = 100$ on GSM8K under the one-block partition ($M = 1$) with $\tau = 0$. An edge \(A\to B\) indicates that model $A$ is certified to outperform model $B$ on the benchmark under the evaluation protocol. Pairs not connected by a directed path remain unresolved.}
\label{figure:complete}
\end{figure*}

\clearpage

\subsection{Top-$k$ Certification}
\label{sec:real_data_topk}
Table~\ref{table:a7} reports the certified Top-4 sets across all four benchmarks. Qwen3.5-9B and Phi-4 appear in every benchmark's Top-4, indicating consistently strong performance across diverse task types, while Granite-4.1-8B is certified in three of the four (all but ACEBench).
The remaining slot varies by benchmark: Falcon-H1R-7B on MMLU-Pro, Intern-S1-Mini on GSM8K and ACEBench, and Gemma-4-E4B-It on MT-Bench and ACEBench.

Table~\ref{table:a10} reports the minimum number of complete benchmark replicates required for Top-4 certification ($\alpha=0.05$, $\tau=0$) on the four benchmarks.
The results show that the two BB-EDGE variants always require the fewest replicates.
Full-graph BB-EDGE is fastest on MMLU-Pro and GSM8K, certifying the Top-4 set after one and two replicates, respectively.
On both benchmarks, the pilot-targeted variant requires one additional replicate for pilot selection, resulting in slightly later certification.
The pilot-targeted variant performs better on MT-Bench and ACEBench, where more replicates are otherwise required.
For example, on MT-Bench, full-graph BB-EDGE also certifies within 5 replicates, whereas the pilot-targeted variant only requires 3 replicates;

\begin{table}[htbp]
    \centering
    \begin{tabular}{ccccc} \toprule
        & MMLU-Pro & GSM8K & MT-Bench & ACEBench \\ \midrule
        \multirow{4}{*}{Top-4} & Falcon-H1R-7B   & Intern-S1-Mini & Gemma-4-E4B-It & Gemma-4-E4B-It \\
         & Granite-4.1-8B  & Granite-4.1-8B & Granite-4.1-8B & Intern-S1-Mini \\
         & Phi-4           & Phi-4          & Phi-4          & Phi-4 \\
         & Qwen3.5-9B      & Qwen3.5-9B     & Qwen3.5-9B     & Qwen3.5-9B \\ \bottomrule
    \end{tabular}
\caption{Certified Top-4 model sets on the four benchmarks at $\alpha=0.05$ and $\tau=0$.}
    \label{table:a7}
\end{table}

\begin{table}[htbp]
    \centering
    \begin{tabular}{lcccc}
        \toprule
        Method & MMLU-Pro & GSM8K & MT-Bench & ACEBench \\
        \midrule
        SERPANT       & \textbf{1} & 4 & 8& 20+ \\
        BFH-$e$-Holm  & 2 & 7 & 17& 17 \\
        CR-EDGE  & 20+ & 20+ & 20+& 20+ \\
        \midrule
        \textbf{BB-EDGE}        & \textbf{1} & \textbf{2} & 5& 4 \\
        \quad \textbf{+Pilot}  & 2 & 3 & \textbf{3}& \textbf{3} \\
        \bottomrule
    \end{tabular}
    \caption{Minimum number of complete benchmark replicates required for Top-4 certification at $\alpha=0.05$ and $\tau=0$. BB-EDGE and the baselines use full-graph certification; +pilot uses the pilot-targeted procedure with $R_0 =1$ pilot replicates, which are included in its counts. ``20+'' indicates that no certificate was issued within the budget of 20 replicates.
    Bold marks the fewest replicates among methods with anytime FWER control.}
    \label{table:a10}
\end{table}

\subsection{Rank Interval}
\label{sec:real_data_interval}
Figure~\ref{figure:a11} shows that ranking uncertainty varies across four benchmarks.
On MMLU-Pro, BB-EDGE uniquely determines every model's rank.
On GSM8K, most ranks are resolved, while Llama-3.1-8B-Instruct, Gemma-4-E4B-it, and MiMo-7B-RL have unresolved ranks within $5$--$7$.
On MT-Bench, uncertainty remains at ranks $1$--$2$, $5$--$6$, and $8$--$9$, corresponding to Gemma-4-E4B-it and Qwen3.5-9B, Intern-S1-Mini and Llama-3.1-8B-Instruct, and Falcon-H1R-7B and DeepSeek-R1-Distill-Qwen-7B, respectively.
On ACEBench, ranks $1$--$2$, $3$--$4$, and $8$--$9$ remain unresolved, while all other ranks are uniquely determined.
These results illustrate how BB-EDGE certifies precise ranks when supported by the evidence while retaining uncertainty in unresolved comparisons.

\begin{figure*}[htbp]
\centering
\subfloat[MMLU-Pro]{%
    \includegraphics[width=0.49\textwidth]{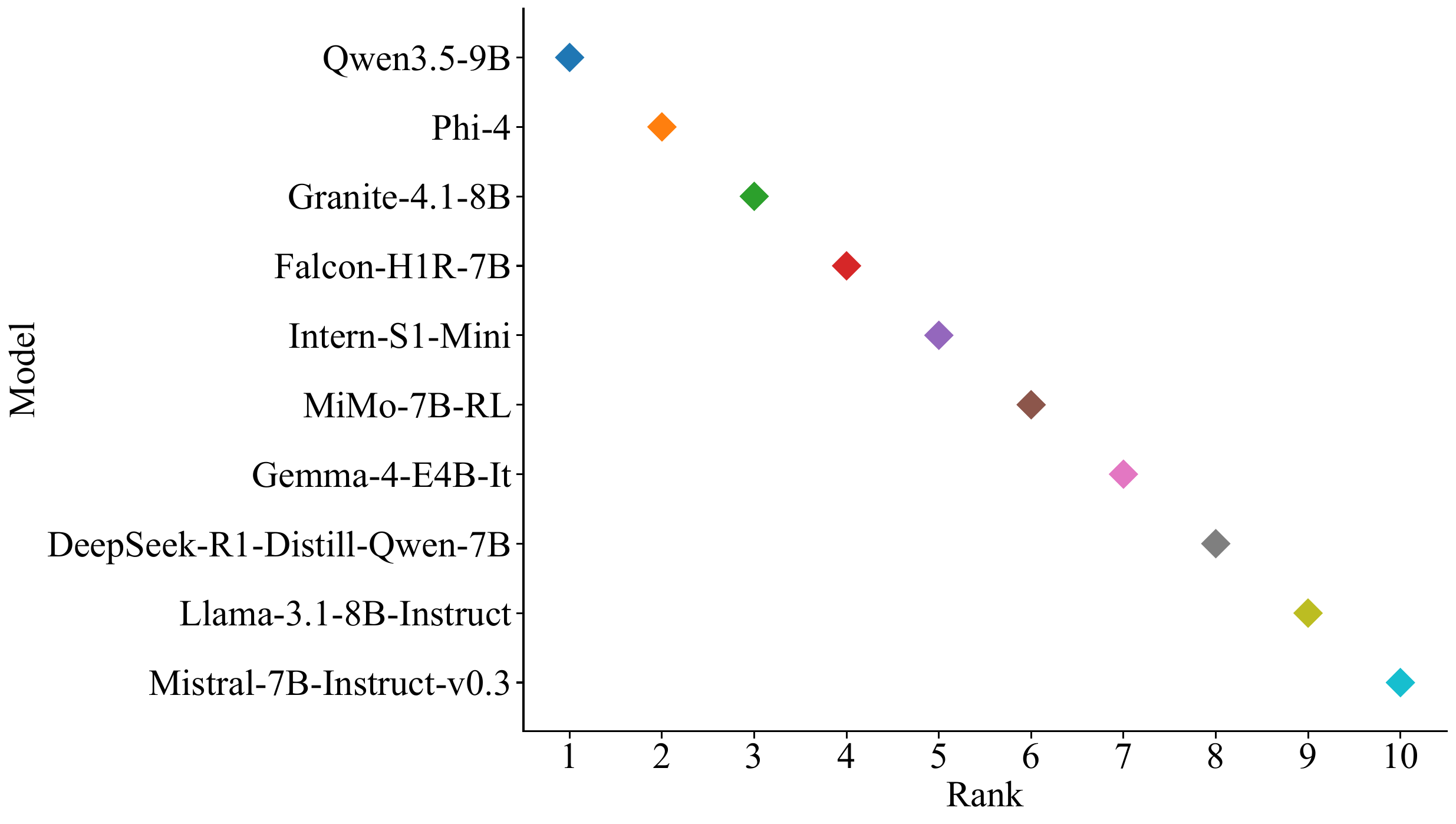}
}
\subfloat[GSM8K]{%
    \includegraphics[width=0.49\textwidth]{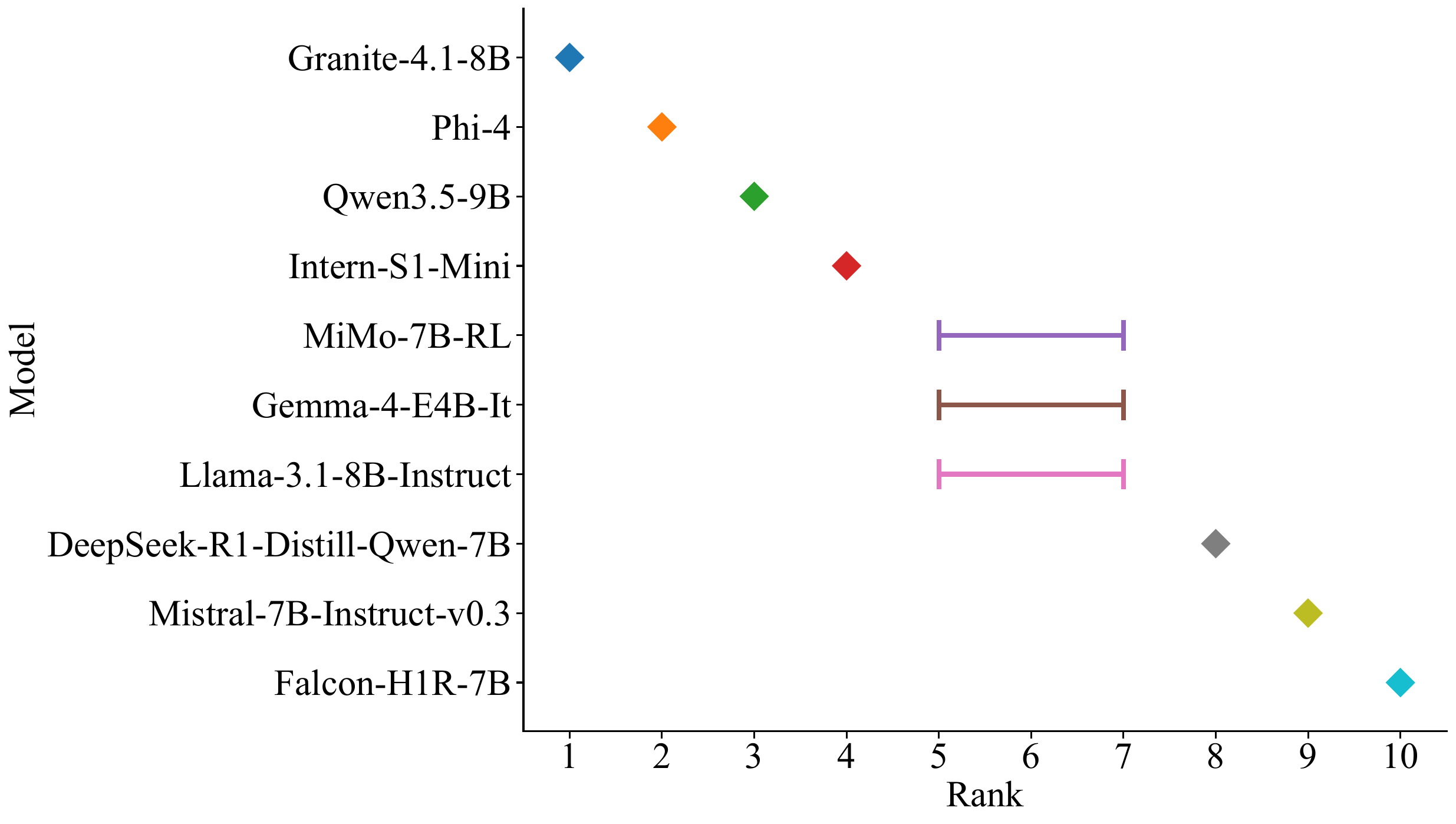}
} \\
\subfloat[MT-Bench]{%
    \includegraphics[width=0.49\textwidth]{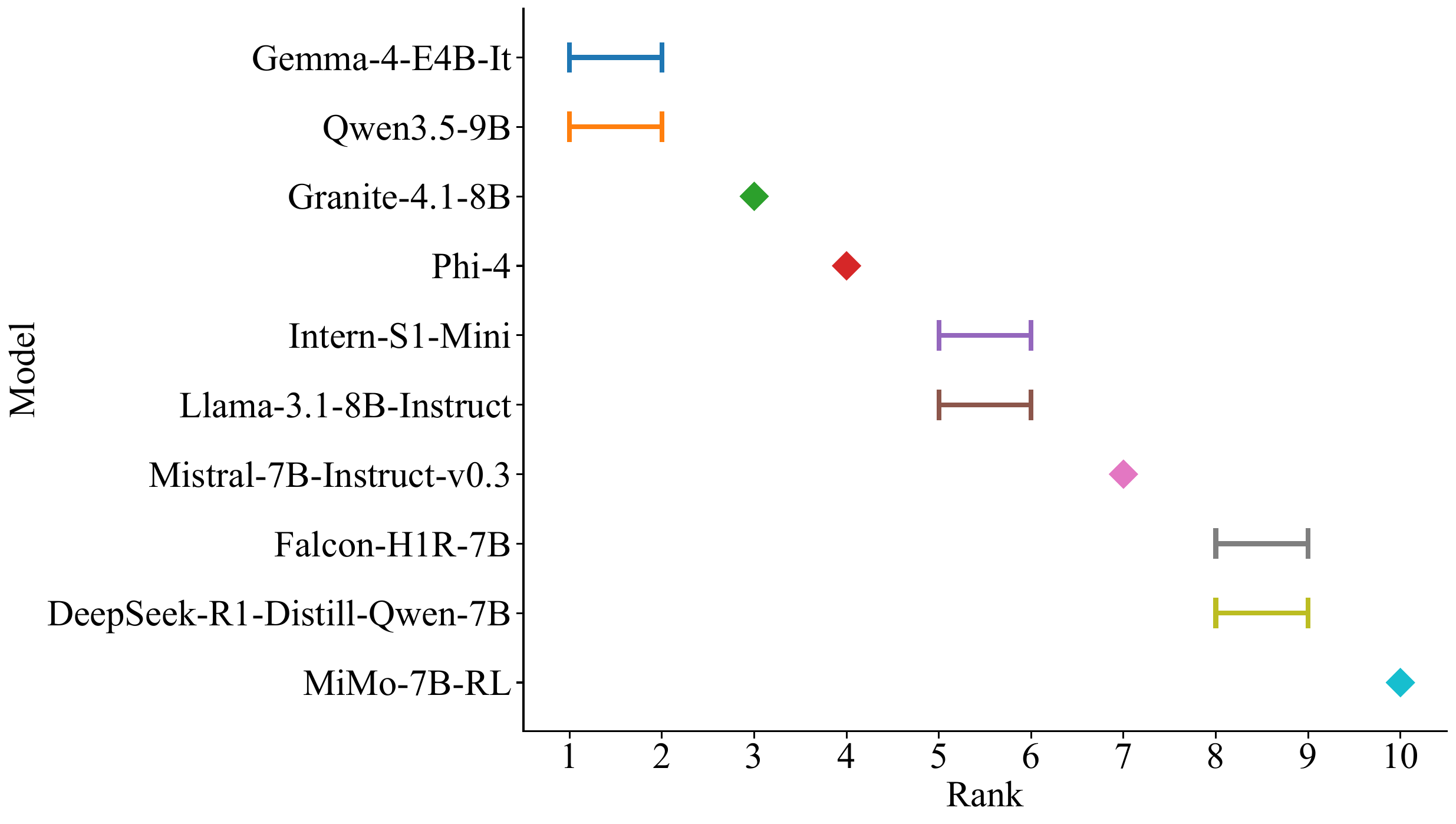}
}
\subfloat[ACEBench]{%
    \includegraphics[width=0.49\textwidth]{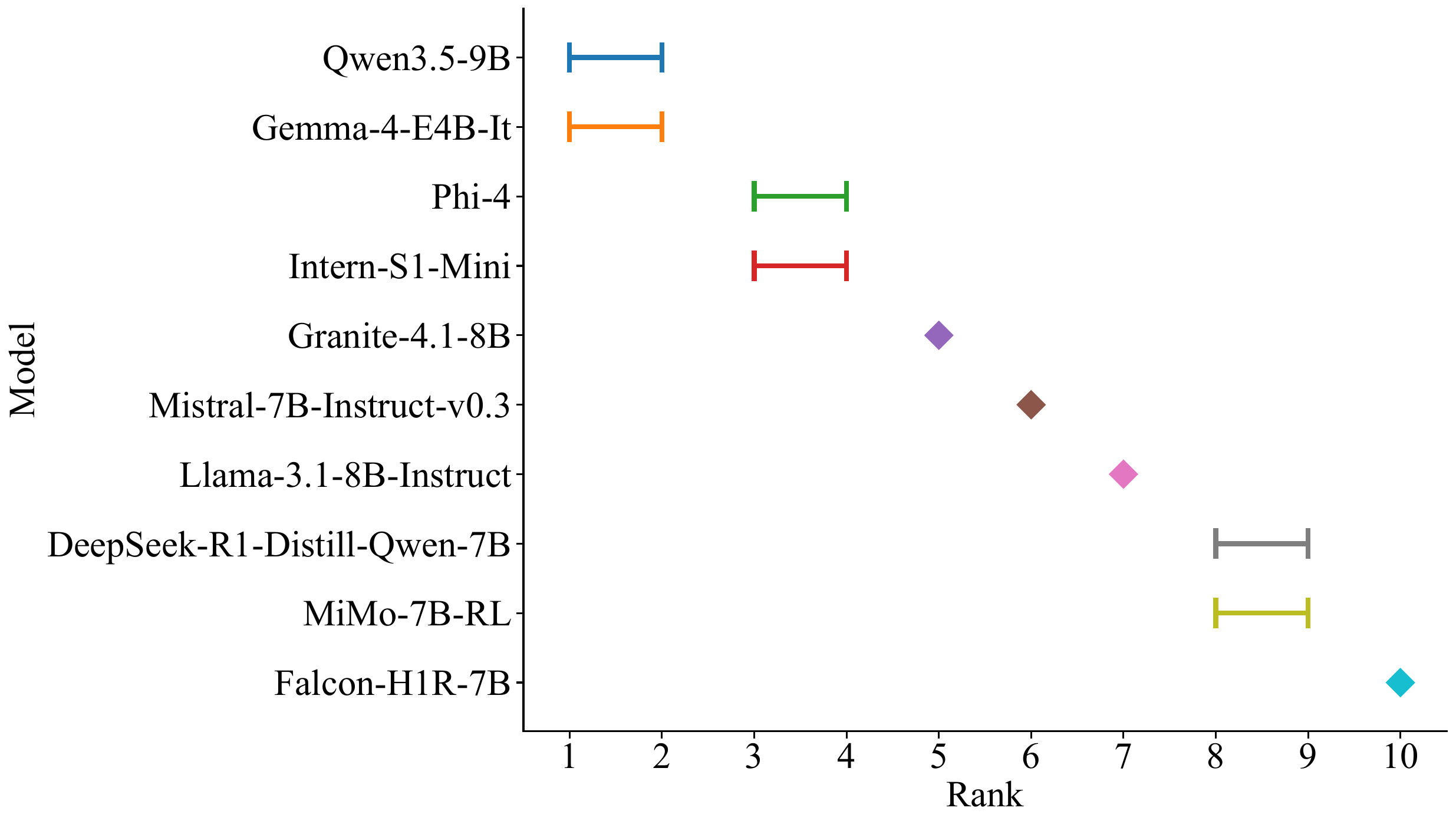}
}
\caption{Rank intervals for LLMs at replicate $r=20$ on four datasets. Each horizontal segment represents the inferred rank interval of a model. Overlapping intervals indicate unresolved ranking.}
\label{figure:a11}
\end{figure*}

\subsection{Closed-source APIs}
We further evaluate BB-EDGE on five API models (DeepSeek-V4-Flash \citep{xu2026deepseek}, GLM-5.3-Flash \citep{glm5team2026glm5vibecodingagentic}, GPT-5.6-Luna \citep{openai2026gpt56luna}, MiniMax-M3 \citep{lai2026minimax}, and Qwen3.8-Flash \citep{qwen3.8flashnext}).
Due to limited resources, this evaluation is restricted to ACEBench Normal Multi-Turn.
Table~\ref{table:a6} reports the mean scores, and Figure~\ref{figure:a9} presents the directed acyclic graph representing the inferred ranking of closed-source APIs at replicate $r=20$ on ACEBench.
GLM-5.3-Flash and Qwen3.8-Flash jointly occupy the top tier, each certified above GPT-5.6-Luna and DeepSeek-V4-Flash, respectively, both of which are in turn certified above MiniMax-M3.
The graph exhibits a symmetric two-branch structure that converges to a single lowest-ranked model.

\begin{table}[htbp]
    \centering
    \begin{tabular}{cc}
        \toprule
         Model & AST-based Accuracy \\
        \midrule
        GLM-5.3-Flash        & $0.8041 \pm 0.0089$ \\
        Qwen3.8-Flash        & $0.7982 \pm 0.0065$ \\
        GPT-5.6-Luna         & $0.7797 \pm 0.0087$ \\
        DeepSeek-V4-Flash    & $0.7756 \pm 0.0080$ \\
        MiniMax-M3           & $0.7198 \pm 0.0133$ \\
        \bottomrule
    \end{tabular}
    \caption{ACEBench scores (mean $\pm$ standard deviation) for each model, sorted
    in descending order of mean score.}
    \label{table:a6}
\end{table}

\begin{figure*}[htbp]
\centering
    \includegraphics[width=0.35\textwidth]{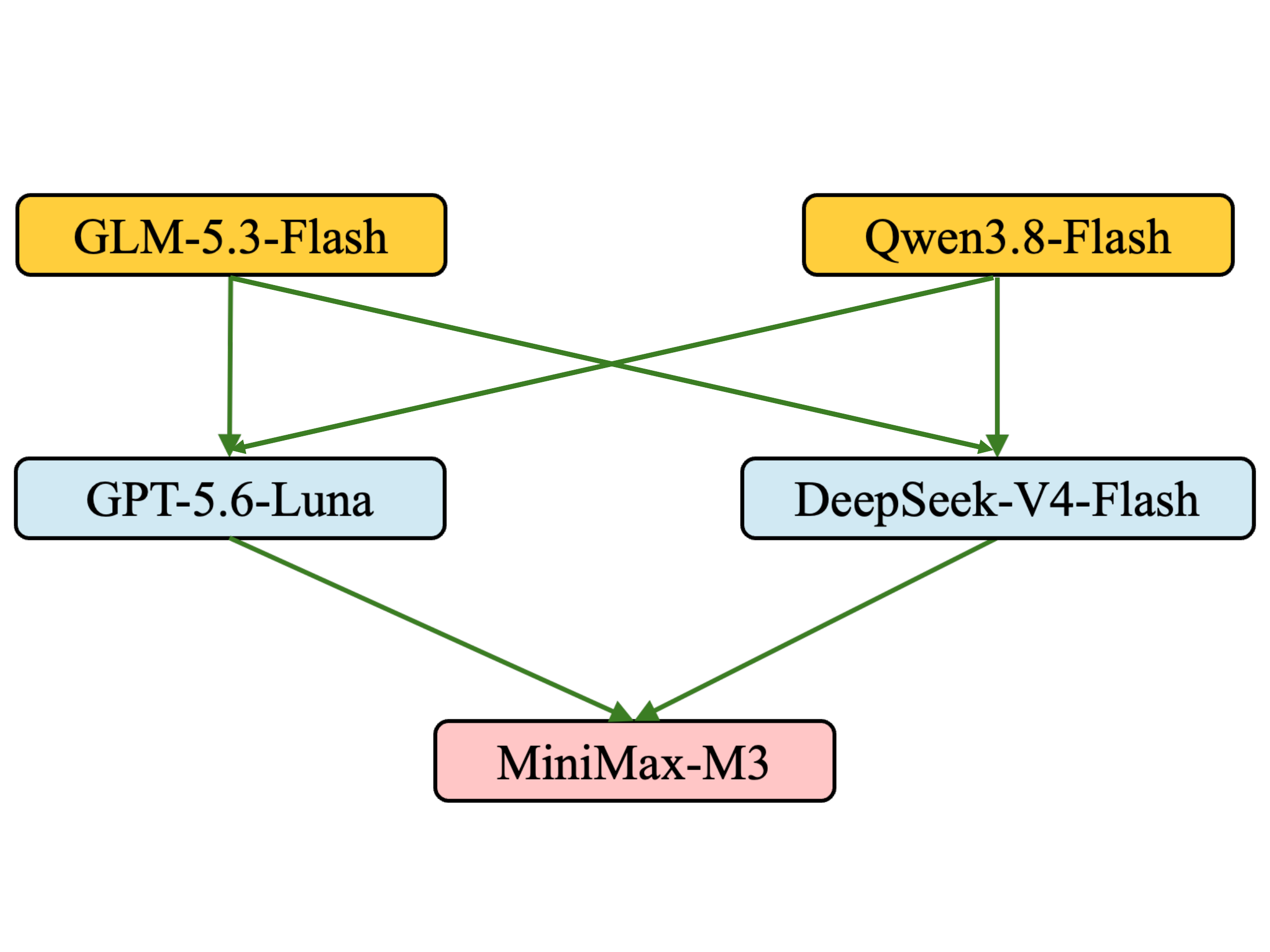}
\caption{Certified comparison graph representing the inferred ranking of APIs on ACEBench at replicate $r=20$.}
\label{figure:a9}
\end{figure*}

\section{Experimental Setup Details}
\label{section_experiments_details}

\subsection{Dataset}
\label{section_dataset}
We evaluate on four publicly available benchmarks: MMLU-Pro \citep{wang2024mmlu}, GSM8K \citep{cobbe2021gsm8k}, MT-Bench \citep{zheng2023judging}, and ACEBench \citep{chen2025acebench}.
Table \ref{table:a1}  summarizes the statistics, splits, evaluation metrics, and block sizes $|C_m|$ (the number of scored items per block); the four datasets span a range of block sizes and pool sizes.

\paragraph{MMLU-Pro} is an enhanced version of MMLU that expands each question's candidate options from 4 to up to 10, yielding a more challenging and discriminative benchmark.
Each MMLU-Pro item is evaluated through a separate model invocation and has no shared conversational context with other items. 
We therefore use singleton blocks under the working assumption that separate stateless item-level executions are conditionally independent given the evaluation protocol.

\paragraph{GSM8K} is a benchmark of grade-school-level math word problems that require multi-step arithmetic reasoning to arrive at a final numeric answer.
Each GSM8K problem is evaluated through a separate model invocation without cross-item context. 
Accordingly, we use singleton blocks under the working assumption that separate item-level executions are conditionally independent given the benchmark and protocol.

\paragraph{MT-Bench} is a benchmark designed to evaluate models' multi-turn conversational and instruction-following abilities. It consists of 80 two-turn questions. Following \citet{zheng2023judging}, the judge rates the two turns separately: the first-turn response is rated given the first question, and the second-turn response is rated given the full two-turn dialogue. Each question therefore contributes two scored items, giving $N = 160$. Ratings on the 1--10 scale are divided by 10. We group the two turns of each dialogue into one block, yielding $M=80$ dialogue-level blocks of size $|C_m|=2$.
This grouping allows arbitrary dependence between the two turns within each dialogue, while assuming conditional independence across separately executed dialogues given the fixed benchmark and evaluation protocol.

\paragraph{ACEBench} is a comprehensive benchmark for evaluating LLM tool usage. We focus on its Normal Multi-Turn subset, which contains $N = 223$ scored tool-call steps, each scored via AST-based function-call matching against the ground truth. Each dialogue contains two or three scored steps. We treat each self-contained dialogue as one block, so the blocks have sizes $|C_m| \in \{2, 3\}$ and weights $w_m = |C_m|/N$. We allow arbitrary dependence among steps within a dialogue and assume conditional independence across separately executed dialogues.

Permutation tests comparing the variance of replicate-level mean scores with that implied by independent blocks find no significant dependence across blocks for any model after Holm correction (smallest raw $p$: GSM8K $0.073$, MMLU-Pro $0.097$, MT-Bench $0.14$, ACEBench $0.37$), supporting the prespecified partitions.
For each model, we compare the variance of replicate-level mean scores with the sum of block-level score variances divided by $N^2$, which is its value under independence across blocks, and obtain a $p$-value by independently permuting each block's scores across replicates, which preserves within-block dependence while removing any dependence across blocks.

\begin{table}[htbp]
\caption{The statistics, split and evaluation metrics of each dataset.}
    \centering
    \begin{tabular}{ccccc}
    \toprule
         Data&  Train Set&  Test Set&Evaluation &$|C_m|$\\
    \midrule
         MMLU-Pro&  70&  12,032&Accuracy &1\\
         GSM8K& 7,470&  1,319 &Exact Match &1\\
         MT-Bench&  -&  160 (80 questions) &LLM Judge$^{*}$  &2\\
         ACEBench& -& 223& AST-based Accuracy &2--3\\
    \bottomrule
    \end{tabular}
    \begin{flushleft}
    \footnotesize
$^{*}$On MT-Bench, BB-EDGE certifies model comparisons based on the LLM-based scores generated by MiniMax-M3 under the given protocol; the reliability of the protocol itself is beyond the scope of this work.
    \end{flushleft}
    \label{table:a1}
\end{table}

\subsection{Baselines}
\label{section_appendix_baseline}

\paragraph{EMR} is the simplest baseline: it computes the empirical mean performance difference $\widehat{\Delta}_{a,b}$ over the $r$ observed replicates and certifies $(a,b)$ whenever $\widehat{\Delta}_{a,b}>\tau$. It accounts for neither sampling uncertainty nor multiplicity across the $H$ directional comparisons.

\paragraph{CBTL} \citep{wang2024confidence} models pairwise comparison probabilities as smooth functions of the item context using kernel smoothing. It certifies $(a,b)$ when the infimum of the fitted log-odds gap exceeds the transformed margin $\log\{(1+\tau)/(1-\tau)\}$, with the test calibrated by a Gaussian multiplier bootstrap over complete replicates.

\paragraph{$t$-Holm} \citep{chandrahas2026evalci} computes a paired one-sided $t$-test for each directional null $H_0:\Delta_{a,b}^{\mathcal B}\leq\tau$ using the complete-replicate mean differences. It then applies Holm's step-down procedure to the resulting $p$-values. This procedure provides fixed-time FWER control under its working assumptions, but repeated monitoring is not covered.

\paragraph{SERPANT} \citep{gu2026anytimevalid} is designed for online pairwise comparisons under an i.i.d.\ item stream. We convert paired item scores into binary preferences, breaking ties with complementary fair coins. Items are processed in replicate-by-item order; after every item, we update each direction's uniform-mixture likelihood-ratio \(e\)-process and apply direct \(e\)-Holm, with absorbing rejections. 
Thus, SERPANT has been monitored at all \(rN\) item-level times. Under a heterogeneous benchmark-average null, intermediate item prefixes need not satisfy the null even when the complete benchmark does, so this item-level procedure need not control the target anytime FWER.
This transformation preserves the benchmark mean-difference hypothesis
for binary scores, since the fair-tie preference probability equals
$(1+\Delta_{a,b}^{\mathcal B})/2$.
For the graded MT-Bench scores, this equivalence need not hold, so
SERPANT is included there only as a preference-based reference.
We use this adaptation primarily to diagnose the effect of
item-level monitoring on fixed-benchmark inference; its certification
times are descriptive when its item-stream assumptions or estimand
do not match the benchmark-average target.

\paragraph{BEB-Holm} is a fixed-time empirical-Bernstein baseline for the same benchmark-average directional hypotheses.
For the equal-sized blocks used in the synthetic experiments, it treats the \(n=rM\) block differences $D^e_{m,s}=\frac{1}{|C_m|}\sum_{i\in C_m}\bigl(S^{(a)}_{i,s}-S^{(b)}_{i,s}\bigr)$ with $s\le r$ and $m\in[M]$ as independent bounded observations.
Let \(\overline D^e_r\) and \(\widehat V^e_r\) denote their sample mean and variance.
We define $L^e_r(\delta)=\overline D^e_r-\sqrt{\frac{2\widehat V^e_r\log(4/\delta)}{n}}-\frac{16\log(4/\delta)}{3(n-1)}$ and obtain the directional \(p\)-value by inverting \(L^e_r(\delta)>\tau\). Holm's procedure is then applied across all directions. 
BEB-Holm is recomputed at each \(r\), but its guarantee is fixed-time and does not cover repeated monitoring.

\paragraph{BFH-$e$-Holm} uses the same block partition, proportional stakes, stake grid, mixture weights, and direct $e$-Holm procedure as BB-EDGE, replacing only the empirical-Bernstein penalty with a Hoeffding penalty.
This comparison isolates the benefit of variance adaptation.

\paragraph{CR-EDGE} uses the same empirical-Bernstein construction, stake grid, mixture weights, and direct \(e\)-Holm procedure as BB-EDGE, but treats each complete benchmark replicate as a single block.
This comparison isolates the efficiency gain from block factorization.

\textbf{BB-EDGE-PI} uses the same block observations, proportional stakes,
predictions, and direct $e$-Holm procedure as BB-EDGE, but replaces the grid mixture with a single predictable plug-in stake $\lambda_r = \min\{0.95,\ 0.5\,\hat g_{r-1}/\hat V_{r-1}\}$, where $\hat g_{r-1} = \max\{0, \hat\mu_{r-1} - \mu_0\}$ and $\hat V_{r-1}$ are regularized running estimates of the mean gap and the residual variance.
This isolates sensitivity to stake selection.

\textbf{AsympCS-Bonf} applies the one-sided Gaussian-mixture asymptotic
confidence sequence to the replicate-level mean differences of each direction, using the sample standard deviation across replicates, $t^\star = 10$, and a Bonferroni-corrected level $\alpha/H$.
No certificate is issued before a burn-in of two replicates. Its time-uniform coverage holds only asymptotically, so it tests whether a variance-adaptive but non-factorized procedure suffices in the small-budget regime.

\textbf{SERPANT-R} converts paired item scores into binary preferences as in
SERPANT, pools the preferences from all items of the first $r$ completed
replicates, and computes the same uniform-mixture likelihood-ratio $e$-values,
which treat items as i.i.d.\ Bernoulli observations. Direct $e$-Holm is applied
only at completed replicates. This tests whether restricting monitoring to
replicate boundaries alone removes the false certifications of item-level
monitoring.

\subsection{Evaluation Metric}
\label{section:evaluation_metric}
\paragraph{Top-$k$ certificate events}
Fix $1\le k<L$. Let $\mathcal S_s$ be the collection of sets certified by the procedure at completed run $s$. In full-graph certification,
\[
\mathcal S_s=\{T\subseteq\cV:|T|=k,\ a\leadsto_s b\text{ for every }a\in T,\ b\notin T\}.
\]
In the pilot-targeted procedure, $\mathcal S_s=\{\mathcal T_0\}$ if all hypotheses in $\cH(\mathcal T_0)$ are rejected, and $\mathcal S_s=\varnothing$ otherwise. 
These are different testing procedures and must be identified separately in an experiment. 
If evaluation stops at the first certificate, set $\mathcal S_s=\varnothing$ after that stopping time. 
Define
\begin{align}
C_{\le r}&=\{\exists s\le r:\mathcal S_s\ne\varnothing\},\nonumber\\
A_{\le r}&=\left\{\exists s\le r,\ \exists T\in\mathcal S_s:\min_{a\in T}\theta_a^{\cB}\le\max_{b\notin T}\theta_b^{\cB}+\tau\right\}. 
\label{eq:metric-certificate-events}
\end{align}
Thus $C_{\le r}$ records any issuance and $A_{\le r}\subseteq C_{\le r}$ records any incorrect issuance by replicate $R$. 
The error definition includes a failure to attain the required margin, even when the selected set contains the $k$ largest means. 
It also handles ties at the true Top-$k$ boundary without choosing an arbitrary true set.

\paragraph{$\operatorname{FCP}_r$}
The false Top-$k$ certification probability is
\begin{equation}
\operatorname{FCP}_r=\Pp(A_{\le r}\mid\cB,\pi). 
\label{eq:metric-fcp}
\end{equation}
On the no-false-edge event, every reported path from \(a\) to \(b\) satisfies \(\Delta_{a,b}^{\mathcal B}>\tau\). 
Consequently,
\[
\mathcal A_{\le r}\subseteq\left\{\exists s\le r,\ \exists(a\to b)\in\mathcal E_s:\Delta_{a,b}^{\mathcal B}\le\tau\right\}.
\]
For full-graph BB-EDGE, Theorem~\ref{thm:bedge-main-iclr} therefore gives $\operatorname{FCP}_r\le\operatorname{FWER}_r\le\alpha$. For pilot-targeted BB-EDGE, $A_{\le r}=C_{\le r}\cap B_0$, where $B_0$ is the candidate's margin-failure event in Appendix~\ref{app:proof-pilot-topk}. Theorem~\ref{thm:pilot-topk} gives the same bound, both conditional on the pilot and after averaging over it.

\paragraph{$\operatorname{CertProb}_r$} 
$\operatorname{CertProb}_r$ measures the probability of issuing any Top-$k$ certificate at replicate $r$:
\begin{equation}
\operatorname{CertProb}_r=\Pp(C_{\le r}\mid\cB,\pi). \nonumber
\end{equation}

\paragraph{Monte Carlo estimation.}
For $B_{\mathrm{MC}}$ independent simulation repetitions under the same setting, let $I_b(r)=\mathbbm1\{A_{\le r}\text{ occurs in repetition }b\}$ and $J_b(r)=\mathbbm1\{C_{\le r}\text{ occurs in repetition }b\}$. Compute
\[
\widehat{\operatorname{FCP}}_r=\frac{\sum_{b=1}^{B_{\mathrm{MC}}} I_b(r)}{B_{\mathrm{MC}}},
\qquad
\widehat{\operatorname{CertProb}}_r=\frac{\sum_{b=1}^{B_{\mathrm{MC}}} J_b(r)}{B_{\mathrm{MC}}}.
\]
Both denominators include uncertified repetitions.
The optional conditional estimate is $\sum_b I_b(r)/\sum_b J_b(r)$ when $\sum_b J_b(r)>0$, and is reported as undefined otherwise. 
Report the event counts and pointwise 95\% binomial Monte Carlo confidence intervals for FCP and certification probability.
These intervals quantify simulation uncertainty, not a second anytime guarantee.

\paragraph{$\operatorname{CR}$.}
Under the specified certification procedure, define
\[
\operatorname{CR}=\inf\{s\in\{1,\ldots,R\}:\mathcal S_s\ne\varnothing\},
\qquad\inf\varnothing=+\infty.
\]
Then $C_{\le r}=\{\operatorname{CR}\le r\}$ for $r\le R$.
For finite-budget experimental summaries, we censor $\operatorname{CR}$ at $R$: repetitions without a certificate are retained and recorded as $\operatorname{CR}=R$.
For the pilot-targeted procedure, the reported value includes the $R_0$ pilot replicate, and an uncertified repetition is recorded as $\operatorname{CR}=R+R_0$.
We report $\operatorname{CertProb}_R$ alongside the mean $\operatorname{CR}$ to distinguish successful certification from censoring.
The full-graph procedure searches all size-$k$ sets, whereas the pilot-targeted procedure checks only the fixed candidate $\mathcal T_0$.

\paragraph{$\operatorname{Coverage}_r$.}
Let $\mathcal I_{\ell,s}=[L_{\ell,s},U_{\ell,s}]$ denote the rank interval of model $\ell$ after replicate $s$.
The anytime coverage through replicate $r$ is defined as
\begin{equation}
\operatorname{Coverage}_r=\mathbb{P}\left(\forall s\le r,\ \forall \ell\in\mathcal V:\operatorname{rank}_{B}(\ell)\in\mathcal I_{\ell,s}\,\middle|\,B,\pi\right).\nonumber
\end{equation}

\paragraph{$\operatorname{Width}_r$.}
The mean rank-interval width at replicate $r$ is defined as
\begin{equation}
\operatorname{Width}_r=\mathbb{E}\left[\frac{1}{|\mathcal V|}\sum_{\ell\in\mathcal V}\left(U_{\ell,r}-L_{\ell,r}\right)\,\middle|\,B,\pi\right].\nonumber
\end{equation}
Smaller values of $\operatorname{Width}_r$ indicate that the rank interval contains fewer candidate models.

\subsection{Experiment details}
For simulated experiments, the five benchmark-level mean values are equally spaced over
$[0.41,0.59]$, each assigned to a tied pair of models; model labels
are shuffled once. Item offsets are fixed and equally
spaced over $[-0.08,0.08]$. Within each tied pair, one model has
item effects $+H$ on the first half of the benchmark and $-H$ on the
second half, while the other has the opposite effects, with
$H\in\{0,0.2\}$. These effects remain fixed across replicates;
Gaussian factors are independently regenerated for each replicate
and Monte Carlo repetition. Probabilities
outside $[0,1]$ are rejected rather than clipped.

We ran locally hosted models on an NVIDIA L40 GPU with temperature $0.7$, top-$p$ $0.95$, and a maximum generation length of 4000 tokens. 
Model-generation and judge calls were executed separately and sequentially using independent stateless requests and random seeds, with no randomness, state, context, batching, or judge output shared across items or models.
For GSM8K and MMLU-Pro, we adopt the evaluation protocol from the lm-evaluation-harness \citep{eval-harness}.
For MT-Bench, we follow \citet{zheng2023judging}, and for ACEBench, we follow \citet{chen2025acebench}.
Each evaluation block was processed through a separate stateless request.
We fixed decoding configurations, randomized block order within each replicate, and used independent seeds. 
These design choices support, but do not by themselves establish, the conditional independence assumptions.
For BB-EDGE, we used 41 equally spaced stakes over $[0,0.95]$ with uniform mixture weights and initialized each block-specific running-mean prediction at $\mu_0=(1+\tau)/2$.

\subsection{Query template}
\begin{table}[htbp]
\centering
\caption{Query template for the multi-task knowledge and reasoning benchmark (MMLU-Pro).}
\label{tab:mmlupro_cot_template}
\begin{promptbox}
<initial_instruction for subject: $\{$subject$\}$>

... (4 more few-shot exemplars omitted for brevity) ...

Question:
<question>
Options:
A. <option_1>
B. <option_2>
...
Answer: Let's think step by step.
\end{promptbox}
\end{table}

\begin{table}[htbp]
\centering
\caption{Query template for the math reasoning task (GSM8K).}
\label{tab:gsm8k_cot_template}
\begin{promptbox}

... (4 more few-shot exemplars omitted for brevity) ...

Question:
<question>
Answer: Let's think step by step.
\end{promptbox}
\end{table}

\begin{table}[htbp]
\centering
\caption{Query template for the multi-turn dialogues task (MT-Bench).}
\label{tab:qmsum_template}
\begin{promptbox}
Turn 1:
User: <question_turn_1>
Assistant: <model response>

Turn 2:
User: <question_turn_2>
Assistant: <model response>
\end{promptbox}
\end{table}

\begin{table}[htbp]
\centering
\caption{Judge prompt structure for MT-Bench single-answer grading.}
\label{tab:mtbench_judge_template}
\begin{promptbox}
[System]: You are a helpful assistant.

[Instruction]
Act as an impartial judge and evaluate the quality of the AI assistant's response, considering helpfulness, relevance, accuracy, depth, creativity, and level of detail.
(For math/reasoning/coding categories: additionally compare against the provided reference answer and identify any mistakes.)
Provide a brief explanation, then output a rating from 1 to 10 in the strict format "[[rating]]".

--- Turn 1 grading (rate the first-turn response) ---
[Question]
<question_turn_1>

[The Start of Assistant's Answer]
<answer_turn_1>
[The End of Assistant's Answer]

--- Turn 2 grading (rate only the second-turn response, given the full dialogue) ---
[The Start of Assistant's Conversation with User]
### User: <question_turn_1>
### Assistant: <answer_turn_1>
### User: <question_turn_2>
### Assistant: <answer_turn_2>
[The End of Assistant's Conversation with User]
\end{promptbox}
\end{table}

\begin{table}[htbp]
\centering
\caption{Query template for the function-calling task (ACEBench).}
\label{tab:acebench_template}
\begin{promptbox}
You are an AI assistant with the role name "assistant." Based on the provided API specifications and conversation history from steps 1 to t, generate the API requests that the assistant should call in step t+1. The API requests should be output in the format [ApiName(key1='value1', key2='value2', ...)], replacing ApiName with the actual API name, key1, key2, etc., with the actual parameter names, and value1, value2, etc., with the actual parameter values. The output should start with a square bracket "[" and end with a square bracket "]".
If there are multiple API requests, separate them with commas, for example: [ApiName(key1='value1', ...), ApiName(key1='value1', ...), ...]. Do not include any other explanations, prompts, or API call results in the output.
... (formatting and parameter rules omitted for brevity) ...

Role Descriptions:
user: User
assistant: The AI assistant role that makes API requests
tool: Provides the results returned from tool calls

API Specifications:
<function>

Conversation history 1..t:
<question>
\end{promptbox}
\end{table}

\end{document}